\newif\ifneurips
\neuripsfalse
\newif\ifmeta
\metatrue

\documentclass[single-column]{fairmeta}

\title{The Reward Was in Your Data All Along: Correcting Flow Matching with Discriminator-Guided RL}

\author[1,2]{Nicolas Beltran-Velez}
\author[1]{Felix Friedrich}
\author[1,3,5]{Zhang Xiaofeng}
\author[1]{Reyhane Askari-Hemmat}
\author[1]{Xiaochuang Han}
\author[1,3,4,6,*]{Adriana Romero-Soriano}
\author[1,*]{Michal Drozdzal}

\affiliation[1]{FAIR at Meta}
\affiliation[2]{Columbia University}
\affiliation[3]{Mila -- Qu\'{e}bec AI Institute}
\affiliation[4]{McGill University}
\affiliation[5]{Universit\'{e} de Montr\'{e}al}
\affiliation[6]{Canada CIFAR AI Chair}
\contribution[*]{Joint senior authors.}

\usepackage{amsmath}
\usepackage{amssymb}
\usepackage{mathtools}
\usepackage{amsthm}
\usepackage{aliascnt}
\usepackage{comment}
\usepackage{tcolorbox}
\usepackage{algorithm}
\usepackage{algorithmic}
\usepackage{float}
\usepackage{placeins}
\usepackage{wrapfig}
\usepackage{makecell}
\usepackage{xspace}
\usepackage{enumitem}

\newif\ifshowcomments

\newcommand{\resolved}[1]{}
\newcommand{\lprl}{\lambda_{\textsc{PRL}}}
\newcommand{\drl}{\texttt{DRL}\xspace}

\usepackage{url}
\usepackage{titletoc}
\usepackage{colortbl}

\RenewDocumentCommand{\paragraph}{s o m}{\par\noindent\textbf{#3}\xspace}
\makeatletter
\NewDocumentCommand{\labeledparagraph}{O{} m m}{%
  \par\phantomsection%
  \begingroup%
    \def\parhead@short{#1}%
    \ifx\parhead@short\@empty%
      \def\@currentlabelname{#2}%
    \else%
      \def\@currentlabelname{#1}%
    \fi%
    \label{#3}%
  \endgroup%
  \paragraph{#2}%
}
\NewDocumentCommand{\parheadref}{m}{\nameref{#1}}
\makeatother

\definecolor{paperGreen}{HTML}{26A17B}
\definecolor{paperOrange}{HTML}{DD6B20}

\definecolor{tableTeal}{HTML}{BEE3D7}
\definecolor{tableRedDark}{HTML}{E7B5A2}

\DeclareMathOperator{\KLop}{KL}
\DeclareMathOperator{\TVop}{TV}
\newcommand{\TV}[2]{\TVop\!\left(#1,\,#2\right)}

\DeclareMathOperator{\Var}{Var}
\DeclareMathOperator{\Cov}{Cov}

\newcommand{\R}{\mathbb{R}}

\newcommand{\X}{\mathcal{X}}

\newcommand{\E}{\mathbb{E}}

\newcommand{\KL}[2]{\KLop\!\left(#1\,\|\,#2\right)}
\newcommand{\norm}[1]{\left\lVert#1\right\rVert}

\theoremstyle{plain}
\newtheorem{theorem}{Theorem}[section]
\newaliascnt{proposition}{theorem}
\newtheorem{proposition}[proposition]{Proposition}
\aliascntresetthe{proposition}
\newaliascnt{lemma}{theorem}
\newtheorem{lemma}[lemma]{Lemma}
\aliascntresetthe{lemma}
\newaliascnt{corollary}{theorem}

\aliascntresetthe{corollary}
\theoremstyle{definition}
\newaliascnt{definition}{theorem}

\aliascntresetthe{definition}
\newaliascnt{assumption}{theorem}

\aliascntresetthe{assumption}
\theoremstyle{remark}
\newaliascnt{remark}{theorem}
\newtheorem{remark}[remark]{Remark}
\aliascntresetthe{remark}

\crefname{theorem}{theorem}{theorems}
\Crefname{theorem}{Theorem}{Theorems}
\crefname{proposition}{proposition}{propositions}
\Crefname{proposition}{Proposition}{Propositions}
\crefname{lemma}{lemma}{lemmas}
\Crefname{lemma}{Lemma}{Lemmas}
\crefname{corollary}{corollary}{corollaries}
\Crefname{corollary}{Corollary}{Corollaries}
\crefname{definition}{definition}{definitions}
\Crefname{definition}{Definition}{Definitions}
\crefname{assumption}{assumption}{assumptions}
\Crefname{assumption}{Assumption}{Assumptions}
\crefname{remark}{remark}{remarks}
\Crefname{remark}{Remark}{Remarks}

\abstract{Score- and flow-matching models often rely on preference-based reinforcement
learning for two purposes: aligning with subjective preferences and,
surprisingly, recovering properties---such as visual realism and coherent
object structure---that matching-based training is intended to learn from the
data itself. We argue that this reflects a structural mismatch. Matching losses measure $\ell_2$ regression error on the
velocity or score field under training-time marginals, a proxy poorly aligned
with the visual and semantic properties that determine sample quality at
inference. Given a reward aligned with these properties, RL sidesteps the
mismatch by evaluating the model on its own samples and following the reward
landscape directly. The challenge is to obtain such a reward without relying on
human preferences, which are expensive and conflate data realism with annotator
inclinations.
We propose Discriminator-Guided RL (\drl). \drl{} trains a discriminator to
separate data from base-model samples in a pretrained representation space and
uses its logit as the reward in KL-regularized RL. The pretrained space restricts the discriminator to perceptually meaningful directions, and the logit estimates the log-likelihood ratio between data and model, which is the optimal reward for targeting the data distribution.
Across SiT, JiT, REPA, and RAE, \drl{} reduces guidance-free FID (e.g., $9.38 \to 2.62$ on SiT) and semantic-space FD (e.g., $88.2 \to 19.3$ on DINOv3 for SiT), with consistent gains across all backbones, and improves human-preference rewards without training on them.
It also yields
a better Pareto frontier between preference reward and image fidelity under
subsequent preference-based post-training, increasing alignment while
reducing low-level artifacts such as oversaturation and excessive brightness.

}
\date{\today}
\correspondence{Nicolas Beltran-Velez, \email{nb2838@columbia.edu}}

\begin{document}

\maketitle

\begin{figure}[!bp]
    \vspace{-0.25cm}
    \centering
    \begin{subfigure}[t]{0.48\linewidth}
        \centering
        \includegraphics[trim={0 222pt 0 0}, clip, width=\linewidth]{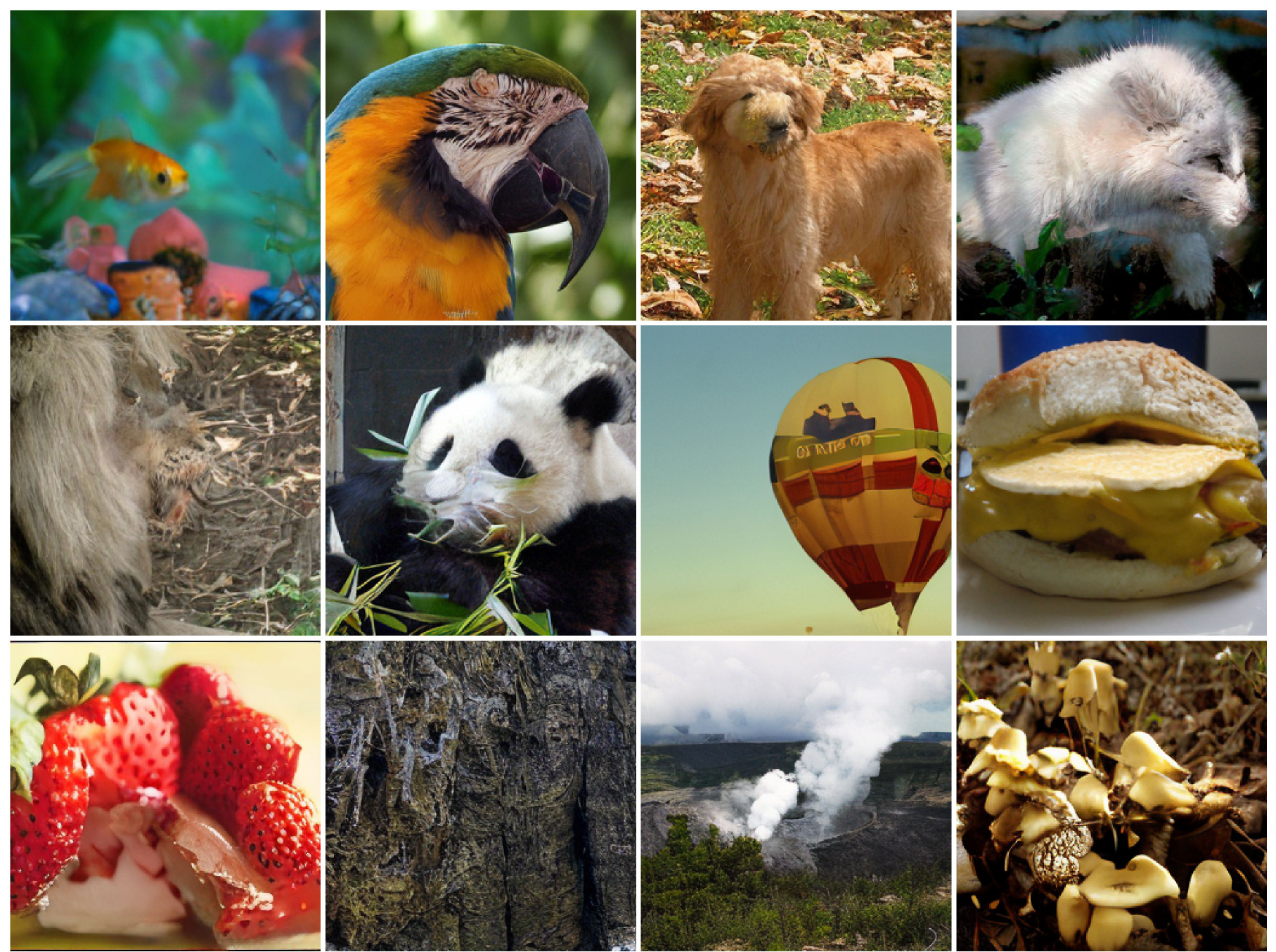}
        \caption{Base model (REPA SiT-XL/2)}
        \label{fig:teaser_base}
    \end{subfigure}
    \hfill
    \begin{subfigure}[t]{0.48\linewidth}
        \centering
        \includegraphics[trim={0 222pt 0 0}, clip, width=\linewidth]{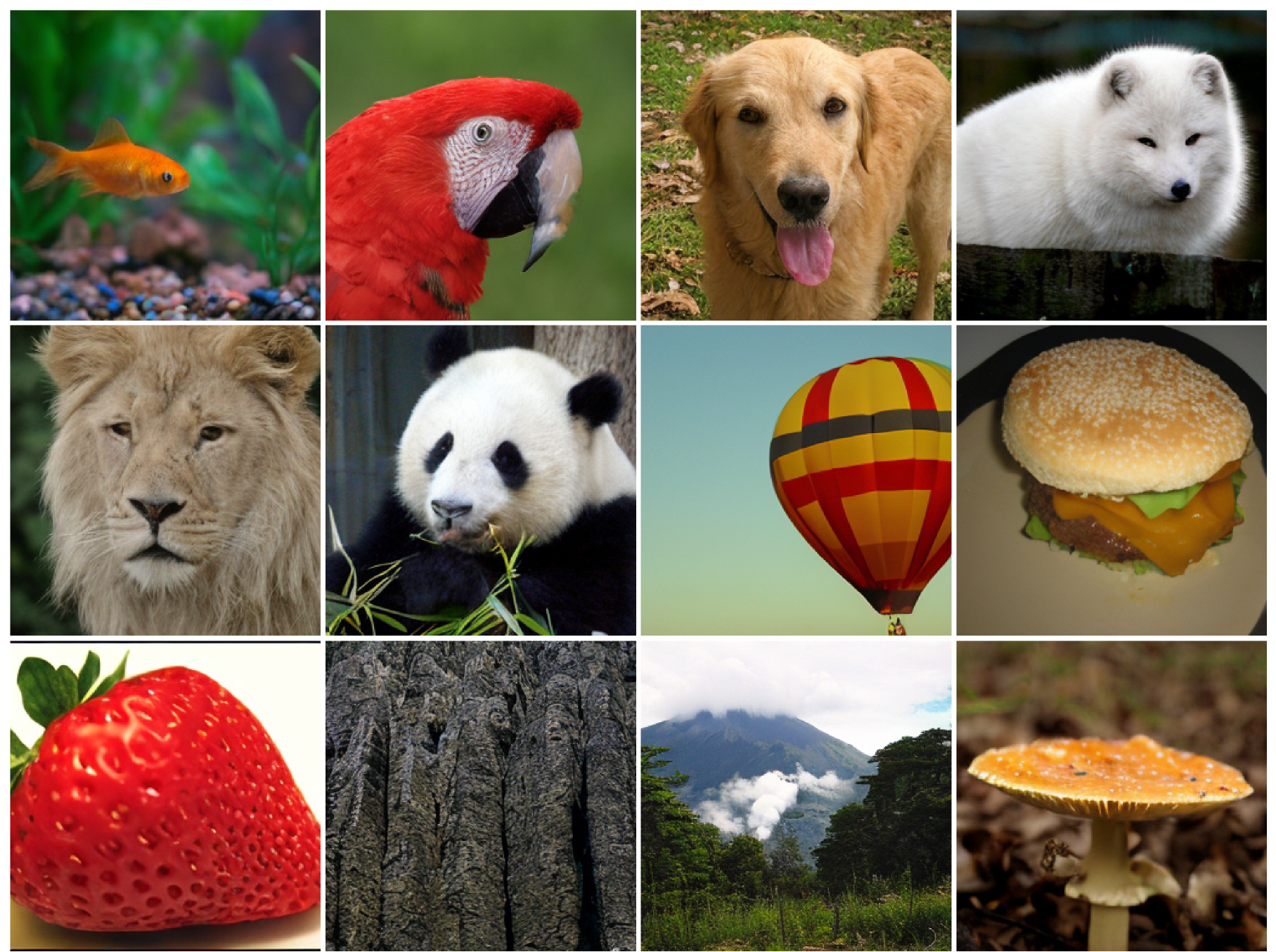}
        \caption{After \drl (ours)}
        \label{fig:teaser_finetuned}
    \end{subfigure}
        \vspace{-0.1cm}
    \caption{Seed-matched samples from REPA before and after \drl post-training with CFG${}=1$.
    Our \drl model produces sharper, more coherent images even without using any CFG.
    }
    \label{fig:teaser}
    \vspace{-0.5cm}
\end{figure}

\section{Introduction}

Flow-based models have become the dominant paradigm for generative modeling in continuous domains due to their simulation-free, regression-based training objectives~\citep{lipman2022flow,albergo2023stochastic,ho2020denoising,song2020score}.
In practice, however, these models are rarely trained \emph{only} with these losses. Instead,
they follow a multi-stage pipeline: (1) a flow or score matching (FSM) objective is used to fit a base model
$p_{\mathrm{base}}$ to a data distribution $q$; and (2) reinforcement learning (RL) is used to tilt $p_{\mathrm{base}}$ towards desirable regions under a reward $r(x)$.

The standard motivation for RL post-training is that it enables fitting a model to the implicit
distribution $p^*(x)\!\propto\!p_{\mathrm{base}}(x)\exp(r(x))$ for which no data exist but which is easy to specify via a reward~\citep{xu2023imagereward,black2023training,fan2023dpok,schuhmann2022laion}.
Nevertheless, in practice, RL post-training is also used to  enforce properties that are already present in the data, like visual realism, coherent object structure~\citep{wallace2024diffusion,domingoenrich2025adjoint}, or physically correct motion~\citep{ye2025dataregularizedreinforcementlearningdiffusion,liu2025improving}.
This is both puzzling---why can the model not learn these real-world properties from the
data directly but recover them with RL?---and undesirable, as preference
data, which is the norm for training the rewards in these applications, is  expensive to collect and conflates the real-world properties we want to preserve with subjective concerns.\looseness-1

This raises two questions that we focus on in this paper:
\textbf{(1)} Why do matching-based objectives fail to capture properties that are present in the data but that RL helps recover?
\textbf{(2)} Can we use RL to correct these failures without relying on preference-based rewards or data?

First, we argue that the fact that RL recovers properties already present in
$q$ points to the matching objective, rather than data or capacity, as a
plausible bottleneck. To build intuition, we study a simplified setting with
probability-flow ODEs and ask when FSM losses control the property gap
$\bigl|\mathbb{E}_{p_{\mathrm{base}}}[r] - \mathbb{E}_q[r]\bigr|$ on a property
$r$ of interest. The guarantees we find are weak. In the worst case, none
exist: the FSM loss is measured under the interpolation marginals $q_t$, but
sampling is governed by the model's rollout distribution $p_t$
(\Cref{prop:sft_no_certificate}). Under regularity assumptions a guarantee does
exist. However, it degrades with the reward's $\ell_2$ Lipschitz constant
(\Cref{prop:lipschitz_reward_certificate}), which we expect to be large for
many properties of interest. Although these bounds are worst-case, they suggest how RL
can help: by construction it avoids both obstructions, evaluating the model on
its own trajectories rather than $q_t$ and following the reward landscape
directly instead of the geometry of velocity space.

However, RL is only useful if we optimize a reward
that is a good proxy for the properties of $q$ we
want to preserve, and a priori we do not have one.
To this end, motivated by the analysis above, we introduce \textbf{Discriminator-Guided
RL (\drl)}. \drl trains a discriminator to estimate
the density ratio between the data and
the model's distributions in a pretrained
self-supervised learning (SSL) representation space. The logit of this
discriminator then serves as the reward for KL-regularized RL.
Working in an SSL space is the central design choice:
it restricts the reward to use discrepancies visible
through the representation, both making
density-ratio estimation tractable and confining
corrections to semantically meaningful axes---all
without ever relying on preference data.

We apply \drl to state-of-the-art image generation architectures, SiT~\citep{ma2024sit}, JiT~\citep{li2025jit}, REPA~\citep{yu2025repa}, and RAE~\citep{zheng2026rae}, and show that even with a simple linear discriminator, \drl delivers large improvements over the base model across the board (\Cref{fig:teaser}), measured by Fr\'echet distances in multiple feature spaces. These gains directly translate into better scores under held-out human-preference rewards, despite \drl never being trained on such rewards.
Beyond improving the base model, \drl also strengthens the post-training pipeline that follows. Applying preference-based RL (PRL) on top of \drl rather than the base model yields a better reward--distortion Pareto frontier, since \drl absorbs the corrections recoverable from $q$ and leaves PRL to handle genuine subjective preferences. Full experimental details are in \Cref{app:experimental_setup}.

Overall, our results suggest viewing RL post-training of flow-based models not only as a way to optimize external preferences, but also as a complementary mechanism for recovering structure in the training data that is imperfectly captured by standard matching-based objectives.

\section{Preliminaries}
\label{sec:prelim}
\paragraph{Flow and Score Matching.}
\label{sec:diffusion}
Flow and diffusion models generate data by
transporting a simple base distribution, typically
Gaussian noise, to the data distribution $q$
\citep{ho2020denoising,song2020score,lipman2022flow,albergo2023stochastic,liu2022flow}.
A standard training construction introduces
independent random variables $X_1 \sim q$ and
$X_0 \sim \mathcal{N}(0, I)$, and defines the
auxiliary interpolation
\citep{albergo2023stochastic,lipman2022flow,liu2022flow}
\begin{equation}
X_t = \alpha(t) X_1 + \beta(t) X_0, \quad \text{where } t \in [0,1],\ \alpha(0)\!=\!0,\ \beta(0)\!=\!1,\ \alpha(1)\!=\!1,\ \beta(1)\!=\!0.
\label{eq:interpolation}
\end{equation}
This
induces a family of marginals $q_t$ interpolating
between $q_0\!=\!\mathcal{N}(0,I)$ and $q_1\!=\!q$.
Sampling from $q$ then reduces to learning the
velocity field
$v_t(x)\!:=\!\mathbb{E}[\dot{\alpha}(t)X_1+\dot{\beta}(t)X_0\mid X_t\!=\!x]$
or the score field
$s_t(x)\!=\!\nabla_x \log q_t(x)$, since the SDE
\begin{equation}
dX_t =
\Bigl[
v_t(X_t)
+ \tfrac{1}{2}\sigma(t)^2 s_t(X_t)
\Bigr]dt
+ \sigma(t)\, dW_t,
\qquad
X_0 \sim \mathcal{N}(0,I),
\label{eq:ode_sde}
\end{equation}
has marginals $q_t$ for any noise schedule
$\sigma(t)$
\citep{song2020score,lipman2022flow,albergo2023stochastic}; see \Cref{app:proofs:sde_marginals}.
Both fields are conditional expectations, and so
admit simulation-free $\ell_2$ regression targets
\citep{hyvarinen2005estimation,vincent2011connection,ho2020denoising,song2020score,lipman2022flow,albergo2023stochastic}:
conditional flow matching (CFM) and denoising score
matching (DSM) use, respectively,
\begin{equation}
\label{eq:dsm}
\E\!\left[\left\|v_\theta(X_t,t)-\bigl(\dot{\alpha}(t)X_1+\dot{\beta}(t)X_0\bigr)\right\|^2\right]
\quad\text{and}\quad
\E\!\left[\left\|s_\theta(X_t,t)+\bigl(X_t-\alpha(t)X_1\bigr)/\beta(t)^2\right\|^2\right],
\end{equation}
with expectation over $(t,X_0,X_1)$. Under
Eq. \eqref{eq:interpolation}, $v_t$ and $s_t$ are
recoverable from each other
\citep{karras2022elucidating,domingoenrich2025adjoint} (see \Cref{app:proofs:score_velocity}),
so learning one suffices; we write
$\mathcal{L}_{\mathrm{FSM}}$ for either objective
when the distinction is unimportant.

\paragraph{RL Post-training for Flow Models.}
\label{sec:posttraining}
\label{sec:posttraining:rl}
Given a pretrained model $p_{\mathrm{base}}$ and a reward $r\!:\!\mathcal{X}\!\to\!\mathbb{R}$, KL-regularized RL \citep{ziebart2010modeling} aims to move probability mass toward high-reward regions without drifting too far from the base model by solving
\[
\max\nolimits_p \;
\mathbb{E}_{x \sim p}[r(x)]
- (1/\lambda)\;\KL{p}{p_{\mathrm{base}}},
\]
or equivalently minimizing the reverse KL, $\KL{p}{p^*}$, where $p^*(x) \propto \exp(\lambda r(x))\, p_{\mathrm{base}}(x)$ and
$\lambda > 0$ is a hyper-parameter controlling the reward-KL trade-off.

Unfortunately, directly optimizing this objective is infeasible for flow models as the endpoint densities required by the KL term are unreliable. 
Instead, the standard approach is to use a nonzero noise schedule $\sigma(t)$ in the sampling SDE (Eq. \eqref{eq:ode_sde}) and replace the endpoint regularization by a KL over path distributions \citep{fan2023dpok,domingoenrich2025adjoint}, which can be analytically evaluated using Girsanov's theorem \citep{girsanov1960transforming,domingoenrich2025adjoint}. If $\mathbb{P}_\theta$ and $\mathbb{P}_{\mathrm{base}}$ denote the trajectory distributions induced by the fine-tuned and base models under the same training SDE, the new objective is then
\begin{equation}
\mathcal{L}_{\mathrm{RL}}(\mathbb{P}_\theta)
:=
- \mathbb{E}_{x \sim p_\theta}[r(x)] +
(1/\lambda)\;\KL{\mathbb{P}_\theta}{\mathbb{P}_{\mathrm{base}}}.
\label{eq:rl_path_obj}
\end{equation}
If training uses a specific, so-called memoryless $\sigma(t)$, the minimizer of Eq. \eqref{eq:rl_path_obj} has the same endpoint as the original KL-regularized objective, and admits sampling after training with any noise schedule \citep{domingoenrich2025adjoint}. This objective can then be optimized with standard policy gradient methods such as REINFORCE, PPO, or GRPO \citep{williams1992simple,black2023training,fan2023dpok,schulman2017proximal,shao2024deepseekmath,liu2025flowgrpotrainingflowmatching}, or with adjoint-based methods if the reward is differentiable \citep{domingoenrich2025adjoint}; details are deferred to \Cref{app:adjoint_matching}.

\section{Motivation: Understanding the Limitations of Flow and Score Matching}
\label{sec:limitations}

To motivate our method, we first ask why FSM losses might fail to capture properties of the data distribution that RL is able to recover.
With infinite capacity and perfect optimization, the minimizer of either matching loss should reproduce the data distribution $q$ exactly.
Therefore, a natural starting point is to consider how these losses behave in the setting where optimization is only approximate.
We analyze the ODE-based formulation for simplicity.

Let $r\!:\!\X\!\to\![0,1]$ quantify a property of interest in the data---say whether an object in an image is well formed.
A model that preserves this property adequately satisfies $\E_{p_v}[r]\!\approx\!\E_q[r]$.
Writing $\mathcal{L}_{\mathrm{FSM}}$ for either matching objective and $v^*$ for its
minimizer (with overloaded notation for the score), we ask: when does
$\bigl|\mathcal{L}_{\mathrm{FSM}}(v)-\mathcal{L}_{\mathrm{FSM}}(v^*)\bigr|\le\varepsilon$
imply $\bigl|\E_{p_v}[r]-\E_q[r]\bigr| \to 0$ as $\varepsilon\to 0$, and at what rate?\looseness-1

Our first result shows that, in the worst case, no such rate exists: low FSM loss can be arbitrarily uninformative about the reward gap.

\ifmeta
\begin{wrapfigure}[20]{r}{0.40\textwidth}
    \vspace{-1.1em}
    \centering
    \includegraphics[width=\linewidth]{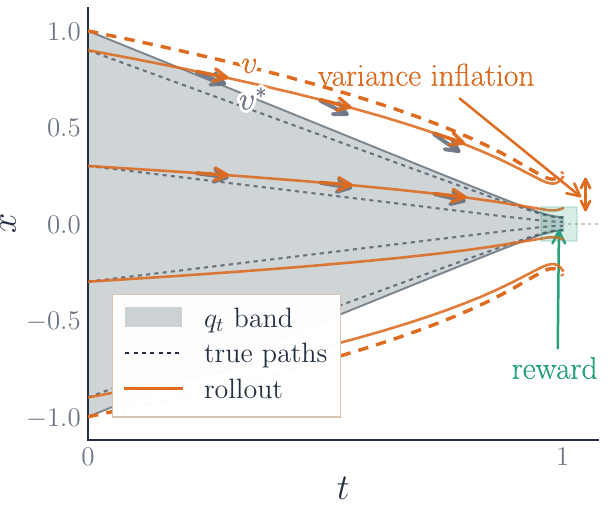}
    \vspace{-1.1em}
    \caption{Distribution shift. Rollout paths (orange) drift off the noising band $q_t$ (gray) and miss the high-reward region at $t{=}1$.
    The further from $q_t$, the larger the error of true field (gray) \textit{vs.}\ learned field (orange).
    Construction in \Cref{app:prop:ode_general_no_certificate}.}
    \label{fig:distribution_shift_main}
\end{wrapfigure}
\fi
\begin{proposition}[No reward certificate from FSM]
\label{prop:sft_no_certificate}
Fix a (sufficiently regular) $q$ and a bounded reward $r\!\in\![0,1]$ with $r(x)\le\eta$ on some region $B\subset\X$.
For every $\varepsilon>0$ and every $\delta>0$ there is a velocity field $v$ and a score field $s$, whose training error under $q_t$ is at most $\varepsilon$ uniformly in $t$, but whose probability-flow ODE endpoint law $p$ satisfies $\E_p[r]\le\eta+\delta$, regardless of $\E_q[r]$.
A precise statement and proof is given in \Cref{app:prop:ode_general_no_certificate}.
\end{proposition}

The construction exploits a train--test mismatch.  An early error pushes the
trajectory into regions of low $q_t$-mass, so subsequent errors occur at states
unseen during training and compound along the rollout
(\Cref{fig:distribution_shift_main}). The same pathology motivates DAgger in
behavioral cloning~\citep{ross2011reduction}, with $q_t$ and $p_t$ playing the
roles of expert and learner.

This counterexample is, however, adversarial. Under regularity assumptions on
$q$ and $v^*$ a quantitative reward guarantee does exist --- though, as we
will see, a loose one.

\begin{proposition}[Reward certificate under uniform velocity control]
\label{prop:lipschitz_reward_certificate}
Let $r\!:\!\X\!\to\![0,1]$ be $L_r$-Lipschitz in that $|r(x)-r(y)|\le L_r\|x-y\|$ for all $x,y$, and let $v^*(\cdot,t)$ be $L_v$-Lipschitz in $x$ uniformly in $t\!\in\![0,1]$.
If $\sup_{t,x}\|v(x,t)-v^*(x,t)\|\le\varepsilon$, then the endpoint laws $p$ and $q$ of the probability-flow ODE satisfy
\begin{equation}
\label{eq:lipschitz_certificate}
\bigl|\E_p[r]-\E_{q}[r]\bigr|
\;\le\;
\varepsilon L_r\,((e^{L_v}-1)/L_v)\,.
\end{equation}
This dependence on $L_r$, $L_v$, and $\varepsilon$ is essentially tight, and an on-policy variant under the rollout marginals $p_t$ holds as well; see \Cref{app:lipschitz_certificate} for a formal statement and proof.
\end{proposition}
\ifmeta
\begin{wrapfigure}[15]{l}{0.4\textwidth}
    \vspace{-0.9em}
    \centering
    \includegraphics[width=\linewidth]{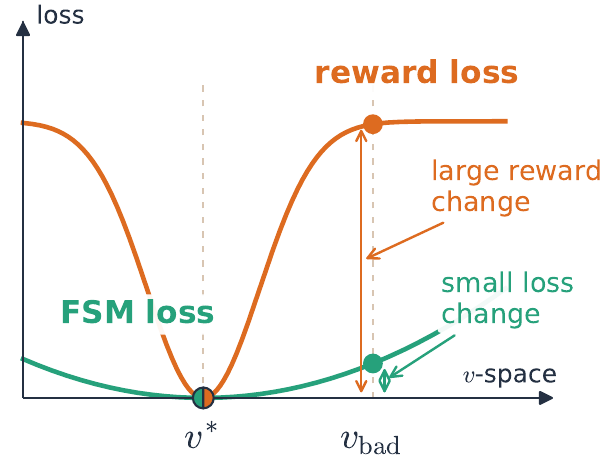}
    \vspace{-0.8em}
    \caption{Geometry mismatch. A small step in velocity space can produce large reward changes.}
    \label{fig:geometry_mismatch_main}
    \vspace{-0.9em}
\end{wrapfigure}
\fi
The bound factorizes into three terms: the velocity error $\varepsilon$, which FSM directly minimizes; the factor $(e^{L_v}\!-\!1)/L_v$, which controls how errors compound as the ODE is integrated from $t{=}0$ to $1$; and the reward Lipschitz constant $L_r$, which measures how sharply $r$ responds to small $\ell_2$ changes in the sample.

The last term exposes a second, more practically relevant obstruction: FSM
controls errors in velocity space, while the reward is a function of samples,
and when the two geometries are misaligned small velocity errors can map to
large reward errors (\Cref{fig:geometry_mismatch_main}).
The constant $L_r$
quantifies this: rearranged, the bound says that to achieve a reward gap of size
$\delta$ we need $\varepsilon \approx \delta/L_r$, so the sharper the
reward, the smaller the velocity error must be. For rewards of practical
interest, this is unfavorable. For example, in pixel space, whether a hand or animal face
looks well-formed can flip under a few edge pixels, making $L_r$ effectively
very large and the required $\varepsilon$ correspondingly small. Similar
phenomena will occur in any representation space not directly aligned with
the reward.

For matching losses, this is doubly unfavorable as the gradient
signal-to-noise ratio degrades precisely in the small-$\varepsilon$
regime the bound demands. With $\xi := Y - v^*(X_t,t)$ and
$\E[\xi\mid X_t,t]=0$, the per-sample gradient of the CFM loss splits as
\[
g_\theta
= 2 J_\theta^{\top}(v_\theta - v^*)
- 2 J_\theta^{\top}\xi,
\qquad J_\theta := \nabla_\theta v_\theta(X_t,t).
\]
The first term is the optimization signal and vanishes as
$v_\theta \to v^*$; the second is irreducible noise from the
conditional variance of the regression target and does not. The
achievable $\varepsilon$ can therefore plateau well above what
Eq. \eqref{eq:lipschitz_certificate} requires.
A similar decomposition can be derived for score matching.
\looseness-1

\paragraph{RL's edge.}
These obstructions suggest how RL helps.
First, it evaluates the model on the samples it
generates, so there is no train–test mismatch.
Second, and more importantly, with a reward in
hand, we can use the reward landscape directly,
focusing optimization on the directions of
reward improvement and bypassing the geometry
of velocity or score space.
Moreover, while generally loose in practice, when the reward is such that the KL-regularized optimum is the desired target $q$ and the noise schedule is fixed,
the RL objective controls the property gap directly: a Pinsker-style bound~(\Cref{app:rl_reward_bound}) gives
\[
\left|\mathbb{E}_p[r]-\mathbb{E}_q[r]\right|
\le
\sqrt{\frac{\lambda}{2}
\left(\mathcal{L}_{\rm RL}(p)-\mathcal{L}_{\rm RL}(q)\right)}.
\]
This verifies that we optimize the right quantity in the idealized setting.
The implication is that even with perfect data, RL post-training can be
beneficial when a good reward is available: the reward provides an optimization signal directly aligned with the property of interest, whereas FSM must recover that property indirectly from the regression target alone.

\ifmeta
\begin{wrapfigure}[16]{r}{0.40\textwidth}
    \vspace{-1.0em}
    \centering
    \includegraphics[width=\linewidth]{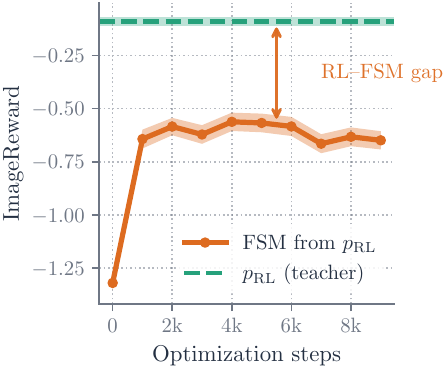}
    \vspace{-0.8em}
   \caption{Distillation gap on Stable Diffusion 1.5: a score-matching student trained on samples from an RL teacher plateaus well below the teacher's reward. }
   \vspace{5.1em}
    \label{fig:distillation_gap}
\end{wrapfigure}
\fi

To show that these distinctions matter in practice, we compare RL training against supervised distillation from an RL teacher.
Starting from Stable Diffusion 1.5~\citep{rombach2022high}, we first train an RL teacher on ImageReward using adjoint matching~\citep{domingoenrich2025adjoint}, and then use its samples to fine-tune another Stable Diffusion student using score matching.
If score matching could recover the reward-relevant properties learned by RL, the student should match the teacher.
In practice it does not: the student stalls well below the teacher's reward and fails to close the gap (\Cref{fig:distillation_gap}).
See \Cref{app:distillation_details} for the full setup.
\looseness-1

\section{Method: Discriminator-Guided RL}
\label{sec:implications_fixing_sft}
Unfortunately, RL is only beneficial if we have a reward that captures the aspects of $q$ that FSM may have failed to learn. Finding such a reward is non-trivial.
The standard approach, learning a reward from human preferences, is only incidentally aligned with this goal. While preference rewards can correlate with $q$-properties (\textit{e.g.}, annotators prefer well-formed faces), they can also drag optimization along orthogonal axes like aesthetic appeal, so optimizing them is not necessarily desirable.
Moreover, preference data is expensive to collect, and hard to define for many domains beyond images and video. We would instead like a reward derived directly from $q$.
To this end, we introduce Discriminator-Guided RL (\drl).

\paragraph{From samples to rewards.}
The starting point of \drl is the observation that, in KL-regularized RL, there is a reward whose optimum is exactly $q$. Recall from \Cref{sec:posttraining} that, at $\lambda\!=\!1$, the unique maximizer of the KL-regularized objective is $p^*(x)\!\propto\!\exp(r(x))\,p_{\mathrm{base}}(x)$. Forcing $p^*\!=\!q$ identifies the ideal reward, up to additive constants, as the log density ratio between target and reference,
\begin{equation}
  r^*(x)=\log\frac{q(x)}{p_{\mathrm{base}}(x)}.
  \label{eq:ideal_reward}
\end{equation}
This ratio is intractable in closed form, but it can be estimated from samples.
Training a discriminator $D\!:\!\X\!\to\!(0,1)$ to separate $q$ from $p_{\mathrm{base}}$ under the logistic loss has a well-known optimum, $D^*(x)=q(x)/(q(x)+p_{\mathrm{base}}(x))$.

\ifmeta
\begin{wrapfigure}[20]{r}{0.50\textwidth}
\vspace{-1.4em}
\begin{minipage}{\linewidth}
\begin{algorithm}[H]
\caption{Discriminator-Guided RL (\drl)}
\label{alg:drl_pipeline}
\footnotesize
\vspace{0.1em}
\centerline{\rule{0.04\linewidth}{0.3pt}\hspace{0.4em}%
  \textbf{Stage 1: Reward Estimation}%
  \hspace{0.4em}\rule{0.04\linewidth}{0.3pt}}
\begin{algorithmic}[1]
\REQUIRE Samples $x \sim q$, encoder $\phi$, base model $p_{\mathrm{base}}$
\REPEAT
    \STATE Sample $x_{\mathrm{real}} \sim q$,\; $x_{\mathrm{fake}} \sim p_{\mathrm{base}}$
    \STATE $\ell_{\mathrm{disc}}(\psi) \leftarrow -\log D_\psi(\phi(x_{\mathrm{real}})) - \log\bigl(1{-}D_\psi(\phi(x_{\mathrm{fake}}))\bigr)$
    \STATE Update $\psi$ by descending $\ell_{\mathrm{disc}}(\psi)$
\UNTIL{convergence}
\STATE Define $\hat r(x) \leftarrow \mathrm{logit}\,D_\psi(\phi(x))$
\end{algorithmic}
\vspace{0.2em}
\centerline{\rule{0.04\linewidth}{0.3pt}\hspace{0.4em}%
  \textbf{Stage 2: KL-Reg.\ RL via Adjoint Matching}%
  \hspace{0.4em}\rule{0.04\linewidth}{0.3pt}}
\begin{algorithmic}[1]
\REQUIRE Reward $\hat r$, base velocity $v_{\mathrm{base}}$, KL weight $\lambda$
\REQUIRE Schedule $\sigma(t)$, grid $0{=}t_0{<}\cdots{<}t_K{=}1$
\STATE Initialize $v_\theta \leftarrow v_{\mathrm{base}}$
\REPEAT
    \STATE Sample $\{X_k\}$ from $dX_t = [2v_\theta - \tfrac{1}{t}X_t]\,dt + \sigma(t)\,dB_t$
    \STATE $\tilde a_1 \leftarrow -\lambda\,\nabla_x \hat r(X_K)$;\;
    \STATE Solve backward
    \STATE\quad $\dot{\tilde a}_t = -\tilde a_t^\top \nabla_x[2v_{\mathrm{base}}(X_t,t) - \tfrac{1}{t}X_t]$
    \STATE Set $L_{\mathrm{AM}}(\theta) \leftarrow$
    \STATE\quad $\tfrac12 \sum_{k=0}^{K} \bigl\| \tfrac{2}{\sigma(t_k)}\bigl(v_\theta {-} v_{\mathrm{base}}\bigr) + \sigma(t_k)\tilde a_k \bigr\|^2$
    \STATE Update $\theta$ by descending $L_{\mathrm{AM}}(\theta)$
\UNTIL{convergence}
\STATE \textbf{return} $v_\theta$
\end{algorithmic}
\end{algorithm}
\end{minipage}
\vspace{-1.0em}
\end{wrapfigure}
\fi

Hence, its logit recovers $r^*$. This motivates the reward estimator
\begin{equation}
  \hat r(x) := \log\frac{D(x)}{1-D(x)}
  \label{eq:reward_estimator}
\end{equation}

In practice, we parametrize  the logit $\hat r$ directly and use a sigmoid to define the discriminator.

\paragraph{Using representations.}
Estimating $r^*$ directly in the flow output space is, however, both statistically hard and semantically unreliable. For example, a discriminator might separate $q$ from $p_{\mathrm{base}}$ using artifacts that have no bearing on the relevant properties, achieving low classification error without producing a useful estimate of the density ratio (see \parheadref{sec:exp_ablations}). We therefore constrain the reward to a pretrained representation space. Given a frozen encoder $\phi\!:\!\X\!\to\!\mathcal{Z}$, we set $\hat r(x)\!=\!h(\phi(x))$ for a learned head $h$. This both reduces the dimensionality of the estimation problem and restricts the discriminator to $\phi$-visible structure.

Mathematically, this restriction means that \drl will generally never be able to target $q$ exactly. Nevertheless, its new target has a simple and intuitive characterization. It is the solution to the following constrained optimization problem:
\begin{equation}
  \min\nolimits_p\;\KL{p}{p_{\mathrm{base}}}
  \qquad\text{subject to}\qquad
  p^\phi = q^\phi,
  \label{eq:feature_correction}
\end{equation}
where $p^\phi,q^\phi$ denote the pushforwards of $p,q$ under $\phi$ (see \Cref{prop:feature_correction} for a formal statement and proof). In words, \drl makes the smallest KL change to $p_{\mathrm{base}}$ that aligns its representation-space distribution with that of the target, while leaving any variation invisible to $\phi$ unchanged; choosing $\phi$ thus chooses which aspects of $q$ \drl is allowed to correct.

Finally, when the learned reward is imperfect, we recover a feature-space test-function bound that augments the standard KL-regularized RL suboptimality bound with an additional term capturing the expected discrepancy between the ideal target in \Cref{eq:feature_correction} and the target implied by the learned reward. We defer the precise statement and proof to \Cref{app:feature_bound_proof}.

\ifmeta
\paragraph{The \drl pipeline.}
These two ideas combine into the \drl algorithm specialized to flow models summarized in \Cref{alg:drl_pipeline}. \emph{Stage~1} trains a discriminator in a frozen representation space to distinguish $q$ from $p_{\mathrm{base}}$ and defines the reward as its logit. \emph{Stage~2} fine-tunes the base model with KL-regularized RL under that reward. To exploit reward gradients, we instantiate the RL stage with adjoint matching~\citep{domingoenrich2025adjoint}, a state-of-the-art RL algorithm for flow models.
Further implementation details for both stages are deferred to \Cref{app:additional_implementation_details}.
A short description of adjoint matching is given in \Cref{app:adjoint_matching}.
\fi

\section{Experiments}
\label{sec:experiments_main}

\ifmeta
\begin{figure}[t]
    \vspace{-0.6em}
    \centering
    \includegraphics[width=1.0\textwidth,trim={1cm 0.4cm 0.2cm 0}]{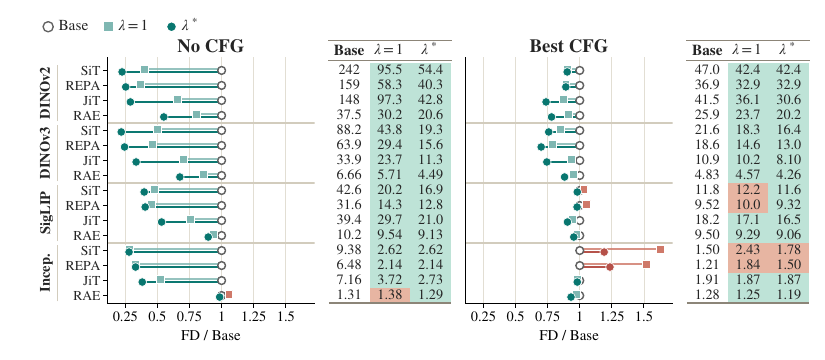}
    \caption{\textbf{Distribution alignment.} Each row shows \drl's FD normalized by the corresponding Base FD (vertical line); lower is better. Squares show $\lambda{=}1$ and filled circles show the value optimized over $\lambda$. Raw FD values in tables: \textcolor{tableTeal}{\rule{0.7em}{0.7em}}\,indicates improvement, \textcolor{tableRedDark}{\rule{0.7em}{0.7em}}\,indicates degradation over Base. Most markers fall left of the Base line, showing broad alignment gains, with the most significant gains in the no-CFG setting.}
    \label{fig:alignment}
    \vspace{-0.4cm}
\end{figure}
\fi

We validate \drl by using it to fine-tune four ImageNet-pretrained flow models: SiT~\citep{ma2024sit}, JiT~\citep{li2025jit}, REPA~\citep{yu2025repa}, and RAE~\citep{zheng2026rae}, covering latent- and pixel-space architectures. Notably, RAE and REPA already use SSL representations during pretraining as the latent space and regularization respectively.
Unless noted otherwise, the discriminator is a class-conditional linear projection head~\citep{miyato2018cgansprojectiondiscriminator} on frozen DINOv2-Large~\citep{oquab2024dinov2} features, trained for $10$k steps. The RL stage is $3$k steps of adjoint matching~\citep{domingoenrich2025adjoint}. Importantly, this is only a small fraction of the $1\text{M}{+}$ steps typically used to pre-train these models. Full description of every experimental setup, evaluation protocol, and hyperparameter is in \Cref{app:experimental_setup}.
Finally, we run adjoint matching with a new local-linear integrator we introduce for the memoryless SDE required for RL; we found it essential for stable training and view it as an important contribution of independent interest but leave its full description in App.\ \Cref{app:local_linear_integrator}.\looseness-1

\labeledparagraph[distributional alignment]{\drl Reduces the Distributional Gap Along Semantically Meaningful Directions.}{sec:exp_alignment}
We first investigate whether \drl reduces the distributional gap to the data along semantically meaningful directions by computing Fr\'echet Distance (FD)~\citep{heusel2017gans} in four feature spaces: DINOv2~\citep{oquab2024dinov2}, DINOv3~\citep{simeoni2025dinov3}, SigLIP~\citep{zhai2023sigmoid}, and InceptionV3~\citep{szegedy2016rethinking}.
We consider two settings: the theoretically motivated unit $\lambda$ and a tuned $\lambda^\star$ value selected over $1, 5,10,20,40$. 
We consider values larger than $1$ because they often improve RL optimization in practice.
To stabilize training at larger $\lambda$, we found it helpful to add an $R_1$ penalty~\citep{mescheder2018ganconvergence}, which penalizes the discriminator's input gradient on real samples and smooths the reward landscape (see \parheadref{sec:exp_ablations} for more details).

\Cref{fig:alignment} plots FD ratios (\drl/Base) for each (model, feature) pair, alongside their raw FD values with and without tuned CFG.
Without CFG, \drl substantially improves distribution alignment: the tuned $\lambda^\star$ improves all $16$ model--feature pairs, and $\lambda=1$ improves all but one.
The gains are especially large in DINOv2 and DINOv3, where FD often drops by more than half.
With the best CFG setting, the Base models are already much closer to the data, but \drl still further improves alignment in most cases.
In particular, $\lambda^\star$ improves FD in $14/16$ pairs, with consistent gains in DINOv2, DINOv3, and SigLIP.
The only clear exceptions are InceptionV3 for SiT and REPA under CFG, where FD increases slightly relative to Base.
Across both settings, tuning $\lambda$ consistently improves over the unit-weight setting, indicating that pushing past the theoretical value pays off in practice.\looseness-1

\ifmeta
\begin{wrapfigure}[14]{r}{0.50\textwidth}
    \vspace{-0.8\baselineskip}
    \centering
    \includegraphics[width=\linewidth]{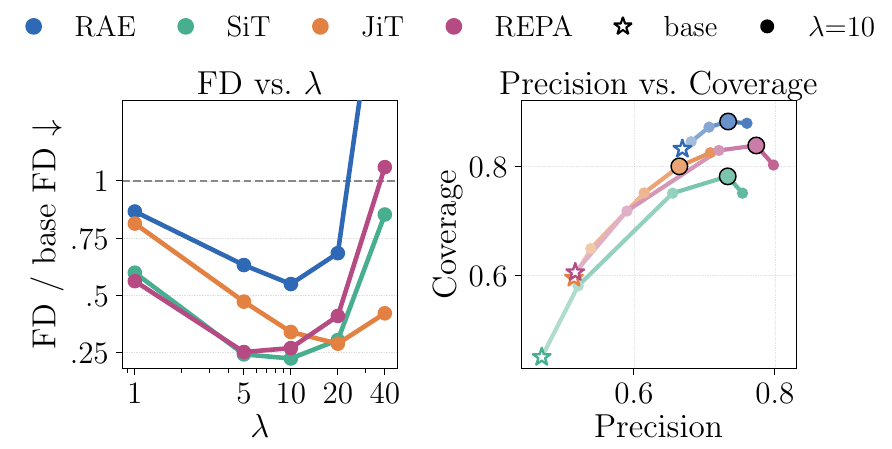}
    \caption{\textbf{FD vs.\ $\lambda$ and Precision--Coverage in DINOv2-L.} FD decreases through $\lambda{=}5{-}10$ and degrades at high $\lambda$ as coverage drops.}
    \label{fig:lambda_baseline_and_sweep}
    \vspace{-0.8\baselineskip}
\end{wrapfigure}
\fi
To further investigate the role of $\lambda$, \Cref{fig:lambda_baseline_and_sweep} plots FD against $\lambda$ in DINOv2 space alongside the corresponding precision--coverage trajectory. FD bottoms out around $\lambda \in [5, 10]$ and then degrades; coverage drops at $\lambda{=}20$, and by $\lambda{=}40$ (not pictured; see \Cref{fig:lambda_sweep_1,fig:lambda_sweep_2}) the trajectory has collapsed off the frontier---the model overfits to the discriminator's reward signal. We adopt $\lambda{=}10$ as the default for all downstream experiments, since it sits at the knee of this trade-off.\looseness-1

Additional metrics and $\lambda$ ablations for each model--feature pair are in \Cref{app:full_results}. Importantly, KD and FD$_{\mathrm{val}}$ follow the same trends despite not being used to tune $\lambda$.
\ifmeta
\begin{figure}[t]
    \centering
    \includegraphics[width=\textwidth]{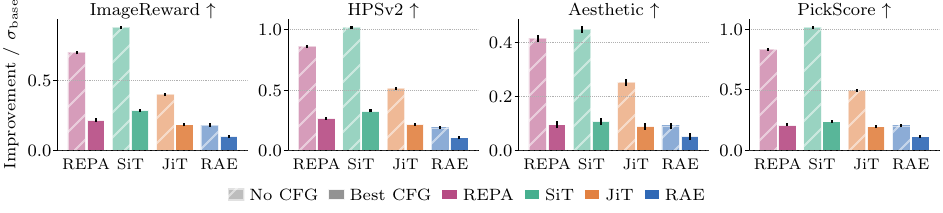}
    \caption{\textbf{Held-out preference reward gains.} Normalized improvement $(r_{\drl}-r_{\mathrm{base}})/\sigma_{\mathrm{base}}$ at $\lambda{=}10$ on four rewards, with and without CFG, where $\sigma_{\mathrm{base}}$ is the base-model reward standard deviation. \drl improves every reward without ever seeing preference data.}
    \vspace{-0.4cm}
    \label{fig:reward_improvement}
\end{figure}
\fi

\labeledparagraph[image-quality transfer]{\drl Improves Image Quality Without Preference Data. }{sec:exp_reward_transfer}
We next examine whether these distributional gains translate into perceptually better images. As proxies for image quality, we use four held-out preference reward models trained on human comparisons: ImageReward~\citep{xu2023imagereward}, PickScore~\citep{kirstain2023pickapic}, Aesthetics v2.5~\citep{discus0434_aesthetic_predictor_v2_5_2024}, and HPSv2~\citep{wu2023human}. \Cref{fig:reward_improvement} plots the per-model improvement on each reward, $(r_{\drl}-r_{\mathrm{base}})/\sigma_{\mathrm{base}}$, both without CFG and with best CFG. Without preference data or access to these rewards during training, \drl improves every reward on every architecture, with the largest gains for SiT and REPA and smaller but consistent gains for JiT and RAE.

These improvements are also evident visually.
The accompanying samples show fixed-seed outputs at $\lambda$ values of $1$, $10$, and $20$.
As $\lambda$ grows, shapes sharpen and global structure becomes more coherent, e.g., the spider's shape, the koala's face, and the car's front end all become noticeably more defined. 
We provide more samples and $\lambda$ values in App.\ \Cref{fig:am_samples_jit,fig:am_samples_sit,fig:am_samples_repa,fig:am_samples_rae}.\looseness-1

\ifmeta
\begin{wrapfigure}[24]{r}{0.50\textwidth}
    \vspace{-1.0\baselineskip}
    \centering
    \includegraphics[width=\linewidth]{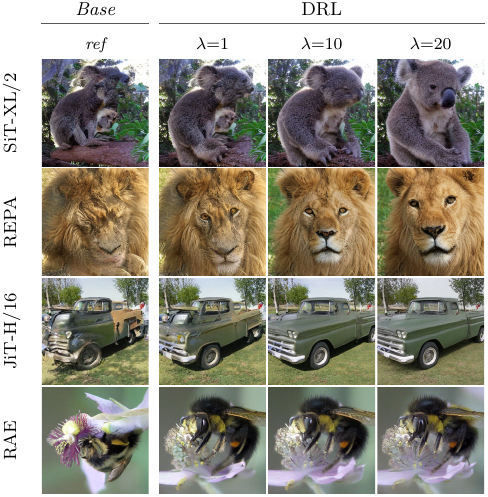}
    \caption{\textbf{Effect of $\lambda$.} Same noise and class label at $\lambda \in \{1, 10, 20\}$. Larger $\lambda$ produces sharper, more coherent samples while preserving content and composition.}
    \label{fig:lambda_sweep}
    \vspace{-0.8\baselineskip}
\end{wrapfigure}
\fi

\labeledparagraph[preference alignment]{\drl Provides a Better Foundation for Preference Alignment.}{sec:exp_reward}
\drl improves distribution matching, but aligning to genuine subjective
preferences still requires fine-tuning with preference-based RL (PRL). We argue that \drl also makes PRL itself more effective. PRL from the base model
is asked to do two things with one imperfect scalar reward: repair
distributional errors left by the generative model \emph{and} optimize
subjective preference. The result is a reward--drift trade-off: small
$\lprl$ leaves structural failures uncorrected; large $\lprl$ exploits
the proxy along nuisance directions like oversaturation or excessive
brightness~\citep{domingoenrich2025adjoint}. By handling part of the data
repair before preference optimization begins, \drl should ease this
trade-off and let PRL focus on genuinely subjective improvements.\looseness-1

\ifmeta
\begin{figure}[t!]
    \centering
    \makebox[\textwidth][c]{%
        \includegraphics[width=1.08\textwidth]{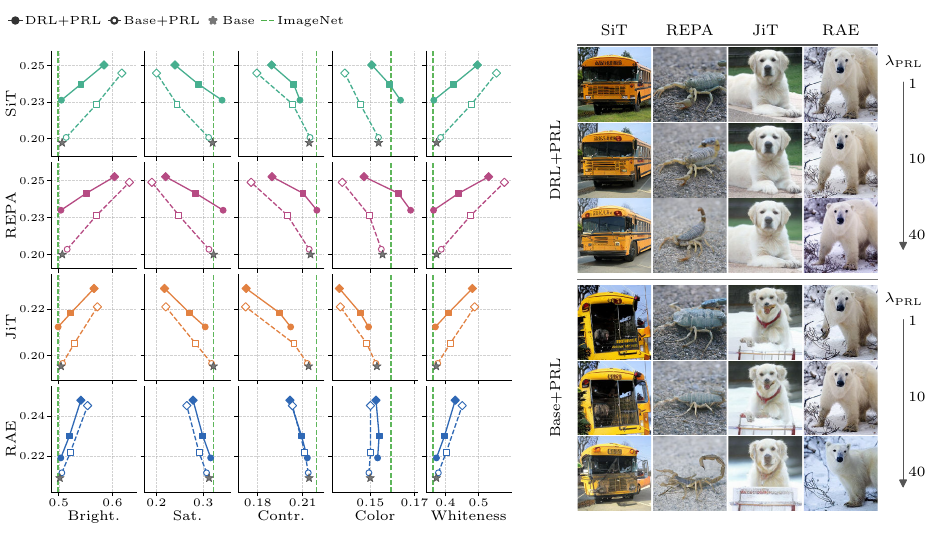}%
    }
    \caption{\textbf{\drl provides a better starting point for PRL.} Left: HPSv2 reward against five low-level statistics under PRL from the base model (dashed, hollow) or the \drl checkpoint (solid, filled); green dashed lines mark ImageNet reference statistics. Right: matched fixed-seed samples at $\lprl\!\in\!\{1,10,40\}$. Across statistics, DRL+PRL gives better reward--drift Pareto fronts than Base+PRL and less brightness/whiteness drift as PRL strength increases. See \Cref{app:setup_preference_rl} for details on how the statistics are computed.}
    \label{fig:pareto_hpsv2_cfg1}
    \label{fig:rl_samples_main}
    \vspace{-0.8em}
\end{figure}
\fi

We test this by running KL-regularized PRL from either the base model or
the \drl checkpoint. We train with ImageReward~\citep{xu2023imagereward}
at $\lprl\!\in\!\{1,10,40\}$ and evaluate on the held-out HPSv2
reward~\citep{wu2023human}, tracking five low-level statistics that
commonly drift under aggressive PRL: brightness, saturation, contrast,
colorfulness, and whiteness; see \Cref{app:setup_preference_rl} for details
on how these statistics are computed. \Cref{fig:pareto_hpsv2_cfg1} plots HPSv2
against each statistic. DRL+PRL shifts the reward--drift frontier
upward across all five axes, achieving higher HPSv2 than Base+PRL at
comparable drift. The pattern holds with other forms of guidance, and on other rewards, including ImageReward itself (\Cref{app:pareto}).\looseness-1

Samples from the model further support this conclusion.
In \Cref{fig:rl_samples_main}, Base+PRL at
low $\lprl$ leaves structural errors unresolved---the JiT dog has a
malformed face, the REPA scorpion a shell-like body, the SiT bus a
distorted chassis. Larger $\lprl$ corrects some of these, especially
for JiT and REPA, but at the cost of noticeably brighter and whiter
images. DRL+PRL avoids the trade-off: samples remain structurally
coherent and naturally colored across all $\lprl$. Per-model grids and
guided variants appear in
\Cref{fig:rl_samples_jit,fig:rl_samples_sit,fig:rl_samples_repa,fig:rl_samples_rae}.\looseness-1

\ifmeta
\begin{figure}[t]
    \centering
    \includegraphics[width=\textwidth]{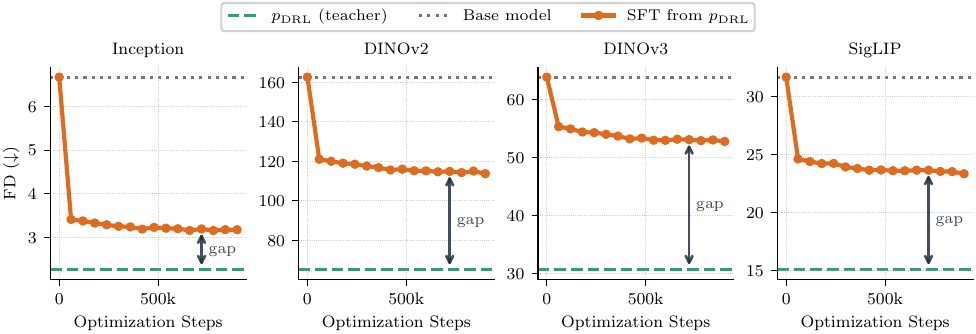}
    \caption{\textbf{\drl performance cannot be reached with flow matching alone.} FD ($\downarrow$) over optimization steps for a student trained with flow matching on samples from a \drl teacher ($\lambda{=}1$, no $R_1$), evaluated in four feature spaces. The student improves rapidly over the base model (grey, dotted) but plateaus well above the teacher (amber, dashed) across all four feature spaces.}
    \label{fig:distillation_from_drl}
\end{figure}
\fi

\labeledparagraph[DRL distillation]{\drl Performance Cannot Be Reached with Flow Matching Alone.}{sec:exp_distillation_from_drl}
In \Cref{sec:limitations} we argued that RL was needed because flow matching cannot faithfully
recover many properties of the data distribution.
To demonstrate that \drl genuinely learns aspects of the distribution that are hard to learn
with flow matching alone, we use a distillation experiment similar to the one in \Cref{sec:limitations}.
Concretely, we take the REPA model post-trained with \drl ($\lambda=1$, $R_1=0$) and use it as
a teacher by fine-tuning REPA on samples from the teacher.
If the features learned by \drl were learnable by the flow matching objective, we should see
no gap between the teacher model and the student model.
As \Cref{fig:distillation_from_drl} demonstrates, this is not the case. Despite training for over $900$k gradient
steps, and seeing more than $50$ million samples (over $150\times$ what RL sees and over $40\times$ the size of ImageNet),
the student model is not able to faithfully replicate the teacher.
This supports the argument from \Cref{sec:limitations} that the RL objective matters, rather than only the samples produced by the teacher.
We provide further details of the setup of this experiment in \Cref{app:setup_distillation_from_drl}.

\labeledparagraph[design choices]{Design Choices.}{sec:exp_ablations}
We close by ablating \drl's two main design choices---the discriminator
feature space and the discriminator architecture---on REPA. Full setup
is in \Cref{app:ablation_setups}.\looseness-1

\textit{Discriminator design}. We compare four discriminators---a linear head and an MLP-2 head on
frozen DINOv2 features, a fully fine-tuned DINOv2, and a DINOv2
architecture trained from scratch---at $\lambda{=}1$ with $R_1{=}0$, and
at $\lambda{=}10$ with $R_1\!\in\!\{0, 10^{-5}, 10^{-3}, 10^{-1}\}$.
\Cref{fig:disc_ablation} plots DINOv2-L FD across this sweep, with
representative samples along the top row. The
full $R_1$ sweep at $\lambda{=}1$ is in \Cref{app:disc_ablation_full}.\looseness-1

Pretrained features are essential: the from-scratch discriminator
underperforms despite reaching $95\%$ validation accuracy, and even
fine-tuning DINOv2 trails frozen features in the theoretically motivated
setting ($\lambda{=}1$, $R_1{=}0$). At $\lambda{=}10$ without $R_1$, the
model collapses to high-reward regions; a small $R_1$ not only stabilizes
training but improves over $(\lambda{=}1, R_1{=}0)$ in both FD and visual
quality.
Curiously, while $R_1$ helps at $\lambda{=}10$, it degrades performance at $\lambda{=}1$ (\Cref{fig:disc_ablation_full}), suggesting its role here differs from the vanishing-gradient role it plays in standard GANs~\citep{mescheder2018ganconvergence}.

Finally, since the MLP and linear heads perform similarly, we prefer the linear head for simplicity.\looseness-1

\ifmeta
\begin{wrapfigure}[21]{r}{0.45\textwidth}
    \centering
    \vspace{-1.1em}
    \includegraphics[width=\linewidth]{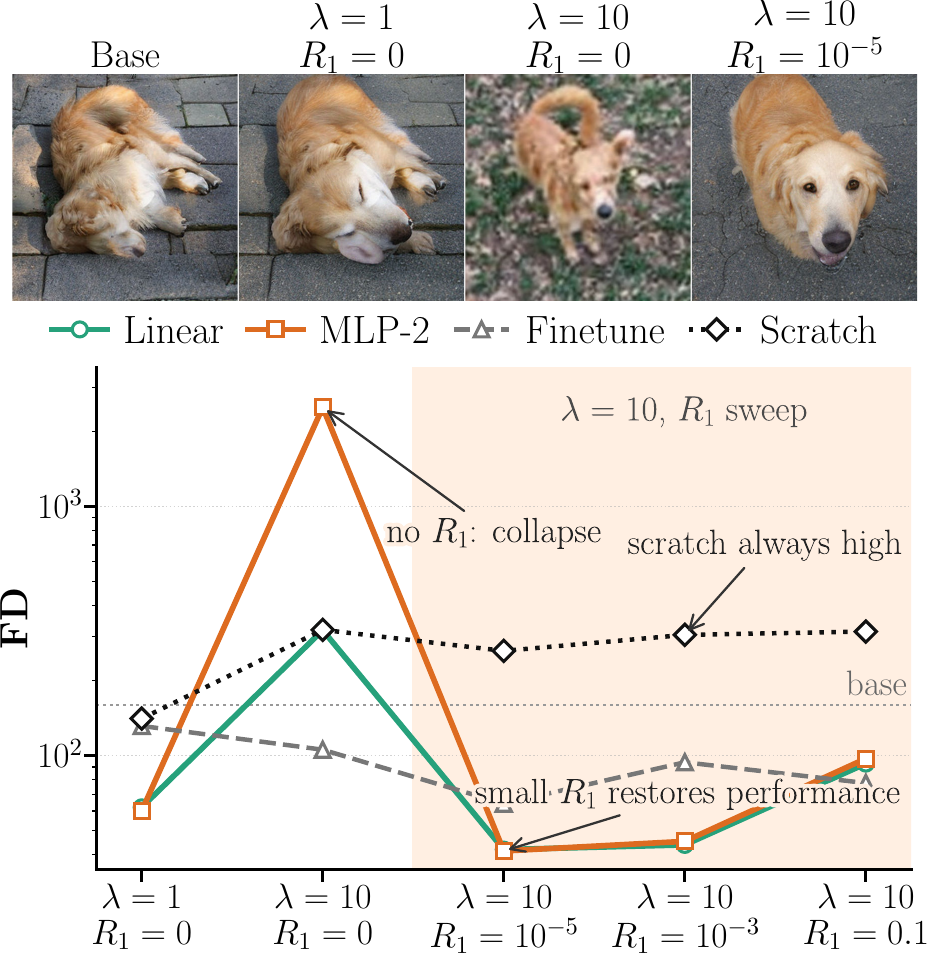}
    \caption{\textbf{Discriminator ablation.} Frozen feature heads work well at $\lambda{=}1$ without $R_1$. At $\lambda{=}10$, removing $R_1$ leads to collapse while a small $R_1$ restores performance. The "from-scratch" discriminator is the worst.}
    \label{fig:disc_ablation}
    \vspace{-1.2em}
\end{wrapfigure}
\fi

\textit{Feature space.}
We re-train the discriminator on six other frozen embedders---DINOv2-B,
DINOv3 (B/L), SigLIP (B/L), and InceptionV3---at $\lambda{=}1$,
$R_1{=}0$, and evaluate FD in every embedder's space.
\Cref{tab:feature_ablation_short} reports a representative subset; the
full $7{\times}7$ table is in \Cref{tab:feature_ablation} with KD as an additional metric. All six SSL
embedders improve both FD and KD over the base on every evaluation
space, with small spread between them; DINOv2-L is consistently
strongest, confirming our default. The one exception is InceptionV3,
which barely improves over the base on any SSL evaluation space and is
outperformed on its own evaluation space by every SSL embedder---we
suspect its classification-only training fails to encode perceptual
structure as linear directions. Larger $\lambda$ and nonzero $R_1$
narrow the gap (FID 3.4, FD 107 on DINOv2-B) but do not close it.\looseness-1

\section{Related Work}
\label{sec:related_work}

We highlight three primary threads of related work below and delay a more detailed discussion to App.\ \Cref{app:related_work}.

\paragraph{RL vs.\ Imitation Learning.}
A long line of research has studied the limitations of supervised imitation learning and the benefits of on-policy methods~\citep{ross2010efficient,ross2011reduction}. We extend this line to flow and score matching, showing that they may suffer from similar pathologies despite not being demonstration-based, with the continuous-time setting yielding even weaker guarantees than standard DAgger-style bounds.

\paragraph{Inverse RL.}
Our remedy mirrors the structure of inverse RL~\citep{abbeel2004apprenticeship,ziebart2008maximum}, which recovers a reward from expert demonstrations and trains a policy against it. Generative adversarial imitation learning (GAIL)~\citep{ho2016generative} and adversarial inverse reinforcement learning (AIRL)~\citep{fu2017learning} apply this idea to standard RL control; we adapt it to improve generative models. Moreover, while these prior works are largely motivated by closing the train--test gap, our observation that the method only helps when paired with an embedder suggests that its principal benefit may be addressing geometric obstructions, not the mismatch itself. We provide further discussion of this point in \Cref{app:limitations}.

\paragraph{Calibration.}
Our theory also asks a calibration question: for a property $r$, does small flow- or score-matching loss imply $\E_p[r]\approx\E_q[r]$? Related work enforces distribution-level constraints directly in controlled generation and fairness-oriented fine-tuning~\citep{khalifa2020distributional,shen2023finetuning}. Closest to us is \citet{smith2510calibrating}, who calibrate generators by finding the KL-closest model satisfying user-specified moment constraints. \drl solves a related KL-constrained problem: it seeks a distribution close to the original model whose pushforward through a feature map matches the data pushforward. Unlike \citet{smith2510calibrating}, we focus on aligning the full distribution in high-dimensional SSL feature spaces, and solve the implicit KL-constrained problem by learning a discriminator between real and generated samples, whose logit is then used for RL.

\section{Conclusions}
\label{sec:conclusions}

\ifmeta
\begin{wraptable}[13]{r}{0.42\textwidth}
\centering
\vspace{-1.0\baselineskip}
\caption{\textbf{Feature space ablation (REPA).} \\ $\lambda{=}1$, $R_1{=}0$, FD ($\downarrow$). Each row is an evaluation feature space; each column is the feature space used to train the discriminator. Best \drl per row in \textbf{bold}. SSL embedders are L variants. Full FD/KD table in \Cref{tab:feature_ablation}.}
\label{tab:feature_ablation_short}
\footnotesize
\setlength{\tabcolsep}{2pt}
\begin{tabular}{l l c cccc}
\toprule
 & & & \multicolumn{4}{c}{Training} \\
\cmidrule(lr){4-7}
 &  & Base & DINOv2 & DINOv3 & SigLIP & Incep. \\
\midrule
\multirow{4}{*}{\rotatebox[origin=c]{90}{Eval}} & Incep. & \cellcolor[HTML]{F5D3BC}6.43 & \cellcolor[HTML]{BEE3D7}\textbf{2.14} & \cellcolor[HTML]{D5EDE5}3.02 & \cellcolor[HTML]{C3E5DB}2.34 & \cellcolor[HTML]{F8E0D0}5.68 \\
 & DINOv2 & \cellcolor[HTML]{F5D3BC}159 & \cellcolor[HTML]{BEE3D7}\textbf{58.3} & \cellcolor[HTML]{C3E5DB}63.2 & \cellcolor[HTML]{DEF1EB}87.1 & \cellcolor[HTML]{F6D8C5}151 \\
 & DINOv3 & \cellcolor[HTML]{F5D3BC}63.7 & \cellcolor[HTML]{D9EFE8}29.4 & \cellcolor[HTML]{BEE3D7}\textbf{18.5} & \cellcolor[HTML]{F1F9F6}38.8 & \cellcolor[HTML]{F5D5C0}62.0 \\
 & SigLIP & \cellcolor[HTML]{F5D3BC}31.2 & \cellcolor[HTML]{BEE3D7}\textbf{14.3} & \cellcolor[HTML]{CDE9E0}16.5 & \cellcolor[HTML]{C9E8DE}16.1 & \cellcolor[HTML]{F6D9C5}29.9 \\
\bottomrule
\end{tabular}

\vspace{-1.0\baselineskip}
\end{wraptable}
\fi

In this paper, we have argued that for flow-based models, RL is valuable not only as a method for steering toward externally specified rewards, but also as a means of evaluating the model on its own samples and exploiting reward geometry rather than $\ell_2$ distance.
\drl turns this view into a method, and with its discriminator-based reward, consistently improves distributional alignment, preference metrics, and performance--drift tradeoffs under preference-based RL.

Several directions follow: 1) The reliance on frozen SSL features invites the question of whether comparable representations can be learned jointly with the model. 2) Our bounds are worst-case, and quantifying a priori which properties FSM struggles with, as well as how these obstructions translate to practical SDE/ODE samplers, would clarify when an RL stage is helpful; we discuss these limitations further in \Cref{app:limitations}. 3) More broadly, RL is one way to exploit on-policy evaluation and non-$\ell_2$ geometry; other sample-based losses like MMD may offer different trade-offs.

Together our results point to a view of generative training in which matching objectives and RL are not two methods with different goals, but rather two complementary tools for getting the most out of our data.

\section*{Acknowledgments}

We thank Brian Trippe, Brian Karrer, Ricky T. Q. Chen, and Sebastian Salazar for helpful comments on the manuscript and for insightful discussions about the ideas presented here.

\clearpage
\bibliographystyle{assets/plainnat}
\bibliography{paper}

\clearpage
\beginappendix

\startcontents[appendix]
\ifdefined\nolinenumbers\nolinenumbers\fi

\titlecontents{appsection}
  [0em]
  {\small}
  {\contentslabel{2.5em}}
  {}
  {\titlerule*[0.5pc]{.}\contentspage}

\titlecontents{appsubsection}
  [1.5em]
  {\small}
  {\contentslabel{2.5em}}
  {}
  {\titlerule*[0.5pc]{.}\contentspage}

\section*{Appendix Contents}
\printcontents[appendix]{app}{1}[2]{}

\clearpage
\ifdefined\linenumbers\linenumbers\fi

\FloatBarrier
\clearpage
\section*{Appendix}
\section{Detailed Related Work}
\label{app:related_work}

\paragraph{Covariate shift, imitation learning, and exposure bias.}
The distribution mismatch pathology we identify in flow and score matching
(\Cref{prop:sft_no_certificate}) is analogous to the classical covariate-shift
problem in imitation learning with behavioral cloning. Behavioral cloning controls prediction error
under the expert's state distribution, but the learned policy is deployed under
its own induced state distribution. As a result, errors that are large under the
model's rollout distribution can appear small under the expert's distribution
\citep{ross2010efficient,ross2011reduction}. This phenomenon has been
extensively studied in reinforcement learning and imitation learning~\citep{laskey2017dart,rajaraman2020toward,swamy2021moments}, and has
recently received renewed attention in language modeling through on-policy
distillation and related methods
\citep{agarwal2024onpolicydistillation,hübotter2026sdpo,lu2025onpolicydistillation,shenfeld2026selfdistillationenablescontinual}.

Independently, the diffusion literature has studied analogous train--test
mismatch under the names exposure bias, sampling drift, and error propagation
\citep{ning2023input,li2023error,ning2023elucidating}. Because this thread has
evolved separately from the RL literature, existing approaches often
address the mismatch through ad hoc, model-specific off-distribution corrections, such
as perturbing training inputs to mimic sampling-time errors
\citep{ning2023input}, modifying the sampling trajectory via shifted time steps
\citep{li2023alleviating}, or modifying the norm of transitions in the sampling process \citep{ning2023elucidating}.
The closest work to an on-policy correction is \citet{li2023error}, who use an MMD regularizer to
match forward noising marginals to short model-induced denoising marginals.
However, their rollouts are warm-started from the data noising distribution and
are very short in their main experiments: only \(L=5\) reverse steps out of
\(T=1000\). Thus the method addresses local train--test mismatch rather than the
full sampling-trajectory occupancy mismatch.

We connect these two lines of work in two ways. First, we show that flow and
score matching admit worst-case pathologies analogous to those in
DAgger-style imitation learning, with even weaker guarantees due to the
continuous-time nature of the dynamics. Second, motivated by this connection,
we focus on an explicitly on-policy solution, closer in spirit to
reinforcement learning and imitation learning than to off-policy
perturbation-based corrections. Given our positive results, we believe an interesting direction for future work is to revisit sampling-based approaches such as \citet{li2023error} with the stronger tooling used in our paper, such as adjoint-based training and the use of semantic feature spaces in the endpoint distributions for the MMD computations.

\paragraph{Inverse RL and adversarial imitation.}
A classical alternative to behavioral cloning is inverse reinforcement learning (IRL), which infers a reward under which expert behavior is optimal and then optimizes that reward with RL~\citep{ng2000algorithms,abbeel2004apprenticeship}. Maximum-entropy and maximum-causal-entropy IRL are especially close in form to our setting: they model expert trajectories as an exponential tilt of a reference trajectory measure by cumulative reward, \(p_r(\tau)\propto \mu_{\mathrm{env}}(\tau)\exp(R_r(\tau))\), where \(\mu_{\mathrm{env}}\) is induced by the initial-state distribution and environment dynamics~\citep{ziebart2008maximum,ziebart2010modeling}. KL-regularized post-training has the analogous endpoint form \(p_r(x)\propto p_{\mathrm{base}}(x)\exp(\lambda r(x))\), so \(r^*(x)=\lambda^{-1}(\log q(x)-\log p_{\mathrm{base}}(x))\) makes \(q\) the optimal tilted distribution, up to additive constants. 

Closely related IRL algorithms include adversarial imitation methods, which use discriminators in related but different ways than \drl. GAIL trains a discriminator between expert and current learner state-action samples, yielding an adversarial objective for matching occupancy measures rather than a fixed reusable reward~\citep{ho2016generative}. AIRL keeps the adversarial imitation loop but constrains the discriminator logit to decompose into a reward term and a potential-based shaping term, with the goal of recovering a reward that can be re-optimized under changes in dynamics rather than just a policy that matches the expert in the training MDP~\citep{fu2017learning}. \drl adapts the discriminator-as-reward idea from GANs and adversarial imitation~\citep{goodfellow2014generative} to generative post-training, but with a fixed endpoint reference: the discriminator is trained only once on \((q, p_{\mathrm{base}})\) through a representation \(\phi\), and its logit defines a fixed reward used by KL-regularized RL. This avoids the notoriously unstable alternating min-max optimization~\citep{mescheder2018ganconvergence}, and exposes a single hyperparameter \(\lambda\) controlling the strength of the correction. We also differ from prior work in our use of self-supervised representations for \(\phi\), without which, as we show, the procedure is not practical.

\paragraph{Adversarial training in diffusion models.} Adversarial training has also been combined with diffusion models directly, in two main ways. Early work on CIFAR-10 and LSUN Churches~\citep{jolicoeur2020adversarial} augmented the matching loss with an adversarial term on the model's posterior-mean estimate \(\E[X_1 \mid X_t]\) sharpening the one-step denoiser used at the final sampling step. More recent approaches~\citep{xiao2021tackling,sauer2024adversarial,xu2024ufogen,yin2024improved} use GANs to essentially amortize the transition kernel $p(x_s \mid x_t)$ for $s$ and $t$, exploiting the fact that GANs are good implicit models in order to enable few-step sampling and distillation.

In addition to the differences described in the previous section, \drl differs from these lines of work in two ways. First, our goal is to correct an existing flow, not to learn an implicit model or a few-step sampler. Second, while \citet{jolicoeur2020adversarial} share our motivation, training the discriminator on the posterior mean \(\E[X_1 \mid X_t]\) is not principled: \(\E[X_1 \mid X_t]\) is the MSE denoiser, not a sample from the model, so matching its distribution to data does not in general yield sharp, high-quality samples. \drl instead trains the discriminator on actual model samples and corrects the velocity field directly via adjoint matching.

\paragraph{Reward-based post-training for diffusion and flow models.}
RL post-training of diffusion and flow models requires two ingredients: a reward and an algorithm to optimize it. As discussed throughout the paper, \drl is a contribution to the first; here we discuss its relation to the second. A growing literature studies optimization algorithms for this setting \citep{black2023training,fan2023dpok,domingoenrich2025adjoint,liu2025flowgrpotrainingflowmatching}. While we expect \drl to benefit directly from continued progress in this literature, our analysis in \Cref{sec:limitations} predicts that not every optimizer is suited to the role: an effective Stage~2 optimizer needs to be on-policy and to propagate reward information through the model in a way that exploits its landscape---either through $\nabla_x r$, as in adjoint matching, or through score-function estimators that estimate $\nabla_\theta \E_{p_\theta}[r]$ directly.

Methods that recast reward optimization as a matching problem against a reward-tilted target, such as Tilt Matching \citep{potaptchik2025tilt} and reward-weighted regression \citep{peters2007reinforcement,fan2502online,black2023training}, take a different route: they use the reward to estimate a matching objective under the target distribution, rather than leveraging the reward landscape directly in the update. From the perspective of our analysis, this leaves them subject to the same matching geometry that we identify as a potential limitation. Some instantiations are additionally off-policy, which may further reintroduce a train--test mismatch. These considerations need not be decisive in every problem instance, but they suggest a plausible partial explanation for the underperformance of RWR reported in \citet{black2023training} and of Tilt Matching reported in Table~1 of \citet{potaptchik2025tilt}, relative to methods that either propagate $\nabla_x r$  or rely on policy-gradient-style estimators. Further empirical validation of this hypothesis is an interesting direction for future work.

\paragraph{Calibration.}
Our analysis of surrogate objectives is closely related to
calibration: for a statistic \(r\), we ask whether low flow- or
score-matching loss is enough to control the calibration error
\(|\E_{p_v}[r]-\E_q[r]|\). Distribution-level constraints of this
form have been studied in controlled language
generation~\citep{khalifa2020distributional} and in fairness-oriented
text-to-image fine-tuning~\citep{shen2023finetuning}. The closest
connection is to \citet{smith2510calibrating}. Their applications
differ significantly from ours --- e.g. they rebalance animal-class proportions in a conditional image model so that lions, tigers, and other wildlife
categories appear evenly, and balance male and female character
frequencies in language-model stories about different professions
--- but one of their methods, CGM-reward, is closely related to
\drl. CGM-reward finds the KL-closest model satisfying moment
constraints \(\E_{p_\theta}[h(x)] = h^\star\), with \(h\) typically
a class indicator (e.g., the gender of the character in a generated
story); the max-entropy dual yields the exponential tilt
\(p_{\alpha^\star}(x) \propto p_{\theta_{\mathrm{base}}}(x)
\exp(\alpha^{\star\top}h(x))\). \drl targets a more general object,
the full pushforward \(p^\phi = q^\phi\) under KL regularization
(\Cref{prop:feature_correction}), but in the linear case the two
coincide in form: our class-conditional discriminator
\(r_y(x) = w_y^\top\phi(x) + b_y\) gives a tilt linear in \(\phi\).
Nevertheless, even in this special case the two generally estimate
different targets, since CGM-reward solves an empirical max-entropy
dual against user-specified moments while \drl recovers the tilt
from a discriminator's logit. Moreover, as shown in
\parheadref{sec:exp_ablations}, the linear head is not essential: MLP
discriminators yield similar gains.

\FloatBarrier
\clearpage
\section{Limitations}
\label{app:limitations}

While our results provide evidence that \drl is an effective post-training method, several caveats remain. These include practical limitations inherited from RL-based fine-tuning, limitations of the representation space used to define the reward, and limitations of the simplified theoretical setting used to motivate the method.

\paragraph{Classifier-free guidance.}
One limitation of \drl, and of flow-based RL methods more generally, is that they do not provide a clean way to incorporate classifier-free guidance (CFG) during training. CFG typically disrupts the structure of the learned flow used to pass between ODE and SDE samplers.
We experimented with CFG during training, but it often led to instability; similar difficulties have been reported by \citet{domingoenrich2025adjoint}.\footnote{\url{https://github.com/microsoft/soc-fine-tuning-sd}}
We also observed that the improvements of \drl over the base model are smaller when both are sampled with their best CFG scale, as shown in \cref{fig:alignment}.
We attribute these issues to the fact that current RL algorithms do not yet handle CFG cleanly, and consider developing principled CFG-compatible RL training an important direction for future work.

\paragraph{Choice of $\lambda$.}
As shown in \cref{sec:experiments_main}, the KL regularization weight $\lambda$ plays a significant role in performance, and a small number of values typically need to be explored.
While we found $\lambda{=}10$ to be a reliable default across our experiments, this may not hold in all settings, and the need to tune $\lambda$ adds to the overall computational cost---a burden shared by RL methods more broadly.
An interesting direction for future work is to make the method more robust to the choice of $\lambda$, and more generally to amortize the cost of tuning this hyperparameter, for example with a ControlNet-like mechanism~\citep{zhang2023adding}.

\paragraph{Reverse KL behavior.}
As discussed in the main text, \drl optimizes a reverse KL objective and therefore inherits both its advantages and its drawbacks.
In particular, while \drl is robust at low $\lambda$ values, at large $\lambda$ the model can become mode-seeking, as visible in the PRDC curves (\cref{app:prdc}).
We believe this partly explains why R1 gradient regularization proved beneficial at higher $\lambda$ (\parheadref{sec:exp_ablations}).
Although we did not find this to be a practical issue at the $\lambda$ values we recommend, it is worth keeping in mind.

\paragraph{Additional compute.}
As a post-training method, \drl requires additional compute beyond base-model training (see \cref{app:experimental_setup} for details).
While this cost is only a small fraction of pretraining, it is nonetheless an extra expense that may not be justified in all scenarios.

\paragraph{Role of the representation.}
Our experiments suggest that \drl is most effective when the discriminator is trained in a pretrained representation space. This is a strength, because it lets the reward focus on semantic discrepancies, but it is also a limitation: the method can only correct distributional differences that are visible to the chosen embedder, so it requires a suitable representation in the first place. Moreover, although our analysis emphasizes both train--test mismatch and reward-geometry mismatch, the empirical importance of the embedder suggests that a large part of \drl's benefit may come from changing the geometry in which the reward is estimated, rather than from on-policy optimization alone.

\paragraph{Scope of the theory.}
Our theoretical analysis is intentionally carried out in a simplified probability-flow ODE setting. We believe this captures the core geometric obstructions behind \cref{prop:sft_no_certificate,prop:lipschitz_reward_certificate}, but it is not a complete model of every practical training setup. In particular, while the reward certificate in \cref{prop:lipschitz_reward_certificate} should have close SDE analogues under suitable drift control, the no-certificate construction in \cref{prop:sft_no_certificate} uses the ODE structure directly. It is therefore not clear how far that worst-case construction extends to stochastic samplers, especially because path-KL objectives can yield Pinsker-type control at the trajectory level. This is nevertheless consistent with our intent: the theory is meant to identify failure modes and motivate the method, not to claim that the worst case occurs generically. Furthermore, while a construction like the ODE one is possibly too extreme, as discussed in \Cref{app:related_work} many works have documented significant train--test gaps in practice for SDE samplers as well.

A related gap is the distinction between training and inference dynamics. In our experiments, RL training uses the memoryless SDE, while evaluation uses the ODE sampler. The memoryless-schedule equivalence guarantees agreement at the optimum, but during training the SDE and ODE rollout distributions need not coincide exactly. While this means that SDE-level Pinsker bounds do not by themselves certify the ODE samples used at inference, we did not observe large discrepancies between the two in practice. Moreover, as argued in the paper, we hypothesize that the main driver of \drl's improvements is correcting geometric obstructions, which should affect both SDE and ODE samplers similarly through the reward gradient. This does not affect our explanation of the limitations of flow and score matching.

\FloatBarrier
\clearpage
\section{Proofs for Section~\ref{sec:prelim}: Preliminaries}
\label{app:background}

This appendix collects proofs for the statements in \cref{sec:diffusion,sec:posttraining}.

\subsection{Score-velocity relationship}
\label{app:proofs:score_velocity}

\begin{proposition}[Score-velocity equivalence]
\label{prop:score_velocity}
Let $X_1 \sim q$ and $X_0 \sim \mathcal{N}(0, I)$ be independent, and define the interpolation $X_t = \alpha(t) X_1 + \beta(t) X_0$ with marginal density $q_t$. The velocity field $v_t(x) := \mathbb{E}[\dot{\alpha}(t) X_1 + \dot{\beta}(t) X_0 \mid X_t = x]$ and the score $s_t(x) := \nabla_x \log q_t(x)$ are related by:
\[
v_t(x) = \frac{\dot{\alpha}(t)}{\alpha(t)} x + \beta(t)^2 \left(\frac{\dot{\alpha}(t)}{\alpha(t)} - \frac{\dot{\beta}(t)}{\beta(t)}\right) s_t(x).
\label{eq:score_velocity}
\]
\end{proposition}

\begin{proof}
To compute the velocity, we need $\mathbb{E}[X_1 \mid X_t = x]$ and $\mathbb{E}[X_0 \mid X_t = x]$.

From the interpolation, we can express the noise as $X_0 = (X_t - \alpha(t) X_1)/\beta(t)$. Therefore:
\[
\mathbb{E}[X_0 \mid X_t = x] = \frac{x - \alpha(t) \mathbb{E}[X_1 \mid X_t = x]}{\beta(t)}.
\]

By Tweedie's identity~\citep{efron2011tweedie}, the posterior mean of $X_1$ given $X_t = x$ is:
\[
\mathbb{E}[X_1 \mid X_t = x] = \frac{x + \beta(t)^2 s_t(x)}{\alpha(t)} = \frac{x}{\alpha(t)} + \frac{\beta(t)^2}{\alpha(t)} s_t(x).
\]
Substituting:
\[
\mathbb{E}[X_0 \mid X_t = x] = \frac{x - \alpha(t) \left(\frac{x}{\alpha(t)} + \frac{\beta(t)^2}{\alpha(t)} s_t(x)\right)}{\beta(t)} = \frac{-\beta(t)^2 s_t(x)}{\beta(t)} = -\beta(t) s_t(x).
\]

Now we compute the velocity:
\begin{align*}
v_t(x) &= \dot{\alpha}(t) \mathbb{E}[X_1 \mid X_t = x] + \dot{\beta}(t) \mathbb{E}[X_0 \mid X_t = x] \\
&= \dot{\alpha}(t) \left(\frac{x}{\alpha(t)} + \frac{\beta(t)^2}{\alpha(t)} s_t(x)\right) + \dot{\beta}(t) \left(-\beta(t) s_t(x)\right) \\
&= \frac{\dot{\alpha}(t)}{\alpha(t)} x + \left(\frac{\dot{\alpha}(t) \beta(t)^2}{\alpha(t)} - \dot{\beta}(t) \beta(t)\right) s_t(x) \\
&= \frac{\dot{\alpha}(t)}{\alpha(t)} x + \beta(t)^2 \left(\frac{\dot{\alpha}(t)}{\alpha(t)} - \frac{\dot{\beta}(t)}{\beta(t)}\right) s_t(x).
\end{align*}
This is exactly \eqref{eq:score_velocity}. \qedhere
\end{proof}

\subsection{SDE and ODE marginal equivalence}
\label{app:proofs:sde_marginals}

\begin{proposition}[Marginal equivalence]
\label{prop:sde_ode_marginals}
Let $v_t$ and $s_t = \nabla \log q_t$ be the true velocity and score fields for marginals $q_t$. Then the SDE
\[
dX_t = \left[v_t(X_t) + \tfrac{1}{2}\sigma(t)^2 s_t(X_t)\right] dt + \sigma(t)\, dW_t
\]
and the probability flow ODE $dX_t/dt = v_t(X_t)$ produce identical marginal distributions $q_t$ at all times $t$, for any noise schedule $\sigma(t) \geq 0$.
\end{proposition}

\begin{proof}
The Fokker--Planck equation describes how the density $\rho_t$ of a diffusion process evolves. For an SDE of the form $dX_t = b(X_t, t) dt + \sigma(t) dW_t$, the density satisfies
\[
\partial_t \rho_t = -\nabla \cdot (b \, \rho_t) + \frac{\sigma(t)^2}{2} \Delta \rho_t.
\]

For our SDE with drift $b = v_t + \frac{1}{2}\sigma(t)^2 s_t$, this becomes
\[
\partial_t \rho_t = -\nabla \cdot \left[\left(v_t + \tfrac{1}{2}\sigma(t)^2 s_t\right) \rho_t\right] + \frac{\sigma(t)^2}{2} \Delta \rho_t.
\]

We verify that $\rho_t = q_t$ satisfies this equation. Using $s_t = \nabla \log q_t$, we have $s_t \, q_t = \nabla q_t$. Therefore:
\begin{align*}
-\nabla \cdot \left(\tfrac{1}{2}\sigma(t)^2 s_t \, q_t\right) + \frac{\sigma(t)^2}{2} \Delta q_t
&= -\frac{\sigma(t)^2}{2} \nabla \cdot (\nabla q_t) + \frac{\sigma(t)^2}{2} \Delta q_t \\
&= -\frac{\sigma(t)^2}{2} \Delta q_t + \frac{\sigma(t)^2}{2} \Delta q_t = 0.
\end{align*}

The diffusion and score terms cancel exactly, leaving
\[
\partial_t q_t = -\nabla \cdot (v_t \, q_t),
\]
which is the continuity equation for the probability flow ODE. Since both the ODE and SDE satisfy the same continuity equation with the same initial condition $q_0$, they have identical marginals $q_t$ at all times. \qedhere
\end{proof}

\FloatBarrier
\clearpage
\section{Proofs for Section~\ref{sec:limitations}: Understanding the Limitations of Flow and Score Matching}
\label{app:properties}

This appendix collects proofs for the statements in \cref{sec:limitations}. The material follows the order of the main discussion: we first prove the no-certificate result for the standard probability-flow ODE, then collect the reward-geometry results used in \cref{prop:lipschitz_reward_certificate} into a single section, and finally prove the RL reward-regret bound (\cref{app:prop:rl_good_reward}).

\subsection{Distribution Shift and Error Accumulation}

\subsubsection[No reward certificate for the standard probability-flow sampler]{Proof of \cref{prop:sft_no_certificate}: no reward certificate for the standard probability-flow sampler}
\label{app:prop:sft_no_certificate}

We now prove the no-certificate claim by giving two complementary counterexamples for the probability-flow ODE. The first is a one-dimensional Gaussian velocity example with a closed-form calculation. It is useful because every quantity can be written down explicitly, so the $q_t$ versus $p_t$ mismatch is visible in a single formula and provides a simple illustration of what goes wrong. The second proof gives a general construction in the coordinate induced by the true target flow and applies equally to the velocity and score parametrizations. 

Throughout we use the setup in \Cref{sec:diffusion} and the following standard assumptions and notation. We let $(\alpha,\beta)$ be a $C^1$ interpolation schedule \eqref{eq:interpolation} with
\[
\alpha(0)=0,\qquad \beta(0)=1,\qquad \alpha(1)=1,\qquad \beta(1)=0,
\]
and $\alpha(t),\beta(t)>0$ for $t\in(0,1)$. For a target endpoint law $q$, write $q_t$ for the interpolation marginal, $s_q$ for its score, and $v_q$ for its probability-flow velocity. For score-parametrized fields we also write
\begin{equation}
\gamma(t)
:=
\beta(t)^2
\left(
    \frac{\dot\alpha(t)}{\alpha(t)}
    -
    \frac{\dot\beta(t)}{\beta(t)}
\right),
\label{eq:gamma_def}
\end{equation}
so that
\[
v_q(x,t)=\frac{\dot\alpha(t)}{\alpha(t)}x+\gamma(t)s_q(x,t)
\]
by \eqref{eq:score_velocity}.

\ifmeta
\begin{figure}[t]
    \centering
    \begin{minipage}[t]{0.49\textwidth}
        \centering
        \includegraphics[width=\linewidth]{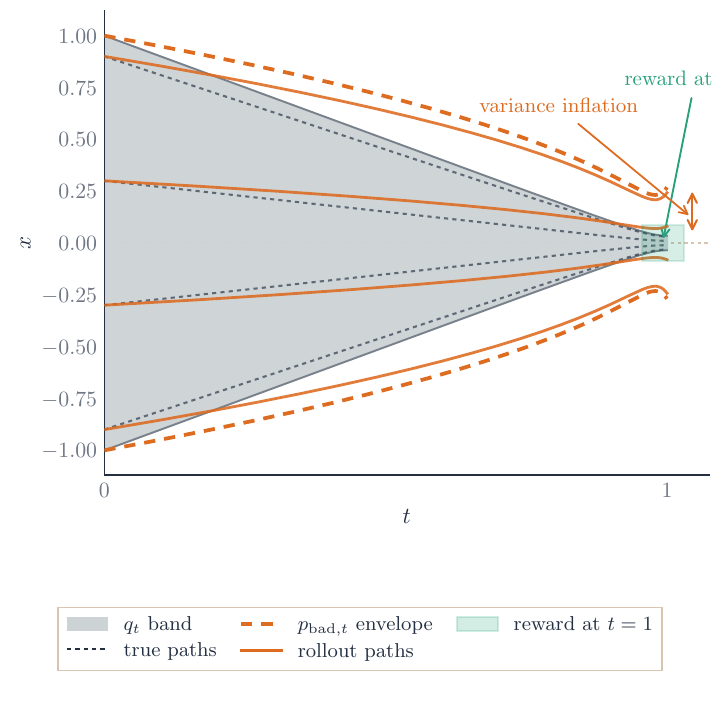}
    \end{minipage}
    \hfill
    \begin{minipage}[t]{0.49\textwidth}
        \centering
        \includegraphics[width=\linewidth]{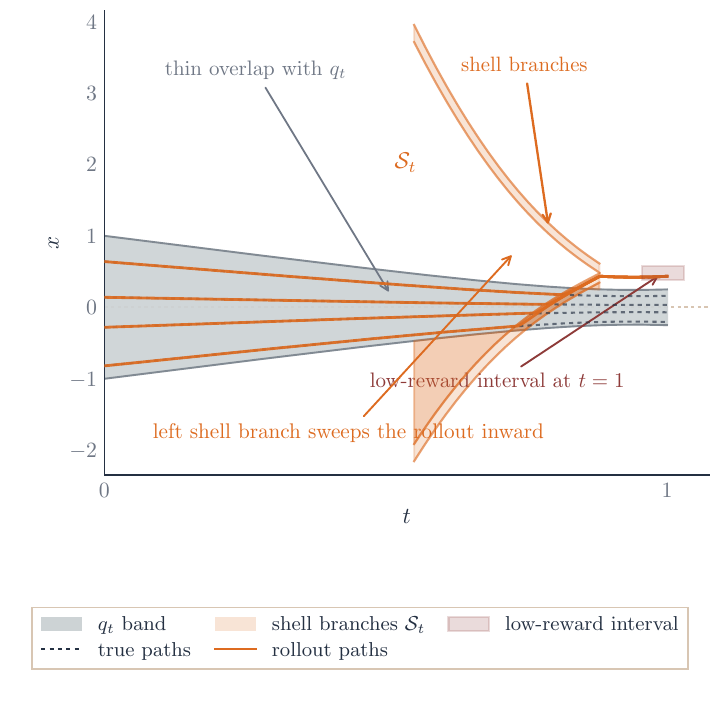}
    \end{minipage}
    \caption{Appendix sketches of the two constructions. Left: velocity case, corresponding to \Cref{app:prop:ode_velocity_no_certificate} and its proof below. The true trajectories contract with $q_t$, while the learned rollout trajectories and the one-sigma envelope of $p_{\mathrm{bad},t}$ show variance inflation under the same linear perturbation. Right: a one-dimensional Gaussian instance of \Cref{app:prop:ode_general_no_certificate}, drawn directly in $x$-space. In the general proof the perturbation is a shrinking radial shell in the target-flow coordinate; in one dimension that shell becomes two moving intervals, shown in orange here, and the left branch sweeps representative rollout trajectories into the low-reward interval at $t=1$. In both panels, the shaded regions indicate typical mass under the relevant marginals, the dashed dark paths denote representative true trajectories, and the solid orange paths denote representative learned rollout trajectories. The obstruction is the same: training measures error under the data marginals $q_t$, while sampling evaluates the learned field along the rollout marginals $p_t$.}
    \label{fig:distribution_shift_appendix}
\end{figure}
\fi

\begin{proposition}[Velocity case: Gaussian variance inflation]
\label{app:prop:ode_velocity_no_certificate}
For every $\varepsilon>0$ and every $\delta\in(0,1/4)$, there exist a target distribution $q$, a bounded reward $r:\mathbb{R}\to[0,1]$, and a velocity field $v_{\mathrm{bad}}$ such that
\[
\sup_{t\in[0,1]}
\E_{X\sim q_t}
\left[
    |v_{\mathrm{bad}}(X,t)-v_q(X,t)|^2
\right]
\le \varepsilon,
\]
but if $p_{\mathrm{bad}}$ denotes the endpoint law of the standard probability-flow ODE sampler
\[
\dot X_t = v_{\mathrm{bad}}(X_t,t),
\qquad
X_0\sim\mathcal{N}(0,1),
\]
then
\[
\E_q[r]-\E_{p_{\mathrm{bad}}}[r]\ge 1-2\delta.
\]
\end{proposition}

\noindent\emph{Intuition.}
The marginals $q_t$ have analytical form given by:
\[
q_t=\mathcal{N}(0,\rho_\tau(t)^2),
\qquad
\rho_\tau(t)^2:=\alpha(t)^2\tau^2+\beta(t)^2.
\]
The idea is to define a perturbation given by, 
\[
\Delta v(x,t)=\sqrt{\varepsilon}\,\frac{x}{\rho_\tau(t)}.
\]
Under  $q_t$, the normalized coordinate $X/\rho_\tau(t)$ is standard Gaussian, so the flow-matching error is exactly $\varepsilon$ at every time. 
However, as time progresses this small error accumulates and becomes larger as  $| x|$ grows.
This ends up increasing the variance of the rollout.
Therefore, by choosing the target endpoint distribution to be sufficiently concentrated, we can make a reward on the high-likelihood regions of $q_1$ to be missed by most of the distribution.

\begin{remark}
This is not the same statement in the main text but it immediately disproves the existence of a guarantee of the form ``with small flow-matching error, we obtain a small reward gap''.
\end{remark}

\begin{proof}
Fix $\varepsilon>0$ and $\delta\in(0,1/4)$. Let $\tau\in(0,1)$ be chosen later and set
\[
q:=\mathcal{N}(0,\tau^2).
\]
Under the interpolation \eqref{eq:interpolation},
\[
X_t=\alpha(t)X_1+\beta(t)X_0,
\qquad
X_1\sim q,
\quad
X_0\sim\mathcal{N}(0,1),
\]
the marginal is
\[
q_t=\mathcal{N}(0,\rho_\tau(t)^2),
\qquad
\rho_\tau(t)^2:=\alpha(t)^2\tau^2+\beta(t)^2.
\]

The corresponding probability-flow velocity is
\[
v_q(x,t)=\frac{\dot\rho_\tau(t)}{\rho_\tau(t)}x.
\]
Indeed,
\[
\dot X_t=\dot\alpha(t)X_1+\dot\beta(t)X_0.
\]
Since $(X_t,\dot X_t)$ are jointly centered Gaussian, the conditional expectation is linear:
\[
v_q(x,t)=\E[\dot X_t\mid X_t=x]
=
\frac{\Cov(\dot X_t,X_t)}{\Var(X_t)}x.
\]
Using the independence of $X_1$ and $X_0$,
\[
\Cov(\dot X_t,X_t)
=
\alpha(t)\dot\alpha(t)\tau^2+\beta(t)\dot\beta(t)
=
\frac{1}{2}\frac{d}{dt}\rho_\tau(t)^2
=
\rho_\tau(t)\dot\rho_\tau(t),
\]
while $\Var(X_t)=\rho_\tau(t)^2$. Therefore
\[
v_q(x,t)=\frac{\dot\rho_\tau(t)}{\rho_\tau(t)}x.
\]

Now define
\[
v_{\mathrm{bad}}(x,t):=v_q(x,t)+\Delta v(x,t),
\qquad
\Delta v(x,t):=\sqrt{\varepsilon}\,\frac{x}{\rho_\tau(t)}.
\]
If $X\sim q_t$, then $\E[X^2]=\rho_\tau(t)^2$, so
\[
\E_{q_t}\bigl[|v_{\mathrm{bad}}(X,t)-v_q(X,t)|^2\bigr]
=
\varepsilon\,\E_{q_t}\left[\frac{X^2}{\rho_\tau(t)^2}\right]
=
\varepsilon.
\]
Thus the training error is exactly $\varepsilon$ uniformly in time.

The learned ODE is
\[
\dot X_t
=
\left(
    \frac{\dot\rho_\tau(t)}{\rho_\tau(t)}
    +
    \frac{\sqrt{\varepsilon}}{\rho_\tau(t)}
\right)X_t.
\]
Solving this scalar linear ODE gives
\[
X_t
=
\rho_\tau(t)\exp\left(
    \sqrt{\varepsilon}\int_0^t\frac{du}{\rho_\tau(u)}
\right)X_0.
\]
Hence the rollout marginal is
\[
p_{{\rm bad},t}
=
\mathcal{N}(0,\rho_\tau(t)^2k(t)),
\qquad
k(t):=
\exp\left(
    2\sqrt{\varepsilon}\int_0^t\frac{du}{\rho_\tau(u)}
\right).
\]
In particular,
\[
\E_{p_{{\rm bad},t}}\bigl[|\Delta v(X,t)|^2\bigr]
=
\varepsilon\,\E_{p_{{\rm bad},t}}\left[\frac{X^2}{\rho_\tau(t)^2}\right]
=
\varepsilon k(t),
\]
so the same perturbation is measured as $\varepsilon$ under $q_t$ but as $\varepsilon k(t)$ under the rollout law $p_{{\rm bad},t}$.

We now show that $k(1)$ can be made arbitrarily large by taking $\tau$ small. Because $\beta$ is $C^1$ with $\beta(1)=0$, there exist $M>0$ and $t_0<1$ such that
\[
\beta(t)\le M(1-t)
\qquad
\text{for all }t\in[t_0,1].
\]
Let $A:=\sup_{t\in[0,1]}|\alpha(t)|<\infty$. Then
\[
\rho_\tau(t)
=
\sqrt{\alpha(t)^2\tau^2+\beta(t)^2}
\le
\sqrt{A^2\tau^2+M^2(1-t)^2}
\qquad
\text{for }t\in[t_0,1].
\]
Therefore
\[
\int_0^1\frac{dt}{\rho_\tau(t)}
\ge
\int_{t_0}^1\frac{dt}{\sqrt{A^2\tau^2+M^2(1-t)^2}}
=
\frac{1}{M}
\sinh^{-1}\left(
    \frac{M(1-t_0)}{A\tau}
\right),
\]
which diverges as $\tau\downarrow 0$. Hence $k(1)\to\infty$ as $\tau\downarrow 0$.

Choose $K>0$ such that, for $Z\sim\mathcal{N}(0,1)$,
\[
\Pr(|Z|\le K)\ge 1-\delta.
\]
Define
\[
r(x):=\mathbf{1}\{|x|\le K\tau\}.
\]
Under the target endpoint law $q=\mathcal{N}(0,\tau^2)$,
\[
\E_q[r]=\Pr(|Z|\le K)\ge 1-\delta.
\]
Under the bad endpoint law $p_{\mathrm{bad}}=\mathcal{N}(0,\tau^2k(1))$,
\[
\E_{p_{\mathrm{bad}}}[r]
=
\Pr\left(|Z|\le \frac{K}{\sqrt{k(1)}}\right)
\le
\sqrt{\frac{2}{\pi}}\,\frac{K}{\sqrt{k(1)}}.
\]
Here we used that, for $Z\sim\mathcal{N}(0,1)$ and any $a\ge 0$,
\[
\Pr(|Z|\le a)=\int_{-a}^{a}\frac{1}{\sqrt{2\pi}}e^{-z^2/2}\,dz\le \sqrt{\frac{2}{\pi}}\,a.
\]
Since $k(1)\to\infty$, choose $\tau$ small enough that $\sqrt{2/\pi}\,K/\sqrt{k(1)}\le\delta$. Then
\[
\E_q[r]-\E_{p_{\mathrm{bad}}}[r]\ge 1-2\delta. \qedhere
\]

\end{proof}

\noindent For the general proof, we use the standard notion of a flow map; see, for example, \citep{arnold1978ode,boffi2024flow}.
Intuitively, the flow map is the function that, given an initial time $a$, a final time $t$, and a point $y$, returns the point reached by evolving $y$ from time $a$ to time $t$.
Formally, given $0<a<1$ and a point $y\in\mathbb{R}^d$, let $X^{a,y}$ denote the solution of
\[
\dot X_t=v_q(X_t,t),
\qquad
X_a=y.
\]
The associated flow map is defined by
\[
\Phi_{a,t}(y):=X^{a,y}_t,
\]
that is, $\Phi_{a,t}(y)$ is the point reached at time $t$ by the exact probability-flow ODE initialized at $y$ at time $a$.

\begin{proposition}[Radial-shell no-certificate construction]
\label{app:prop:ode_general_no_certificate}
Let $q$ be the endpoint target law, let $(q_t)_{t\in[0,1]}$ be the interpolation marginals, let $v_q$ and $s_q$ denote the corresponding population velocity and score, and let $r:\mathbb{R}^d\to[0,1]$ be a bounded reward. Fix $0<a<b<1$. Assume that the exact probability-flow ODE
\[
\dot X_t=v_q(X_t,t)
\]
generates flow maps $\Phi_{a,t}:\mathbb{R}^d\to\mathbb{R}^d$ that are $C^1$ diffeomorphisms for every $t\in[a,1]$. Assume further that there exist $x_\star\in\mathbb{R}^d$, $\rho_\star>0$, and $\eta\in[0,1]$ such that
\[
r(x)\le \eta
\qquad
\text{whenever } \|x-x_\star\|\le \rho_\star.
\]

Then for every $\varepsilon>0$ and every $\delta\in(0,1)$, the following two statements hold.

\textbf{Velocity parametrization.}
There exists a field $v_{\rm bad}$ such that
\[
v_{\rm bad}(x,t)=v_q(x,t)
\qquad
\text{for all }t\notin[a,b],
\]
and
\[
\sup_{t\in[0,1]}
\E_{X\sim q_t}
\left[
\|v_{\rm bad}(X,t)-v_q(X,t)\|^2
\right]
\le \varepsilon.
\]
Let $p_{\rm bad}$ denote the endpoint law of the probability-flow ODE sampler \eqref{eq:ode_sde} run with $v_{\rm bad}$ in place of $v_q$. Then
\[
\E_{p_{\rm bad}}[r]\le \eta+\delta.
\]

\textbf{Score parametrization.}
If $\inf_{t\in[a,b]}|\gamma(t)|>0$, then there exists a field $s_{\rm bad}$ such that
\[
s_{\rm bad}(x,t)=s_q(x,t)
\qquad
\text{for all }t\notin[a,b],
\]
and
\[
\sup_{t\in[0,1]}
\E_{X\sim q_t}
\left[
\|s_{\rm bad}(X,t)-s_q(X,t)\|^2
\right]
\le \varepsilon.
\]
Define the induced velocity field
\[
v_{\rm bad}(x,t)
=
\frac{\dot\alpha(t)}{\alpha(t)}x+\gamma(t)s_{\rm bad}(x,t),
\]
and let $p_{\rm bad}$ denote the endpoint law of the probability-flow ODE sampler \eqref{eq:ode_sde} driven by this $v_{\rm bad}$. Then
\[
\E_{p_{\rm bad}}[r]\le \eta+\delta.
\]
\end{proposition}

\noindent\emph{Intuition.}
The goal is to hide a harmful perturbation inside a small score-matching or flow-matching error. We do this with a thin shell that moves inward over time. Away from the shell, the learned field agrees with the true score/flow, so if the shell is thin enough then the training error stays small. At the same time, the shell is arranged to catch trajectories and carry them into a low-reward region by time $1$. \Cref{fig:distribution_shift_appendix} (right) gives a sketch in one dimension, where the shell becomes two moving intervals. We apply the construction only on a subinterval $[a,b]\subset(0,1)$ to avoid the endpoints.

However, designing the shell directly in $x$-space is awkward because the true dynamics already move trajectories around on their own, so the shell's effect gets tangled with the motion of the original dynamics. We avoid this by working in a coordinate in which the true dynamics are trivial: for a path $X_t$, define the inverse-flow coordinate
\[
Y_t := \Phi_{a,t}^{-1}(X_t),
\]
which records the starting point at time $a$ of the trajectory passing through $X_t$. By construction, $Y_t$ is constant along the true ODE, so any motion of a trajectory in $Y$-space comes entirely from the perturbation. We then 
translate the motion in this space back to $x$-space and check that the training error is small while still carrying trajectories into the low-reward region.

\begin{proof}
The proof is relatively straightforward but requires some bookkeeping. We organize it in steps for clarity.

Under the assumptions of the proposition, the laws $q_t$ for $t\in[a,1]$ admit densities. With a slight abuse of notation, we use the same symbol for the law and its density when no confusion arises.

\medskip\noindent\textbf{Step 1: Work in the target-flow coordinate and rewrite the reward there.}
For any path $X_t$ on $[a,1]$, define
\[
Y_t:=\Phi_{a,t}^{-1}(X_t).
\]
This coordinate records where $X_t$ came from at time $a$ under the true target flow. If $X_t$ follows the true dynamics
then $X_t=\Phi_{a,t}(Y_a)$, so $Y_t=Y_a$ is constant.

Because the shell construction will be carried out in $Y$-space, it is convenient to rewrite the endpoint reward in the same coordinate. Set
\[
\mu:=q_a,
\qquad
B(z,\rho):=\{y\in\mathbb{R}^d:\ \|y-z\|\le \rho\},
\qquad
\widetilde r(y):=r(\Phi_{a,1}(y)).
\]
Let
\[
z:=\Phi_{a,1}^{-1}(x_\star),
\]
and consider the open ball $\{x\in\mathbb{R}^d:\ \|x-x_\star\|<\rho_\star\}$ in endpoint space. Because $\Phi_{a,1}$ is a diffeomorphism, its preimage is an open neighborhood of $z$, so there exists $\rho>0$ such that
\[
B(z,\rho)\subset \Phi_{a,1}^{-1}\left(\left\{x\in\mathbb{R}^d:\ \|x-x_\star\|<\rho_\star\right\}\right).
\]
Therefore
\[
\widetilde r(y)\le \eta
\qquad
\text{for all } y\in B(z,\rho).
\]
Thus it is enough to construct a bad rollout whose terminal target-flow coordinate $Y_1$ lies in the ball $B(z,\rho)$. In the remaining steps we will first define the bad dynamics directly in this coordinate, where the true dynamics are frozen, and then recover the corresponding path in $x$-space by applying the flow map.

\medskip\noindent\textbf{Step 2: Build the shrinking shell in the target-flow coordinate.}
We first choose a large ball that contains almost all of the mass at time $a$. Because
\[
\mu(B(z,R))\to 1
\qquad
\text{as } R\to\infty,
\]
we can choose $R>\rho$ such that
\[
\mu(B(z,R))\ge 1-\delta.
\]
Our goal is to build a thin shell that starts just outside $B(z,R)$ and then moves inward until it lies inside the low-reward ball $B(z,\rho)$.

Let
\[
T:=b-a,
\qquad
c:=\frac{R+1-\rho/2}{T}
\qquad
\ell(t):=R+1-c(t-a)
\quad
\text{for } t\in[a,b].
\]
Here $T$ is the length of the active interval, $c$ is the inward speed, and $\ell(t)$ is the radius of the shell at time $t$.
Then
\[
\ell(a)=R+1,
\qquad
\ell(b)=\rho/2.
\]
So the shell starts strictly outside $B(z,R)$ and ends strictly inside $B(z,\rho)$.

Fix $w>0$, to be chosen later, with
\[
w<\min\left\{\frac12,\frac{\rho}{4}\right\}.
\]
This will be the thickness scale of the shell. The bound $w\le 1/2$ is only used later, when we keep the shell inside a fixed radial region for the training-error estimate.

Define a piecewise-linear function $\psi_w:\mathbb{R}\to[0,c]$ that specifies how strongly the perturbation acts as a function of signed distance from the shell:
\[
\psi_w(s)
:=
\begin{cases}
0, & s\le -w,\\[4pt]
c\left(1+\dfrac{s}{w}\right), & -w\le s\le 0,\\[8pt]
c, & 0\le s\le w,\\[4pt]
c\left(2-\dfrac{s}{w}\right), & w\le s\le 2w,\\[8pt]
0, & s\ge 2w.
\end{cases}
\]
This function is zero outside $[-w,2w]$, ramps up linearly on $[-w,0]$, is flat at height $c$ on $[0,w]$, and ramps down linearly on $[w,2w]$.

For $y\neq z$, define the outward radial direction
\[
n(y):=\frac{y-z}{\|y-z\|},
\]
and set $n(z):=0$.

Now define the perturbation in the target-flow coordinate by
\[
u_w(y,t)
:=
-\mathbf{1}\{t\in[a,b]\}\,
\psi_w(\|y-z\|-\ell(t))\,n(y).
\]
The minus sign means that the perturbation pushes points inward. It is supported on the moving shell
\[
\mathcal{S}_t
:=
\{y\in\mathbb{R}^d:\ \|y-z\|-\ell(t)\in[-w,2w]\}.
\]
So at time $t$, the perturbation is active only on radii in the interval $[\ell(t)-w,\ell(t)+2w]$.
Because $\ell(t)\in[\rho/2,R+1]$ and $w<\min\{1/2,\rho/4\}$, every point $y\in\mathcal{S}_t$ satisfies
\[
\|y-z\|\in[\ell(t)-w,\ell(t)+2w]\subset[\rho/4,R+2].
\]
Thus the shell $\mathcal{S}_t$ always lies inside the fixed radial region $\{\rho/4\le \|y-z\|\le R+2\}$.

\medskip\noindent\textbf{Step 3: Define and analyze the bad dynamics in $Y$-space.}
Under the true dynamics, the target-flow coordinate does not move: $Y_t$ is constant. We now define the bad dynamics in this same coordinate by letting $Y_t$ solve
\[
\dot Y_t=u_w(Y_t,t)
\qquad
\text{on } [a,1]
\]
with initial law $Y_a\sim\mu$. Because the shell support stays in the radial region $\{\rho/4\le \|y-z\|\le R+2\}$, it stays a positive distance from $z$. Therefore the field $u_w(\cdot,t)$ is globally Lipschitz in $y$ for each $t$ and piecewise continuous in $t$. Hence this ODE has a unique absolutely continuous solution on $[a,1]$.

The mechanism has two phases: first the shell catches a trajectory while the trajectory stays fixed in $Y$-space, and then the shell carries the trajectory inward until it reaches $B(z,\rho)$.

Let
\[
\varrho_t:=\|Y_t-z\|,
\qquad
S_t:=\varrho_t-\ell(t).
\]
Here $\varrho_t$ is the distance to the shell center and $S_t$ records the trajectory's position relative to the moving shell. The sign of $S_t$ tells us where the trajectory sits: $S_t<-w$ means it is strictly inside the shell, while $S_t\in[-w,0]$ means it has been caught by the shell's inner side.

Consider $t\in[a,b]$. Because $u_w(y,t)$ always points in the radial direction $n(y)$, it changes only the distance to $z$. For $\varrho_t>0$,
\[
\dot\varrho_t
=
\left(\frac{Y_t-z}{\|Y_t-z\|}\right)^\top \dot Y_t
=
n(Y_t)^\top u_w(Y_t,t)
=
-\psi_w(S_t).
\]
Since $\dot\ell(t)=-c$, we also have
\[
\dot S_t
=
\dot\varrho_t-\dot\ell(t)
=
-\psi_w(S_t)+c.
\]

If $Y_a\in B(z,\rho)$, then whenever $\varrho_t>0$,
\[
\dot\varrho_t=-\psi_w(S_t)\le 0,
\]
so the trajectory stays in $B(z,\rho)$. It therefore remains to consider the case
\[
\rho\le \varrho_a\le R.
\]
In this case
\[
S_a
=
\varrho_a-\ell(a)
=
\varrho_a-(R+1)
\in[\rho-R-1,-1].
\]
Since $w<1$ the trajectory starts strictly inside the shell, with $S_a<-w$.

While $S_t\le -w$, the function $\psi_w$ vanishes and
\[
\dot S_t
=
\dot\varrho_t-\dot\ell(t)
=
0-(-c)
=
c.
\]
So before capture, the shell simply moves inward until it catches the trajectory. The latest catch occurs for the smallest radius that still needs to be captured, namely $\varrho_a=\rho$. Then the catch time is at most
\[
\frac{-S_a- w}{c}
=
\frac{R+1-\rho-w}{c}
<
\frac{R+1-\rho/2}{c}
=
b-a,
\]
where we used $w<\rho/2$. Hence every trajectory with $Y_a\in B(z,R)\setminus B(z,\rho)$ reaches the inner side of the shell before time $b$.

Once $S_t\in[-w,0]$, the definition of $\psi_w$ gives
\[
\dot S_t
=
-c\left(1+\frac{S_t}{w}\right)+c
=
-\frac{c}{w}S_t.
\]
Therefore $S_t$ remains in $[-w,0]$ and moves toward $0$. In particular,
\[
S_b\le 0.
\]
Since $\ell(b)=\rho/2$, we conclude that
\[
\varrho_b
=
\ell(b)+S_b
\le
\ell(b)
=
\rho/2
<
\rho.
\]
So every trajectory that starts in $B(z,R)$ lies in $B(z,\rho)$ by time $b$.

For $t>b$, the perturbation is switched off, so $u_w(Y_t,t)=0$ and the target-flow coordinate remains fixed. Hence
\[
Y_1\in B(z,\rho)
\qquad
\text{on the event } \{Y_a\in B(z,R)\}.
\]
At this point the geometry is complete: every trajectory that starts in $B(z,R)$ has been moved into $B(z,\rho)$ by time $1$.
By Step 1,
\[
\widetilde r(Y_1)\le \eta
\qquad
\text{on } \{Y_a\in B(z,R)\}.
\]
Since $Y_a\sim\mu$ and $\mu(B(z,R))\ge 1-\delta$, we obtain
\[
\begin{aligned}
\E[\widetilde r(Y_1)]
&\le
\eta\cdot \mu(B(z,R))
+
1\cdot \mu(B(z,R)^c)\\
&\le
\eta+\delta.
\end{aligned}
\]
This is the bad behavior we want in the target-flow coordinate. It remains to realize these same dynamics as a perturbation of the probability-flow ODE in $x$-space.

\medskip\noindent\textbf{Step 4: Push the construction to $x$-space.}
Let $X_t$ denote the exact probability-flow ODE started from
\[
\dot X_t=v_q(X_t,t),
\qquad
X_0\sim\mathcal{N}(0,I).
\]
Then $X_a\sim q_a=\mu$. Set
\[
Y_a:=X_a,
\]
let $Y_t$ evolve by
\[
\dot Y_t=u_w(Y_t,t)
\qquad
\text{on } [a,1],
\]
and define the bad rollout by
\[
\bar X_t
:=
\begin{cases}
X_t, & t\in[0,a],\\[4pt]
\Phi_{a,t}(Y_t), & t\in[a,1].
\end{cases}
\]
This is well defined because
\[
\Phi_{a,a}(Y_a)=Y_a=X_a=\bar X_a.
\]
Moreover, because $\bar X_1=\Phi_{a,1}(Y_1)$, the definition of $\widetilde r$ and the Step 3 bound give
\[
\E[r(\bar X_1)]
=
\E[r(\Phi_{a,1}(Y_1))]
=
\E[\widetilde r(Y_1)]
\le
\eta+\delta.
\]
This proves the endpoint reward bound for the constructed path. It remains to identify this path with the standard probability-flow ODE driven by a suitable field.

For $t\in[a,1]$, differentiating the second branch gives
\[
\dot{\bar X}_t
=
\partial_t\Phi_{a,t}(Y_t)+D\Phi_{a,t}(Y_t)\dot Y_t
=
v_q(\Phi_{a,t}(Y_t),t)+D\Phi_{a,t}(Y_t)u_w(Y_t,t)
\]
Here $D\Phi_{a,t}(y)$ denotes the Jacobian matrix of $\Phi_{a,t}$ with respect to the spatial variable $y$.
Because $\Phi_{a,t}$ is a diffeomorphism, the second term can be written as a function of $(x,t)$ alone. Define
\[
\Delta v(x,t)
:=
\begin{cases}
D\Phi_{a,t}(\Phi_{a,t}^{-1}(x))\,u_w(\Phi_{a,t}^{-1}(x),t), & t\in[a,b],\\[4pt]
0, & t\notin[a,b].
\end{cases}
\]
and set
\[
v_{\rm bad}(x,t):=v_q(x,t)+\Delta v(x,t).
\]
By definition, $\Delta v(x,t)=0$ for $t<a$, so $v_{\rm bad}=v_q$ on $[0,a)$. Since $\bar X_t=X_t$ there, $\bar X$ agrees with the exact ODE on $[0,a)$. On the other hand, for $t\in[a,1]$ the computation above gives
\[
\dot{\bar X}_t=v_{\rm bad}(\bar X_t,t).
\]
Hence $\bar X$ is a solution of
\[
\dot X_t=v_{\rm bad}(X_t,t)
\]
started from $\bar X_0\sim\mathcal{N}(0,I)$. Therefore the endpoint law of $\bar X_1$ is exactly $p_{\rm bad}$, and so
\[
\E_{p_{\rm bad}}[r]\le \eta+\delta.
\]
This proves the endpoint reward bound for the velocity construction. It remains to show that the training error under $q_t$ is small.

\medskip\noindent\textbf{Step 5: The velocity error is small under $q_t$.}
We now show that we can choose $w$ small enough to ensure
\[
\sup_{t\in[0,1]}
\E_{X\sim q_t}
\left[
\|v_{\rm bad}(X,t)-v_q(X,t)\|^2
\right]
=
\sup_{t\in[0,1]}
\E_{X\sim q_t}
\left[
\|\Delta v(X,t)\|^2
\right]
\le \varepsilon.
\]
If $t\notin[a,b]$, then $\Delta v(\cdot,t)=0$, so there is nothing to prove. It therefore suffices to consider $t\in[a,b]$. In that case,
\[
\Delta v(x,t)
=
D\Phi_{a,t}(\Phi_{a,t}^{-1}(x))\,u_w(\Phi_{a,t}^{-1}(x),t).
\]
Now fix $t\in[a,b]$, let $X\sim q_t$, and set
\[
Y:=\Phi_{a,t}^{-1}(X),
\]
Because $q_t=(\Phi_{a,t})_{\#}\mu$, we have $Y\sim\mu=q_a$. Therefore
\[
\E_{X\sim q_t}
\left[
\|\Delta v(X,t)\|^2
\right]
=
\int
\left\|
D\Phi_{a,t}(y)u_w(y,t)
\right\|^2
\,q_a(y)\,dy.
\]
The key observation is that $u_w(\cdot,t)$ is supported on a thin radial shell. Indeed, by definition,
\[
u_w(y,t)\neq 0
\qquad\Longrightarrow\qquad
\|y-z\|\in[\ell(t)-w,\ell(t)+2w].
\]
This interval has length $3w$. Thus the whole expectation is concentrated on a thin shell in $Y$-space, and it remains to control the weighted $q_a$-mass of such shells.

Since the shell support always stays in the radial range $[\rho/4,R+2]$, it is enough to control thin radial shells inside the fixed radial region
\[
A_R:=\left\{y\in\mathbb{R}^d:\ \rho/4\le \|y-z\|\le R+2\right\}.
\]

Because $a<1$, we have $\beta(a)>0$, so $q_a$ is a Gaussian-smoothed version of $q$. In particular, $q_a$ has bounded density:
\[
q_a(x)
=
\int
\frac{1}{(2\pi\beta(a)^2)^{d/2}}
\exp\!\left(
    -\frac{\|x-\alpha(a)u\|^2}{2\beta(a)^2}
\right)
\,q(u)\,du
\le
\frac{1}{(2\pi\beta(a)^2)^{d/2}}.
\]
Because $\Phi_{a,t}$ is a $C^1$ flow, its Jacobian is uniformly bounded on the compact set $[a,b]\times A_R$. Also, on the bounded radial region $A_R$, the volume of a radial shell is proportional to its thickness. Combining these two facts with the density bound above, there exists a finite constant $C_R$ such that for every interval $I\subset[\rho/4,R+2]$ and every $t\in[a,b]$,
\[
\int_{\{y:\,\|y-z\|\in I\}}
\|D\Phi_{a,t}(y)\|_{\mathrm{op}}^2
\,q_a(y)\,dy
\le
C_R|I|.
\]
Indeed, if $S_I:=\{y:\,\|y-z\|\in I\}$, then bounded density and bounded Jacobian on $[a,b]\times A_R$ give
\[
\int_{S_I}\|D\Phi_{a,t}(y)\|_{\mathrm{op}}^2\,q_a(y)\,dy\le C_R'\,\mathrm{Vol}(S_I),
\]
while the volume of a radial shell inside the bounded annulus $A_R$ satisfies $\mathrm{Vol}(S_I)\le C_R''|I|$. Absorbing the constants gives the claim.
\[
I_t:=[\ell(t)-w,\ell(t)+2w].
\]
This interval has length $3w$ and is contained in $[\rho/4,R+2]$. Since $u_w(\cdot,t)$ is supported where
\[
\|y-z\|\in I_t,
\]
and $|u_w(y,t)|\le c$, we obtain
\[
\begin{aligned}
\E_{X\sim q_t}
\left[
\|\Delta v(X,t)\|^2
\right]
&\le
\int
\left\|
D\Phi_{a,t}(y)
\right\|_{\mathrm{op}}^2
\left|u_w(y,t)\right|^2
\,q_a(y)\,dy\\
&\le
c^2
\int_{\{y:\,\|y-z\|\in I_t\}}
\left\|
D\Phi_{a,t}(y)
\right\|_{\mathrm{op}}^2
\,q_a(y)\,dy\\
&\le
c^2 C_R |I_t|
=
3c^2C_R w.
\end{aligned}
\]
Hence choosing
\[
w\le \min\left\{\frac12,\frac{\rho}{4},\frac{\varepsilon}{3c^2C_R}\right\}
\]
ensures
\[
\sup_{t\in[0,1]}
\E_{X\sim q_t}
\left[
\|v_{\rm bad}(X,t)-v_q(X,t)\|^2
\right]
\le \varepsilon.
\]

\medskip\noindent\textbf{Step 6: Score parametrization.}
For the score statement, keep the same reward, the same center $z$, and the same shell radii. If needed, choose a smaller width $w$; this changes the perturbation but not the reward bound. Set
\[
\gamma_-:=\inf_{t\in[a,b]}|\gamma(t)|>0.
\]
Define
\[
\Delta s(x,t)
:=
\begin{cases}
\dfrac{\Delta v(x,t)}{\gamma(t)}, & t\in[a,b],\\[1.25em]
0, & t\notin[a,b].
\end{cases}
\]
Set
\[
s_{\rm bad}(x,t):=s_q(x,t)+\Delta s(x,t),
\]
and define
\[
v_{\rm bad}(x,t)
=
\frac{\dot\alpha(t)}{\alpha(t)}x+\gamma(t)s_{\rm bad}(x,t).
\]
Because
\[
v_q(x,t)
=
\frac{\dot\alpha(t)}{\alpha(t)}x+\gamma(t)s_q(x,t),
\]
we obtain
\[
v_{\rm bad}(x,t)=v_q(x,t)+\Delta v(x,t).
\]
So the score-parametrized model induces exactly the same bad ODE rollout as above, and therefore the same endpoint reward bound.

It remains to check the DSM error. If $t\notin[a,b]$, the error is zero. If $t\in[a,b]$, then
\[
\E_{X\sim q_t}
\left[
\|s_{\rm bad}(X,t)-s_q(X,t)\|^2
\right]
=
\E_{X\sim q_t}
\left[
\|\Delta s(X,t)\|^2
\right]
\le
\frac{1}{\gamma_-^2}
\E_{X\sim q_t}
\left[
\|\Delta v(X,t)\|^2
\right].
\]
Using the bound from Step 5 gives
\[
\E_{X\sim q_t}
\left[
\|s_{\rm bad}(X,t)-s_q(X,t)\|^2
\right]
\le
\frac{3c^2C_R}{\gamma_-^2}\,w.
\]
Thus choosing
\[
w\le \min\left\{\frac12,\frac{\rho}{4},\frac{\varepsilon\gamma_-^2}{3c^2C_R}\right\}
\]
ensures
\[
\sup_{t\in[0,1]}
\E_{X\sim q_t}
\left[
\|s_{\rm bad}(X,t)-s_q(X,t)\|^2
\right]
\le \varepsilon.
\]
This proves the score statement.
\qedhere
\end{proof}

\subsection{Geometry mismatch}

\subsubsection[Reward geometry under stronger rollout control]{Reward geometry under stronger rollout control}
\label{app:lipschitz_certificate}

This subsection collects the geometry results used in the main text into one place. We first prove the global certificate stated in \cref{prop:lipschitz_reward_certificate}, then derive from the same coupling argument a weaker on-policy variant under the rollout marginals $p_t$. The proposition also records a matching tightness example, and we close with the indicator-reward counterexample that explains why this kind of geometric control applies only to Lipschitz rewards.

We use only standard ODE stability ideas and the following elementary differential form of Gr\"onwall's inequality; see, for example, \citet{arnold1978ode} for background.

\begin{lemma}[Gr\"onwall inequality]
\label{lem:integral_gronwall}
Let $m:[0,T]\to[0,\infty)$ be absolutely continuous, let $b:[0,T]\to[0,\infty)$ be integrable, and let $L \ge 0$. If
\[
m'(t)\le b(t)+Lm(t)
\qquad
\text{for almost every } t \in [0,T],
\]
then
\[
m(t)\le e^{Lt}m(0)+\int_0^t e^{L(t-s)} b(s)\,ds
\qquad
\text{for every } t \in [0,T].
\]
\end{lemma}

\begin{proof}
The idea is to remove the linear growth term first. Set
\[
z(t):=m(t)e^{-Lt}.
\]
Then $z$ is absolutely continuous, and for almost every $t$,
\[
z'(t)=e^{-Lt}\bigl(m'(t)-Lm(t)\bigr)\le e^{-Lt}b(t).
\]
Integrating from $0$ to $t$ gives
\[
z(t)-z(0)\le \int_0^t e^{-Ls}b(s)\,ds.
\]
Multiplying by $e^{Lt}$ yields
\[
m(t)\le e^{Lt}m(0)+\int_0^t e^{L(t-s)}b(s)\,ds.
\]
\end{proof}

\begin{proposition}[Reward certificates under stronger rollout control]
\label{prop:lipschitz_certificate}
Consider the coupled ODEs
\[
\dot X_t = v(X_t,t),
\qquad
\dot X_t^* = v^*(X_t^*,t),
\qquad
X_0 = X_0^* \sim p_0,
\]
and write $p_t$ and $p_t^*$ for the laws of $X_t$ and $X_t^*$ respectively. Assume that $v^*(\cdot,t)$ is $L_v$-Lipschitz in $x$, uniformly over $t \in [0,1]$.

\begin{enumerate}
    \item \textbf{Global control.}
    If
    \[
    \|v-v^*\|_\infty := \sup_{t,x}\|v(x,t)-v^*(x,t)\| \le \varepsilon,
    \]
    then
    \[
    \|X_1 - X_1^*\|
    \le
    \int_0^1 e^{L_v(1-s)}\varepsilon\,ds
    =
    \begin{cases}
        \dfrac{\varepsilon}{L_v}(e^{L_v}-1), & L_v > 0, \\[6pt]
        \varepsilon, & L_v = 0,
    \end{cases}
    \qquad \text{a.s.}
    \]
    and for every $L_r$-Lipschitz reward $r$,
    \[
    |\E_{p_1}[r]-\E_{p_1^*}[r]|
    \le
    L_r\int_0^1 e^{L_v(1-s)}\varepsilon\,ds.
    \]

    \item \textbf{On-policy control.}
    Define the rollout error function
    \[
    e(t)
    :=
    \E_{X_t\sim p_t}\bigl[\|v(X_t,t)-v^*(X_t,t)\|\bigr].
    \]
    If $e \in L^1([0,1])$, then
    \[
    \E\|X_1 - X_1^*\|
    \le
    \int_0^1 e^{L_v(1-s)} e(s)\,ds.
    \]
    and for every $L_r$-Lipschitz reward $r$,
    \[
    |\E_{p_1}[r]-\E_{p_1^*}[r]|
    \le
    L_r\int_0^1 e^{L_v(1-s)} e(s)\,ds.
    \]
    In particular, if $\sup_{t\in[0,1]} e(t) \le \varepsilon$, then the same endpoint bound as in part 1 holds at $t=1$.

    \item \textbf{Tightness.}
    The exponential weighting factor $e^{L_v(t-s)}$ and its dependence on $L_v$, $L_r$, and the error function cannot be improved in general. More precisely, there exists an explicit one-dimensional population CFM problem whose population minimizer is $v^*(x,t)=L_v x$ and such that, for every nonnegative $e \in L^1([0,1])$, the perturbation
    \[
    \Delta v(x,t):=e(t),
    \qquad
    v(x,t)=v^*(x,t)+\Delta v(x,t)
    \]
    satisfies
    \[
    \E_{X_t\sim p_t}\bigl[\|v(X_t,t)-v^*(X_t,t)\|\bigr]=e(t)
    \]
    and
    \[
    \E\|X_1-X_1^*\|
    =
    \int_0^1 e^{L_v(1-s)}e(s)\,ds.
    \]
    Thus part 2 is attained with equality for every admissible error function. For the linear reward $r(x)=L_r x$, the reward bound in part 2 is also attained with equality. In particular, taking $e(t)\equiv \varepsilon$ gives $\|v-v^*\|_\infty=\varepsilon$ and attains the bound in part 1.
\end{enumerate}
\end{proposition}

\begin{proof}
The two certificates come from the same coupling argument. The only difference is how the field mismatch term is controlled: in part 1 it is bounded pathwise, while in part 2 it is bounded only after taking expectation.

Set
\[
\Delta_t := X_t - X_t^*.
\]
Because both trajectories solve ODEs with the same initial condition, $\Delta$ is absolutely continuous and
\[
\dot\Delta_t = v(X_t,t)-v^*(X_t^*,t)
\qquad
\text{for almost every } t.
\]
We now use this in two slightly different ways.

\textbf{Part 1: global control.}
Set
\[
y(t):=\|\Delta_t\|.
\]
Since $\Delta$ is absolutely continuous, so is $y$, and for almost every $t$,
\[
y'(t)\le \|\dot\Delta_t\|
\le
\|v(X_t,t)-v^*(X_t,t)\|
+
\|v^*(X_t,t)-v^*(X_t^*,t)\|
\le
\varepsilon + L_v y(t).
\]
The first term is controlled by the uniform bound
\(
\|v-v^*\|_\infty\le \varepsilon
\),
evaluated at the point \(X_t\). The second term is controlled by the \(L_v\)-Lipschitz continuity of \(v^*(\cdot,t)\) in its spatial argument, applied to the pair \((X_t,X_t^*)\). In other words, we split the error into direct field mismatch of size at most \(\varepsilon\), and the amplification of the current gap by the dynamics of \(v^*\).
Because $y(0)=0$, applying \cref{lem:integral_gronwall} with $b(t)\equiv \varepsilon$ gives
\[
y(t)
\le
\int_0^t e^{L_v(t-s)}\varepsilon\,ds
=
\begin{cases}
    \dfrac{\varepsilon}{L_v}(e^{L_v t}-1), & L_v > 0, \\[6pt]
    \varepsilon t, & L_v = 0.
\end{cases}
\]
Evaluating this at $t=1$ gives
\[
\|X_1-X_1^*\|
\le
\int_0^1 e^{L_v(1-s)}\varepsilon\,ds
\qquad \text{a.s.}
\]
Taking expectation preserves this bound. Now let $r$ be $L_r$-Lipschitz. Since
\[
|r(X_1)-r(X_1^*)|
\le
L_r\|X_1-X_1^*\|
\qquad \text{a.s.},
\]
we obtain
\[
|\E_{p_1}[r]-\E_{p_1^*}[r]|
\le
L_r\,\E\|X_1-X_1^*\|
\le
L_r\int_0^1 e^{L_v(1-s)}\varepsilon\,ds.
\]

\textbf{Part 2: on-policy control.}
Part 1 relied on a pathwise bound for the field mismatch term at each time $t$. Under on-policy control we no longer have that: the quantity $\|v(X_t,t)-v^*(X_t,t)\|$ is controlled only after averaging over the rollout law. We therefore rewrite the trajectory gap in integral form and then take expectation, which puts the assumption on the error function
\[
e(t)=\E_{X_t\sim p_t}\bigl[\|v(X_t,t)-v^*(X_t,t)\|\bigr]
\]
directly into the estimate. Starting from the integral form,
\[
\Delta_t
=
\int_0^t \bigl(v(X_s,s)-v^*(X_s^*,s)\bigr)\,ds
\]
\[
=
\int_0^t \bigl(v(X_s,s)-v^*(X_s,s)\bigr)\,ds
+
\int_0^t \bigl(v^*(X_s,s)-v^*(X_s^*,s)\bigr)\,ds.
\]
Taking norms and using the triangle inequality gives
\[
\|\Delta_t\|
\le
\int_0^t \|v(X_s,s)-v^*(X_s,s)\|\,ds
+
\int_0^t \|v^*(X_s,s)-v^*(X_s^*,s)\|\,ds.
\]
Now use that $v^*(\cdot,s)$ is $L_v$-Lipschitz in $x$, so
\[
\|v^*(X_s,s)-v^*(X_s^*,s)\|
\le
L_v\|X_s-X_s^*\|
=
L_v\|\Delta_s\|.
\]
Substituting this back yields
\[
\|\Delta_t\|
\le
\int_0^t \|v(X_s,s)-v^*(X_s,s)\|\,ds
+
L_v\int_0^t \|\Delta_s\|\,ds.
\tag{$*$}
\]
Now define
\[
m(t):=\E\|X_t-X_t^*\|=\E\|\Delta_t\|.
\]
Taking expectation in \((*)\) and using Fubini's theorem yields
\[
m(t)
\le
\int_0^t \E\|v(X_s,s)-v^*(X_s,s)\|\,ds
+
L_v\int_0^t m(s)\,ds
=
\int_0^t e(s)\,ds + L_v\int_0^t m(s)\,ds.
\]
To bring this into differential form, set
\[
M(t):=\int_0^t e(s)\,ds + L_v\int_0^t m(s)\,ds.
\]
Then $m(t)\le M(t)$, $M(0)=0$, and for almost every $t$,
\[
M'(t)=e(t)+L_v m(t)\le e(t)+L_v M(t).
\]
Applying \cref{lem:integral_gronwall} to $M$ with $b=e$ gives
\[
M(t)\le \int_0^t e^{L_v(t-s)} e(s)\,ds.
\]
Evaluating at $t=1$ and using $m\le M$, we conclude that
\[
m(1)
\le
\int_0^1 e^{L_v(1-s)} e(s)\,ds.
\]
Therefore
\[
\E\|X_1-X_1^*\|
\le
\int_0^1 e^{L_v(1-s)} e(s)\,ds.
\]
If $r$ is $L_r$-Lipschitz, then
\[
|r(X_1)-r(X_1^*)|
\le
L_r\|X_1-X_1^*\|
\qquad \text{a.s.},
\]
so
\[
|\E_{p_1}[r]-\E_{p_1^*}[r]|
\le
L_r\,\E\|X_1-X_1^*\|
\le
L_r\int_0^1 e^{L_v(1-s)} e(s)\,ds.
\]
If $\sup_t e(t)\le\varepsilon$, then
\[
\int_0^1 e^{L_v(1-s)}e(s)\,ds
\le
\int_0^1 e^{L_v(1-s)}\varepsilon\,ds,
\]
which recovers the same endpoint constant as in part 1.

\textbf{Part 3: tightness.}
Work in one dimension. Fix any $C^1$ function
\[
\theta:[0,1]\to\left[0,\frac{\pi}{2}\right]
\]
satisfying
\[
\theta(0)=0,
\qquad
\theta(1)=\frac{\pi}{2},
\qquad
\theta(t)\in\left(0,\frac{\pi}{2}\right)
\quad\text{for } t\in(0,1).
\]
Let
\[
U_0 \sim \mathcal{N}(0,1),
\qquad
U_1 \sim \mathcal{N}(0,e^{2L_v}),
\]
with $U_0$ and $U_1$ independent, and define
\[
\bar X_t=\alpha(t)U_1+\beta(t)U_0,
\qquad
Y_t=\dot\alpha(t)U_1+\dot\beta(t)U_0,
\]
where
\[
\alpha(t):=e^{L_v(t-1)}\sin\theta(t),
\qquad
\beta(t):=e^{L_v t}\cos\theta(t).
\]
Then $\alpha(0)=0$, $\beta(0)=1$, $\alpha(1)=1$, and $\beta(1)=0$, with $\alpha(t),\beta(t)>0$ for $t\in(0,1)$, so this is a valid interpolation from $\mathcal{N}(0,1)$ to $\mathcal{N}(0,e^{2L_v})$. The population CFM minimizer is
\[
v^*(x,t)=\E[Y_t\mid \bar X_t=x].
\]
Since $(\bar X_t,Y_t)$ is centered and jointly Gaussian,
\[
v^*(x,t)=\frac{\Cov(Y_t,\bar X_t)}{\Var(\bar X_t)}x.
\]
A direct calculation from the definitions of $\alpha$ and $\beta$ gives
\[
\Var(\bar X_t)
=
\alpha(t)^2 e^{2L_v}+\beta(t)^2
=
e^{2L_v t},
\qquad
\Cov(Y_t,\bar X_t)=
L_v e^{2L_v t}.
\]
Hence $v^*(x,t)=L_v x$.

Now fix any nonnegative function $e\in L^1([0,1])$ and define
\[
\Delta v(x,t):=e(t),
\qquad
v(x,t):=v^*(x,t)+\Delta v(x,t)=L_v x+e(t).
\]
Let $W_0\sim\mathcal{N}(0,1)$ and run the coupled ODEs
\[
\dot X_t^*=v^*(X_t^*,t),
\qquad
\dot X_t=v(X_t,t),
\qquad
X_0^*=X_0=W_0.
\]
A direct integration gives
\[
X_t^*=e^{L_v t}W_0,
\qquad
X_t=e^{L_v t}W_0+\int_0^t e^{L_v(t-s)}e(s)\,ds.
\]
Therefore
\[
X_1-X_1^*
=
\int_0^1 e^{L_v(1-s)}e(s)\,ds.
\]
Since $\Delta v$ is independent of $x$, the on-policy error function coincides with the prescribed function:
\[
\E_{X_t\sim p_t}\bigl[|v(X_t,t)-v^*(X_t,t)|\bigr]
=
e(t).
\]
Therefore
\[
\E|X_1-X_1^*|
=
\int_0^1 e^{L_v(1-s)}e(s)\,ds,
\]
so part 2 is attained with equality. For the linear reward
\[
r(x)=L_r x,
\]
\[
|\E_{p_1}[r]-\E_{p_1^*}[r]|
=
L_r\left|\E[X_1-X_1^*]\right|
=
L_r\int_0^1 e^{L_v(1-s)}e(s)\,ds.
\]
Thus the reward bound in part 2 is also attained with equality.

If $e(t)\equiv \varepsilon$, then $\|v-v^*\|_\infty=\varepsilon$ and
\[
\E|X_1-X_1^*|
=
\int_0^1 e^{L_v(1-s)}\varepsilon\,ds
=
\begin{cases}
    \dfrac{\varepsilon}{L_v}(e^{L_v}-1), & L_v>0,\\[6pt]
    \varepsilon, & L_v=0.
\end{cases}
\]
so part 1 is also attained with equality.

The exact reward-tightness statement above uses the linear test function $r(x)=L_r x$, which is Lipschitz but unbounded. If one insists on bounded rewards while keeping this same amplification-tight construction, the endpoint Gaussians can be rescaled to make range truncation negligible. Indeed, replace the endpoint Gaussians by $U_0\sim\mathcal N(0,\sigma^2)$ and $U_1\sim\mathcal N(0,\sigma^2 e^{2L_v})$. The calculation of the CFM minimizer is unchanged, so $v^*(x,t)=L_vx$, and the perturbation $\Delta v(x,t)=e(t)$ still shifts every endpoint by
\[
A:=\int_0^1 e^{L_v(1-s)}e(s)\,ds.
\]
Taking $\sigma$ small makes $X_1^*$ arbitrarily concentrated, so a clipped affine reward of slope $L_r$ can be placed so that, with probability arbitrarily close to one, the shift by $A$ either remains inside the linear region, giving gap $L_rA$, or crosses from the zero plateau to the one plateau, giving gap $1$. Hence for every $\zeta>0$ there is a bounded $L_r$-Lipschitz reward such that
\[
\bigl|\E_{p_1}[r]-\E_{p_1^*}[r]\bigr|
\ge
\min\{1,L_rA\}-\zeta.
\]
Thus the bounded-reward version of this amplification-tight construction is essentially tight as well.

\end{proof}

\paragraph{Remark: A simpler bounded-reward tightness statement.}
The tightness construction in part 3 is engineered to make the full amplification kernel $e^{L_v(1-s)}$ sharp, and for that reason it uses a somewhat contrived setup. This is useful for completeness, but it is more than is needed to show the main geometric point in the paper, namely that the factor $L_r\varepsilon$ itself cannot be removed. The next result isolates that point already for the standard linear Gaussian bridge with $p_0=q=\mathcal{N}(0,1)$.

\begin{theorem}[Bounded-reward tightness of the $L_r\varepsilon$ factor]
\label{thm:bounded_reward_lipschitz_tightness}
There is a universal constant $c>0$ such that the following holds. Fix $\varepsilon>0$ and $L>0$ with $\varepsilon L<1$, and consider the one-dimensional CFM instance with standard linear interpolation
\[
X_t=tX_1+(1-t)X_0,
\qquad
X_0,X_1\stackrel{\mathrm{i.i.d.}}{\sim}\mathcal{N}(0,1),
\]
whose source and target laws are both $\mathcal{N}(0,1)$ and whose population velocity $v^*$ is $1$-Lipschitz in space. Then there is a velocity field $v$ satisfying
\[
\sup_{t\in[0,1],\,x\in\mathbb{R}} |v(x,t)-v^*(x,t)|\le \varepsilon
\]
and a reward $r:\mathbb{R}\to[0,1]$ with $\operatorname{Lip}(r)\le L$ such that the endpoint law $p$ of the ODE driven by $v$ satisfies
\[
\E_p[r]-\E_q[r]\ge c\,L\varepsilon,
\qquad q=\mathcal{N}(0,1).
\]
The explicit nonsmooth construction below gives $c=1/2$.
\end{theorem}

\noindent\emph{Intuition.}
The idea is to build a reward together with a matching velocity perturbation that pushes every generated sample uphill of the reward. The reward is a periodic triangular wave: peaks of height $1$ spaced $2/L$ apart, falling linearly to $0$ at the midpoints. This makes the reward bounded in $[0,1]$, $L$-Lipschitz, and ensures that every point is within distance $1/L$ of a peak. The bad velocity field is $v^*$ plus an admissible nudge of size $\varepsilon$ pointing each sample toward its nearest peak. Over unit time, this moves each sample up a slope of magnitude $L$ until it has traveled distance $\varepsilon$ or reached a peak, so the reward gain is of order $L\varepsilon$ on average. Since every sample moves locally uphill, there is no cancellation between samples moving in different directions. The gain is therefore insensitive to how spread out the samples are, which is why the standard bridge with $p_0=q=\mathcal N(0,1)$ already realizes an $\Omega(L\varepsilon)$ gap.

\begin{proof}
Write
\[
\rho_t:=\sqrt{t^2+(1-t)^2}.
\]
Then $q_t=\mathcal{N}(0,\rho_t^2)$, with $\rho_0=\rho_1=1$ and $1/\sqrt{2}\le \rho_t\le 1$.

First, the reference flow is a pure rescaling. The population CFM velocity is
\[
v^*(x,t)=\E[X_1-X_0\mid X_t=x].
\]
Since $(X_0,X_1,X_t)$ is centered and jointly Gaussian, as earlier we can compute this explicitly as:
\[
v^*(x,t)
=
\frac{\Cov(X_1-X_0,X_t)}{\Var(X_t)}x
=
\frac{2t-1}{2t^2-2t+1}x
=
\frac{\dot\rho_t}{\rho_t}x.
\]
Since
\[
\partial_x v^*(x,t)=\frac{2t-1}{2t^2-2t+1},
\qquad
\left|\frac{2t-1}{2t^2-2t+1}\right|\le 1
\quad \text{for all }t\in[0,1],
\]
the field $v^*(\cdot,t)$ is $1$-Lipschitz uniformly in $t$.
Integrating $\dot X_t=(\dot\rho_t/\rho_t)X_t$ gives $X_t=\rho_tX_0$. Thus, in the rescaled coordinate
\[
Y_t:=\frac{X_t}{\rho_t},
\]
the reference dynamics are frozen, $\dot Y_t=0$, and the reference endpoint is $X_1^*=Y_0\sim\mathcal{N}(0,1)=q$.

We now construct a bounded reward and an uphill perturbation to get the desired bound. Set $T:=2/L$. For a phase $\theta\in[0,T]$, to be chosen later, place reward peaks at the equally spaced points
\[
\{nT-\theta:n\in\mathbb{Z}\}.
\]
Let $d_\theta(y)$ be the distance from $y$ to the nearest peak, and define the triangular-wave reward
\[
r_\theta(y):=1-L\,d_\theta(y).
\]
Every point is at most half a period, $T/2=1/L$, from a nearest peak, so $0\le r_\theta\le 1$. Since $d_\theta$ is $1$-Lipschitz, $r_\theta$ is $L$-Lipschitz.
Let $b_\theta(y)\in\{-1,0,+1\}$ be the unit direction of steepest ascent of $r_\theta$: it equals $+1$ if a nearest peak lies to the right of $y$, $-1$ if a nearest peak lies to the left, and $0$ at peaks and troughs. Define
\[
v_\theta(x,t)
:=
v^*(x,t)+\rho_t\varepsilon\,b_\theta\!\left(\frac{x}{\rho_t}\right).
\]
Because $\rho_t\le 1$ and $|b_\theta|\le 1$,
\[
\sup_{t,x}|v_\theta(x,t)-v^*(x,t)|
=
\sup_{t,x}\rho_t\varepsilon
\left|b_\theta\!\left(\frac{x}{\rho_t}\right)\right|
\le \varepsilon.
\]
Substituting this field into the dynamics of $Y_t=X_t/\rho_t$, the rescaling terms cancel and
\[
\dot Y_t=\varepsilon b_\theta(Y_t).
\]
Thus, in the rescaled coordinate, every trajectory moves at speed $\varepsilon$ toward a nearest reward peak and stops on arrival.

Fix $Y_0=y$ away from the null set of peaks and troughs. Over the unit time interval the trajectory moves a distance $\min\{\varepsilon,d_\theta(y)\}$ toward the nearest peak, so $d_\theta$ decreases by exactly that amount and
\[
r_\theta(Y_1)-r_\theta(Y_0)
=
L\min\{\varepsilon,d_\theta(Y_0)\}.
\]
Couple the reference and perturbed flows through the same start $Y_0\sim\mathcal{N}(0,1)$. The reference endpoint is $X_1^*=Y_0$, while the perturbed endpoint is $X_1=Y_1$ because $\rho_1=1$. Therefore, for the fixed reward $r_\theta$,
\[
G(\theta)
:=
\E_p[r_\theta]-\E_q[r_\theta]
=
\E\bigl[r_\theta(Y_1)-r_\theta(Y_0)\bigr]
=
L\,\E_{Z\sim\mathcal{N}(0,1)}
\bigl[\min\{\varepsilon,d_\theta(Z)\}\bigr].
\]

It remains only to choose the phase. Average $G$ over $\theta\sim\mathrm{Unif}[0,T]$. For each fixed $z$, as $\theta$ ranges over one period, the location of $z$ relative to the nearest peak is uniform over one period. Hence $d_\theta(z)$ is uniform on $[0,T/2]=[0,1/L]$ with density $L$. Since $\varepsilon<1/L$,
\[
\E_\theta\bigl[\min\{\varepsilon,d_\theta(z)\}\bigr]
=
L\int_0^{1/L}\min\{\varepsilon,u\}\,du
=
L\left(
\int_0^\varepsilon u\,du
+
\int_\varepsilon^{1/L}\varepsilon\,du
\right)
=
\varepsilon-\frac{L\varepsilon^2}{2}.
\]
This quantity is independent of $z$, so
\[
\E_\theta G(\theta)
=
L\left(\varepsilon-\frac{L\varepsilon^2}{2}\right)
=
L\varepsilon\left(1-\frac{L\varepsilon}{2}\right)
\ge
\frac{1}{2}L\varepsilon,
\]
where the last inequality uses $L\varepsilon<1$. Therefore some phase $\theta^*$ satisfies $G(\theta^*)\ge L\varepsilon/2$ by simple properties of the expectation. Taking $r=r_{\theta^*}$ and $v=v_{\theta^*}$ proves the theorem with $c=1/2$.
\end{proof}

\subsection{How RL Helps}

\subsubsection{RL reward regret bound}
\label{app:rl_reward_bound}

\begin{proposition}[RL reward regret bound]
\label{app:prop:rl_good_reward}
Let $r:\X \to [0,1]$, let $\mathbb{P}_{\mathrm{base}}$ be the base trajectory distribution, define
\[
\frac{d\mathbb{P}^*}{d\mathbb{P}_{\mathrm{base}}}(X_{0:1})
=
\frac{\exp(\lambda r(X_1))}{Z},
\qquad
Z := \E_{\mathbb{P}_{\mathrm{base}}}[\exp(\lambda r(X_1))],
\]
and let $\mathbb{P}$ be any trajectory distribution with final-sample marginal $p$. Then the final-sample marginal of $\mathbb{P}^*$ is the tilted target $p^*$ defined in \Cref{sec:posttraining}, and
\begin{equation}
\E_{p^*}[r] - \E_p[r]
\leq
\sqrt{\tfrac{1}{2}\KL{\mathbb{P}}{\mathbb{P}^*}}
=
\sqrt{\tfrac{\lambda}{2}\bigl(\mathcal{L}_{\mathrm{RL}}(\mathbb{P}) - \mathcal{L}_{\mathrm{RL}}(\mathbb{P}^*)\bigr)},
\label{eq:app_rl_reward_bound}
\end{equation}
where $\mathcal{L}_{\mathrm{RL}}(\mathbb{P}) = \frac{1}{\lambda}\KL{\mathbb{P}}{\mathbb{P}_{\mathrm{base}}} - \E_p[r]$.
\end{proposition}

\begin{proof}
Here $\mathbb{P}_{\mathrm{base}}$, $\mathbb{P}$, and $\mathbb{P}^*$ are distributions over full sampled trajectories $X_{0:1}$; formally, they are path measures. Because the Radon--Nikodym derivative of $\mathbb{P}^*$ depends only on the terminal state $X_1$, the final-sample marginal of $\mathbb{P}^*$ is exactly
\[
p^*(x) = \frac{\exp(\lambda r(x))\, p_{\mathrm{base}}(x)}{Z},
\]
which is exactly the tilted target. Since
\[
\TV{p}{p^*}
=
\sup_{f:\, 0 \leq f \leq 1} \left|\E_p[f] - \E_{p^*}[f]\right|,
\]
the assumption $r \in [0,1]$ gives
\[
\E_{p^*}[r] - \E_p[r] \leq \TV{p}{p^*}.
\]
With the convention $\TV{p}{p^*} = \frac{1}{2}\|p-p^*\|_1$, Pinsker's inequality yields
\[
\TV{p}{p^*} \leq \sqrt{\tfrac{1}{2}\KL{p}{p^*}},
\]
and the data processing inequality for the endpoint map $X_{0:1}\mapsto X_1$ gives
\[
\KL{p}{p^*} \leq \KL{\mathbb{P}}{\mathbb{P}^*}.
\]
Therefore
\[
\E_{p^*}[r] - \E_p[r]
\leq
\sqrt{\tfrac{1}{2}\KL{\mathbb{P}}{\mathbb{P}^*}}.
\]
The first inequality in \eqref{eq:app_rl_reward_bound} therefore follows by data processing for the endpoint map $X_{0:1}\mapsto X_1$, which gives $\KL{p}{p^*} \le \KL{\mathbb{P}}{\mathbb{P}^*}$. It remains to identify this trajectory-level KL with the RL objective gap. By the definition of $\mathbb{P}^*$,
\begin{align*}
\KL{\mathbb{P}}{\mathbb{P}^*}
&= \E_{\mathbb{P}}\left[\log \frac{d\mathbb{P}}{d\mathbb{P}^*}\right] \\
&= \E_{\mathbb{P}}\left[\log \frac{d\mathbb{P}}{d\mathbb{P}_{\mathrm{base}}}\right]
-
\E_{\mathbb{P}}\left[\log \frac{d\mathbb{P}^*}{d\mathbb{P}_{\mathrm{base}}}\right] \\
&= \KL{\mathbb{P}}{\mathbb{P}_{\mathrm{base}}}
-
\E_{\mathbb{P}}\left[\lambda r(X_1) - \log Z\right] \\
&= \KL{\mathbb{P}}{\mathbb{P}_{\mathrm{base}}} - \lambda \E_p[r] + \log Z.
\end{align*}
Rearranging gives
\[
\mathcal{L}_{\mathrm{RL}}(\mathbb{P})
=
\frac{1}{\lambda}\KL{\mathbb{P}}{\mathbb{P}^*} - \frac{\log Z}{\lambda}.
\]
Since $\log Z / \lambda$ is constant, evaluating the same identity at $\mathbb{P}^*$ and subtracting gives
\[
\KL{\mathbb{P}}{\mathbb{P}^*} = \lambda\bigl(\mathcal{L}_{\mathrm{RL}}(\mathbb{P}) - \mathcal{L}_{\mathrm{RL}}(\mathbb{P}^*)\bigr),
\]
which is the second equality in \eqref{eq:app_rl_reward_bound}.
\end{proof}

\FloatBarrier
\clearpage
\section{Proofs for Section~\ref{sec:implications_fixing_sft}: Method: Discriminator-Guided RL}
\label{app:implications}

\subsection{Why Representation Spaces Help}

\subsubsection{Feature-space correction}
\label{app:feature_correction}

For readability, we work throughout this subsection in the standard density
setting. Assume every distribution we mention admits a well-defined continuous
density, and that the
corresponding conditionals $p(x\mid z)$ for $z=\phi(x)$ are well defined. We
write $p^\phi(z)$ for the density on $\mathcal Z$ induced by passing
$x\sim p$ through $\phi$, and assume that $q^\phi(z)>0$ only where
$p_{\mathrm{base}}^\phi(z)>0$.

\begin{proposition}[Feature-space correction]
\label{prop:feature_correction}
Let $\phi:\mathcal X\to\mathcal Z$ be an encoder, let
$h:\mathcal Z\to\mathbb R$ be any scalar function, and let $\lambda>0$. Assume
\[
Z_{h,\lambda}
:=
\E_{x\sim p_{\mathrm{base}}}\!\left[\exp\!\bigl(\lambda h(\phi(x))\bigr)\right]
<\infty.
\]
Define the tilted density by
\begin{equation}
p_{h,\lambda}(x)
:=
\frac{\exp(\lambda h(\phi(x)))}{Z_{h,\lambda}}\,p_{\mathrm{base}}(x).
\label{eq:app_feature_tilt}
\end{equation}
Then:
\begin{enumerate}
    \item[(i)] the feature marginal is tilted according to
    \begin{equation}
    p_{h,\lambda}^\phi(z)
    =
    \frac{\exp(\lambda h(z))}{Z_{h,\lambda}}\,p_{\mathrm{base}}^\phi(z),
    \label{eq:app_feature_marginal}
    \end{equation}
    where equivalently
    \[
    Z_{h,\lambda}
    =
    \E_{z\sim p_{\mathrm{base}}^\phi}\!\left[\exp(\lambda h(z))\right];
    \]
    \item[(ii)] the conditional distribution inside each feature cell is unchanged:
    \begin{equation}
    p_{h,\lambda}(x\mid z)
    =
    p_{\mathrm{base}}(x\mid z)
    \qquad
    \text{for } p_{h,\lambda}^\phi\text{-a.e. }z.
    \label{eq:app_conditional_unchanged}
    \end{equation}
    \item[(iii)] if
    \begin{equation}
    h^*(z)=\log\frac{q^\phi(z)}{p_{\mathrm{base}}^\phi(z)},
    \label{eq:app_feature_ratio}
    \end{equation}
    then $p_{h^*,1}^\phi=q^\phi$;
    \item[(iv)] $p_{h^*,1}$ is the unique solution of
    \begin{equation}
    \min_p \KL{p}{p_{\mathrm{base}}}
    \qquad
    \text{subject to}
    \qquad
    p^\phi=q^\phi.
    \label{eq:app_feature_projection}
    \end{equation}
\end{enumerate}
\end{proposition}

\begin{proof}
Let $z=\phi(x)$. For part (i), let $g:\mathcal Z\to\mathbb R$ be any bounded measurable test
function. Then
\[
\begin{aligned}
\E_{z\sim p_{h,\lambda}^\phi}[g(z)]
&=
\E_{x\sim p_{h,\lambda}}[g(\phi(x))] \\
&=
\frac{1}{Z_{h,\lambda}}
\E_{x\sim p_{\mathrm{base}}}\!\left[
g(\phi(x))\,\exp(\lambda h(\phi(x)))
\right] \\
&=
\frac{1}{Z_{h,\lambda}}
\int_{\mathcal Z}
g(z)\,\exp(\lambda h(z))\,p_{\mathrm{base}}^\phi(z)\,dz.
\end{aligned}
\]
Thus the pushed-forward law under $\phi$ has density
\[
p_{h,\lambda}^\phi(z)
=
\frac{\exp(\lambda h(z))}{Z_{h,\lambda}}\,p_{\mathrm{base}}^\phi(z),
\]
since the right-hand side gives the same integral against every bounded
measurable test function $g$. Hence
\eqref{eq:app_feature_marginal}. Taking $g\equiv 1$ in the display above gives
\[
1
=
\frac{1}{Z_{h,\lambda}}
\int_{\mathcal Z} \exp(\lambda h(z))\,p_{\mathrm{base}}^\phi(z)\,dz,
\]
so
\[
Z_{h,\lambda}
=
\E_{z\sim p_{\mathrm{base}}^\phi}\!\left[\exp(\lambda h(z))\right].
\]
This completes part (i).

For part (ii), fix a measurable set $A\subseteq\mathcal X$, and let
$g:\mathcal Z\to\mathbb R$ be any bounded measurable test function. On the one
hand, conditioning under $p_{h,\lambda}$ gives
\[
\begin{aligned}
\E_{x\sim p_{h,\lambda}}\!\left[\mathbf 1_A(x)\,g(\phi(x))\right]
&=
\int_{\mathcal Z} g(z)\,p_{h,\lambda}(A\mid z)\,p_{h,\lambda}^\phi(z)\,dz.
\end{aligned}
\]
On the other hand, using the definition of $p_{h,\lambda}$ and then
conditioning under $p_{\mathrm{base}}$,
\[
\begin{aligned}
\E_{x\sim p_{h,\lambda}}\!\left[\mathbf 1_A(x)\,g(\phi(x))\right]
&=
\frac{1}{Z_{h,\lambda}}
\E_{x\sim p_{\mathrm{base}}}\!\left[
\mathbf 1_A(x)\,g(\phi(x))\,\exp(\lambda h(\phi(x)))
\right] \\
&=
\frac{1}{Z_{h,\lambda}}
\int_{\mathcal Z}
g(z)\,\exp(\lambda h(z))\,p_{\mathrm{base}}(A\mid z)\,p_{\mathrm{base}}^\phi(z)\,dz \\
&=
\int_{\mathcal Z}
g(z)\,p_{\mathrm{base}}(A\mid z)\,p_{h,\lambda}^\phi(z)\,dz,
\end{aligned}
\]
where the last step uses part (i). Since these two expressions agree for every
bounded measurable test function $g$, their integrands must agree
$dz$-a.e.:
\[
p_{h,\lambda}(A\mid z)\,p_{h,\lambda}^\phi(z)
=
p_{\mathrm{base}}(A\mid z)\,p_{h,\lambda}^\phi(z).
\]
Therefore
\[
p_{h,\lambda}(A\mid z)
=
p_{\mathrm{base}}(A\mid z)
\qquad
\text{for } p_{h,\lambda}^\phi\text{-a.e. }z.
\]
Since this holds for every measurable $A$, the conditional distributions agree.
In our density setting, we record this as
\[
p_{h,\lambda}(x\mid z)=p_{\mathrm{base}}(x\mid z),
\]
which proves part (ii).

For part (iii), let
\[
h^*(z)=\log\frac{q^\phi(z)}{p_{\mathrm{base}}^\phi(z)}.
\]
Then
\[
Z_{h^*,1}
=
\int \exp(h^*(z))\,p_{\mathrm{base}}^\phi(z)\,dz
=
\int q^\phi(z)\,dz
=
1.
\]
Therefore, by part (i),
\[
p_{h^*,1}^\phi(z)
=
\exp(h^*(z))\,p_{\mathrm{base}}^\phi(z)
=
q^\phi(z),
\]
proving part (iii).

For part (iv), let $p$ be any distribution satisfying $p^\phi=q^\phi$. The KL
chain rule gives
\begin{equation}
\KL{p}{p_{\mathrm{base}}}
=
\KL{p^\phi}{p_{\mathrm{base}}^\phi}
+
\E_{z\sim p^\phi}\!\left[
\KL{p(\cdot\mid z)}{p_{\mathrm{base}}(\cdot\mid z)}
\right].
\label{eq:app_kl_chain_rule}
\end{equation}
Under the constraint $p^\phi=q^\phi$, this becomes
\[
\KL{p}{p_{\mathrm{base}}}
=
\KL{q^\phi}{p_{\mathrm{base}}^\phi}
+
\E_{z\sim q^\phi}\!\left[
\KL{p(\cdot\mid z)}{p_{\mathrm{base}}(\cdot\mid z)}
\right]
\ge
\KL{q^\phi}{p_{\mathrm{base}}^\phi}.
\]
The first term is fixed by the constraint, and the second term is always
nonnegative. Equality holds if and only if
\[
p(\cdot\mid z)=p_{\mathrm{base}}(\cdot\mid z)
\qquad
\text{for } q^\phi\text{-a.e. }z.
\]
Hence the unique minimizer is the distribution with feature marginal $q^\phi$
and within-cell conditional $p_{\mathrm{base}}(\cdot\mid z)$. Parts (ii) and
(iii) show that $p_{h^*,1}$ has exactly this marginal and these conditionals,
so it attains the lower bound. Uniqueness follows because the marginal
$q^\phi$ together with the conditionals $p_{\mathrm{base}}(\cdot\mid z)$
determines the distribution $p$ uniquely.
\end{proof}

\paragraph{Remark.}
We stated this subsection in the density setting because it is the clearest way
to express the main idea: RL reweights feature cells by the feature-space
density ratio and leaves the within-cell conditional unchanged. A more general
measure-theoretic version replaces ordinary ratios such as
$q^\phi(z)/p_{\mathrm{base}}^\phi(z)$ with the Radon--Nikodym derivative
$dq^\phi/dp_{\mathrm{base}}^\phi$, replaces $p(x\mid z)$ with regular
conditional distributions, and writes feature-space integrals against
$p_{\mathrm{base}}^\phi(dz)$ rather than $dz$. The proof is otherwise the same.

\subsubsection{Feature-Space Test-Function Bound}
\label{app:feature_bound_proof}

\begin{theorem}[Feature-space test-function bound]
Let $\hat r(x)=h(\phi(x))$ be a feature-based reward with implied target
$\hat q_\lambda \propto \exp(\lambda \hat r)\,p_{\mathrm{base}}$, and suppose
$p_\theta$ is $\Delta$-suboptimal for the KL-regularized objective with
trade-off parameter $\lambda$. Then for any
$g:\mathcal{Z}\to\mathbb{R}$ with $\|g\|_\infty\leq C$,
\begin{equation}
  \bigl|\E_{p_\theta}[g(\phi(x))]-\E_q[g(\phi(x))]\bigr|
  \;\leq\;
  \underbrace{C\sqrt{2\lambda\Delta}}_{\text{RL suboptimality}}
  \;+\;
  \underbrace{\bigl|\E_{\hat q_\lambda}[g(\phi(x))]
              -\E_q[g(\phi(x))]\bigr|}_{\text{feature-space ratio error}}.
  \label{eq:feature_bound}
\end{equation}
\end{theorem}

\begin{proof}
By the triangle inequality,
\begin{equation}
\bigl|\E_{p_\theta}[g(\phi(x))] - \E_q[g(\phi(x))]\bigr|
\leq \bigl|\E_{p_\theta}[g \circ \phi] - \E_{\hat{q}_\lambda}[g \circ \phi]\bigr| + \bigl|\E_{\hat{q}_\lambda}[g \circ \phi] - \E_q[g \circ \phi]\bigr|.
\end{equation}

\paragraph{First term (RL suboptimality).}
Since $\norm{g}_\infty \leq C$, the function $g \circ \phi$ takes values in $[-C, C]$. Also, because $\hat{q}_\lambda(x) \propto \exp(\lambda \hat r(x))\, p_{\mathrm{base}}(x)$,
\begin{equation}
\mathcal{L}_{\mathrm{RL}}(p) = \frac{1}{\lambda}\KL{p}{p_{\mathrm{base}}} - \E_p[\hat r] = \frac{1}{\lambda}\KL{p}{\hat q_\lambda} - \frac{1}{\lambda}\log \hat Z_\lambda
\end{equation}
for a constant $\hat Z_\lambda$ independent of $p$. Therefore the unique minimizer of the RL loss is $\hat q_\lambda$, and the $\Delta$-suboptimality assumption implies $\KL{p_\theta}{\hat q_\lambda}\le \lambda\Delta$. By Pinsker's inequality,
\begin{equation}
\bigl|\E_{p_\theta}[g \circ \phi] - \E_{\hat{q}_\lambda}[g \circ \phi]\bigr|
\leq 2C\,\TV{p_\theta}{\hat{q}_\lambda}
\leq 2C\sqrt{\frac{\lambda\Delta}{2}}
= C\sqrt{2\lambda\Delta}.
\end{equation}

\paragraph{Second term (feature-space ratio error).}
Since $g \circ \phi$ depends on $x$ only through $\phi(x)$,
\begin{equation}
\E_{\hat{q}_\lambda}[g(\phi(x))] = \E_{\hat{q}_\lambda^\phi}[g]
\quad\text{and}\quad
\E_q[g(\phi(x))] = \E_{q^\phi}[g],
\end{equation}
so $\bigl|\E_{\hat{q}_\lambda}[g \circ \phi] - \E_q[g \circ \phi]\bigr| = \bigl|\E_{\hat{q}_\lambda^\phi}[g] - \E_{q^\phi}[g]\bigr|$. Combining gives \eqref{eq:feature_bound}.
\end{proof}

\FloatBarrier
\clearpage
\section{Adjoint Matching: Background}
\label{app:adjoint_matching}

This appendix contains a self-contained review of adjoint matching~\citep{domingoenrich2025adjoint}, the method we use to optimize the RL objective $\mathcal{L}_{\mathrm{RL}}$ \eqref{eq:rl_path_obj}. We first provide a general description of the method in the SOC setting and then specialize to flow models.

\paragraph{Stochastic optimal control formulation.}
\label{app:am:soc}
Adjoint matching solves $\mathcal{L}_{\mathrm{RL}}$~\eqref{eq:rl_path_obj} under the lens of stochastic optimal control (SOC), a framework for fine-tuning the dynamics of an SDE through a learned drift correction. Concretely, given a base SDE with drift $b$ and noise $\sigma(t)$, SOC adds a learned control $u_\theta(x, t)$ to the drift,
\begin{equation}
dX_t = \bigl[ b(X_t, t) + \sigma(t)\, u_\theta(X_t, t) \bigr]\, dt + \sigma(t)\, dB_t,
\qquad X_0 \sim \mathcal{N}(0, I),
\label{eq:controlled_sde}
\end{equation}
and seeks the $u_\theta$ minimizing a quadratic control cost minus the terminal reward,
\begin{equation}
\min_{u_\theta}\; \E\!\left[ \tfrac{1}{2}\!\int_0^1\!\|u_\theta(X_t, t)\|^2\, dt \;-\; \lambda\, r(X_1) \right].
\label{eq:soc_obj}
\end{equation}
In our setting $b(x, t) = v_t(x) + \tfrac{1}{2}\sigma(t)^2\, s_t(x)$ is the base drift of the sampling SDE~\eqref{eq:ode_sde}---the dynamics that generate $p_{\mathrm{base}}$---and $u_\theta$ is the deviation we are learning to fine-tune the model. By Girsanov's theorem, the path-KL between \eqref{eq:controlled_sde} and the $\sigma(t)$-base SDE is exactly the expected control energy,
\[
\KL{\mathbb{P}_\theta}{\mathbb{P}_{\mathrm{base}}} = \tfrac{1}{2}\, \E\!\left[ \int_0^1 \|u_\theta(X_t, t)\|^2\, dt \right],
\]
so \eqref{eq:soc_obj} is precisely the path-KL-regularized RL objective $\mathcal{L}_{\mathrm{RL}}$~\eqref{eq:rl_path_obj}. Optimizing $u_\theta$ thus tunes the trajectory drift to push $X_1$ toward high reward while staying close to the base process in path-KL. For flow models we parametrize $u_\theta$ in terms of a learned velocity field; we defer the specifics to the specialization paragraph below.

\paragraph{The memoryless schedule.}
\label{app:am:memoryless}
Solving \eqref{eq:soc_obj} does not in general produce the desired endpoint tilt $p^*(x) \propto p_{\mathrm{base}}(x)\exp(\lambda r(x))$. Let
\[
V(x, t) := \min_{u_\theta}\, \E\!\left[\, \tfrac{1}{2}\!\int_t^1 \|u_\theta(X_s, s)\|^2\, ds \,-\, \lambda\, r(X_1) \,\,\big|\,\, X_t = x \right]
\]
denote the optimal cost-to-go from state $x$ at time $t$. \citet{domingoenrich2025adjoint} show that the SOC optimum picks up an $X_0$-dependent factor through the initial-time value $V(\cdot, 0)$,
\[
\mathbb{P}^*(X_0, X_1) \;\propto\; \mathbb{P}_{\mathrm{base}}(X_0, X_1)\, \exp\!\bigl(\lambda r(X_1) + V(X_0, 0)\bigr),
\]
so once we marginalize over $X_0$, the resulting tilt on $X_1$ is no longer just $\exp(\lambda r)$. The bias is largest when $X_0$ and $X_1$ are strongly coupled---exactly the case in flow matching, where the ODE makes $X_1$ a deterministic function of $X_0$.

To remove the bias, \citet{domingoenrich2025adjoint} propose choosing $\sigma$ during training so that the base process makes $X_0$ and $X_1$ independent. Under this \emph{memoryless} schedule, $V(X_0, 0)$ becomes a global constant and the endpoint marginal recovers the desired tilt $p^*(X_1) \propto p_{\mathrm{base}}(X_1)\exp(\lambda r(X_1))$. For the linear interpolation, they show that the unique memoryless schedule is
\begin{equation}
\sigma_{\mathrm{ml}}(t) = \sqrt{\frac{2(1-t)}{t}},
\label{eq:memoryless_schedule}
\end{equation}
large at $t\to 0$ (which destroys dependence on $X_0$) and small at $t\to 1$ (which keeps the endpoint clean). Importantly, \citet{domingoenrich2025adjoint} further show that the memoryless schedule is the \emph{unique} training schedule for which the fine-tuned optimum can be sampled at inference under any other noise schedule---in particular, the standard deterministic ODE ($\sigma(t)=0$). Training with any non-memoryless schedule locks the fine-tuned model to that specific schedule at inference.

\paragraph{Adjoint methods and adjoint matching.}
\label{app:am:loss}
A standard way of optimizing \eqref{eq:soc_obj} is to build an estimator of its gradient in $\theta$ by sampling a trajectory $\{X_t\}$ under $u_\theta$, computing the integrand, and backpropagating through the discretized SDE \eqref{eq:controlled_sde}. \citet{domingoenrich2025adjoint} show that the resulting gradient agrees in expectation with the gradient of the matching objective
\begin{equation}
L_{\mathrm{Basic-AM}}(\theta) = \tfrac{1}{2}\, \E\!\left[ \int_0^1 \bigl\| u_\theta(X_t, t) + \sigma(t)^\top a(t) \bigr\|^2\, dt \right],
\label{eq:basic_am}
\end{equation}
where $\{X_t\}$ and $\{a(t)\}$ are stop-gradient and $a(t)$ is the adjoint state
\[
a(t) \;:=\; \nabla_{X_t}\!\left[\, \tfrac{1}{2}\!\int_t^1\!\|u_\theta(X_s, s)\|^2\, ds \,-\, \lambda\, r(X_1) \right],
\]
computed by a single backward solve of the ODE
\begin{equation}
\dot a(t) = -\bigl[\nabla_x (b + \sigma\, u_\theta)(X_t, t)\bigr]^\top a(t) - \tfrac{1}{2}\nabla_x\!\bigl\|u_\theta(X_t, t)\bigr\|^2,
\qquad a(1) = -\lambda\, \nabla r(X_1).
\label{eq:full_adjoint}
\end{equation}

In practice, \citet{domingoenrich2025adjoint} find this empirically underperforms. They instead propose replacing $a(t)$ with the \emph{lean adjoint} $\tilde a(t)$, obtained by dropping the control-dependent terms of \eqref{eq:full_adjoint},
\begin{equation}
\dot{\tilde a}(t) = -\bigl[\nabla_x b_{\mathrm{base}}(X_t, t)\bigr]^\top \tilde a(t),
\qquad \tilde a(1) = -\lambda\, \nabla r(X_1).
\label{eq:lean_adjoint}
\end{equation}
The \emph{adjoint matching} loss is the same regression as \eqref{eq:basic_am} but with $\tilde a$ in place of $a$,
\begin{equation}
L_{\mathrm{AM}}(\theta) = \tfrac{1}{2}\, \E\!\left[ \int_0^1 \bigl\| u_\theta(X_t, t) + \sigma(t)^\top \tilde a(t) \bigr\|^2\, dt \right].
\label{eq:am_loss_general}
\end{equation}
\citet{domingoenrich2025adjoint} show that $L_{\mathrm{AM}}$ has the same minimizer $u^*$ as $L_{\mathrm{Basic-AM}}$. Crucially, the lean adjoint is still anchored to the reward through its terminal condition $\tilde a(1) = -\lambda\, \nabla r(X_1)$, so the regression target remains informed by the reward landscape; dropping the control-dependent terms only reduces variance and yields significantly better empirical performance.

\paragraph{Specialization to flow matching.}
For the linear interpolation, the base drift of \eqref{eq:ode_sde} under $\sigma_{\mathrm{ml}}$ from \eqref{eq:memoryless_schedule} is $b_{\mathrm{base}}(x, t) = 2\, v_{\mathrm{base}}(x, t) - x/t$. We parametrize the control $u_\theta$ so that the controlled drift $b_{\mathrm{base}} + \sigma_{\mathrm{ml}}\, u_\theta$ has the same shape as $b_{\mathrm{base}}$ but with a learned velocity $v_\theta$ in place of $v_{\mathrm{base}}$. Solving for $u_\theta$ yields
\begin{equation}
u_\theta(x, t) = \frac{2}{\sigma_{\mathrm{ml}}(t)}\bigl( v_\theta(x, t) - v_{\mathrm{base}}(x, t) \bigr).
\label{eq:fm_control}
\end{equation}
The reason for this parametrization is practical: fine-tuning then reduces to learning a new velocity field $v_\theta$, which can be plugged directly into the standard deterministic ODE for sampling at inference time, with no need for the memoryless SDE. Substituting \eqref{eq:fm_control} into \eqref{eq:controlled_sde} gives the corresponding \emph{memoryless training SDE},
\begin{equation}
dX_t = \left[ 2\, v_\theta(X_t, t) - \frac{X_t}{t} \right] dt + \sqrt{\frac{2(1-t)}{t}}\, dB_t.
\label{eq:memoryless_sde}
\end{equation}
Substituting \eqref{eq:fm_control} and $b_{\mathrm{base}}$ into \eqref{eq:lean_adjoint} and \eqref{eq:am_loss_general} gives the lean adjoint and adjoint matching loss for flow models,
\begin{align}
\dot{\tilde a}(t) &= -\Bigl[\nabla_x \!\bigl( 2 v_{\mathrm{base}}(X_t, t) - X_t/t \bigr)\Bigr]^\top \tilde a(t),
\qquad \tilde a(1) = -\lambda\, \nabla r(X_1),
\label{eq:fm_lean_adjoint}\\[0.25em]
L_{\mathrm{AM}}^{\mathrm{FM}}(\theta) &= \tfrac{1}{2}\, \E\!\left[ \int_0^1 \left\| \frac{2}{\sigma_{\mathrm{ml}}(t)}\bigl( v_\theta - v_{\mathrm{base}} \bigr) + \sigma_{\mathrm{ml}}(t)\, \tilde a(t) \right\|^2 dt \right].
\label{eq:am_loss}
\end{align}
The discretization of \eqref{eq:memoryless_sde} and \eqref{eq:fm_lean_adjoint} introduces numerical stiffness at $t \to 0$, which motivates the local linear integrator (\cref{app:local_linear_integrator}). Implementation choices are collected in \cref{app:additional_implementation_details}.

\FloatBarrier
\clearpage
\section{Local Linear Integrator for the Memoryless SDE}
\label{app:local_linear_integrator}

\ifmeta
\begin{figure}[H]
    \centering
    \includegraphics[width=\textwidth]{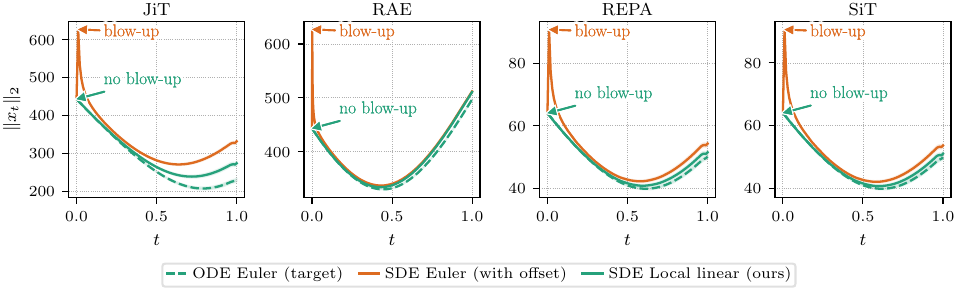}
    \caption{\textbf{Latent norm trajectories under Euler ODE, Euler SDE, and local linear SDE sampling.}
    Estimate of expected latent norm $\E[\|X_t\|]$ against time for $1000$ trajectories sampled under the Euler ODE integrator, Euler SDE integrator with offsets, and the local linear SDE integrator. A perfect SDE integrator should match the ODE norm distribution. The panels show that the Euler SDE integrator contains an early time blow-up in norm, while our local linear SDE integrator stays close to the ODE reference across models.
    Red arrows mark the
    early-time Euler blow-up in each panel, while blue arrows point to the
    corresponding local linear trajectory at $t=0$, where no such blow-up occurs.
    }
    \label{fig:integrator_norms_row}
\end{figure}
\fi

This appendix describes the training-time integrator used in our experiments,
explains the numerical issues that motivate it, and presents empirical results
that validate its use. Although our experiments focus on adjoint matching, the
same problems arise whenever an RL algorithm samples the memoryless SDE, and
therefore apply to any path-based RL method on flow models, including
non-adjoint-matching approaches such as Flow-GRPO
\citep{liu2025flowgrpotrainingflowmatching}.

Given its superior performance and simplicity, we believe this integrator
should become the standard for all methods that require sampling from the
memoryless SDE for flow models.

\subsection{The Memoryless SDE and the Local Linear Integrator}
\label{app:local_linear_integrator:ll}

Recall from \cref{app:am:memoryless} that solving the KL-regularized RL
problem for flow models requires sampling from the memoryless SDE during
training, 
\begin{equation}
dX_t
=
\left[2\,v_\theta(X_t,t) - \frac{1}{t} X_t\right] dt
+ \sqrt{\frac{2(1-t)}{t}}\, dB_t.
\label{eq:memoryless_sde_eta1_clean}
\end{equation}
Using this memoryless SDE during training is what guarantees that the optimal
terminal marginal is the tilted distribution
$p^*(x) \propto p_{\mathrm{base}}(x)\exp(r(x))$, while still allowing us to
switch back to the standard ODE at test time.

The difficulty is that \cref{eq:memoryless_sde_eta1_clean} is numerically
stiff near $t=0$: both the linear drift coefficient $-1/t$ and the diffusion
coefficient $\sqrt{2(1-t)/t}$ diverge as $t \to 0^+$. A naive explicit
discretization therefore struggles on the first few steps. 

\paragraph{Euler--Maruyama with offsets.}
\citet{domingoenrich2025adjoint} handle this stiffness by replacing the
singular terms $1/t$ and $\sqrt{2(1-t)/t}$ with the regularized versions
$1/(t+\delta)$ and $\sqrt{2(1-t+\varepsilon)/(t+\delta)}$, and then applying
Euler--Maruyama to the modified SDE. With the practical choice
$\delta = \varepsilon = \Delta t$, the update is
\begin{tcolorbox}[colback=gray!10, boxrule=0pt, arc=0pt]
\textbf{Euler--Maruyama with $\delta=\varepsilon=\Delta t$ \citep{domingoenrich2025adjoint}:}
\vspace{-0.5em}
\begin{equation}
x_{k+1}^{\mathrm{Euler}}
=
x_k
+
\left[2\,v_\theta(x_k,t_k) - \frac{x_k}{t_k+\Delta t}\right]\Delta t
+
\sqrt{\frac{2(1-t_k+\Delta t)\Delta t}{t_k+\Delta t}}\,\xi_k,
\qquad
\xi_k \sim \mathcal N(0,I).
\label{eq:euler_offset_forward}
\end{equation}
\end{tcolorbox}

\ifmeta
\begin{wrapfigure}[15]{l}{0.36\textwidth}
    \centering
    \includegraphics[width=\linewidth]{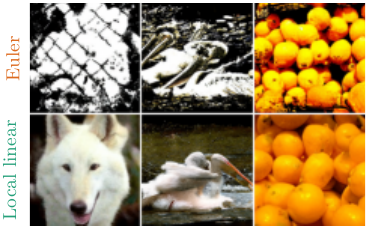}
    \caption{\textbf{Matched-noise JiT samples at $K{=}50$.}
    Top: Euler--Maruyama. Bottom: local linear. Under the same Brownian
    motion, Euler with $\delta{=}\varepsilon{=}\Delta t$ already produces
    visibly degraded samples while local linear with the same offsets
    remains stable.}
    \label{fig:integrator_jit_samples}
\end{wrapfigure}
\fi
\noindent The shift $\delta$ moves the singularity away from $t=0$, and
$\varepsilon$ provides slack near $t=1$ in the adjoint matching loss, which
\citet{domingoenrich2025adjoint} found to improve training speed.\footnote{The precise role of $\varepsilon$ is orthogonal to the
present discussion, but we keep it in the expressions for completeness.}

In practice, we observed that this discretization still produced low-quality
samples for some models. The failure mode matches the intuition above: even with the
offsets, an explicit first-order step cannot keep up with the rapid
contraction of the stiff linear part at early times, so the noise increment
dominates and trajectories blow up in norm, taking the model far from the training distribution.
\Cref{fig:integrator_norms_row} illustrates the effect across the four
ImageNet models used in the paper. It plots the latent norm of $1000$
trajectories under the Euler ODE, under Euler SDE with offsets, and under the
local linear SDE integrator introduced next. A marginal-preserving SDE integrator
should produce the same norm distribution as the ODE, but
the Euler--Maruyama SDE discretization instead exhibits a pronounced early-time blow-up.

\paragraph{The local linear integrator.}
We address the same stiffness with a different splitting. Our key observation is
that if we freeze the velocity $v_\theta(X_t, t)$ to its left-endpoint value
on each interval $[t_k, t_{k+1}]$, the remaining dynamics form a \emph{linear}
SDE with time-dependent coefficients, which can be integrated in closed form
via an integrating factor. Writing $u_k := t_k + \delta$ for the shifted time
and specializing to $\delta = \varepsilon = \Delta t$ as in the
Euler--Maruyama baseline, this yields

\ifmeta
\begin{tcolorbox}[colback=gray!10, boxrule=0pt, arc=0pt]
\textbf{Local linear update with $\delta=\varepsilon=\Delta t$ (ours):}
\vspace{-0.5em}
\begin{equation}
x_{k+1}^{\mathrm{LL}}
=
\Phi_k x_k + \Omega_k\, v_\theta(x_k,t_k) + \sqrt{V_k}\,\xi_k,
\qquad
\xi_k \sim \mathcal N(0,I),
\label{eq:local_linear_forward_summary}
\end{equation}
\vspace{-1.1em}
\[
\Phi_k = \frac{u_k}{u_{k+1}},\qquad
\Omega_k = u_{k+1}(1-\Phi_k^2),\qquad
V_k = (1+2\Delta t)(1-\Phi_k^2) - \frac{2}{3}\bigl(u_{k+1}-\Phi_k^2 u_k\bigr).
\]
\end{tcolorbox}
\fi

\noindent A full derivation, including the general case
$\delta, \varepsilon \geq 0$, is given in
\cref{app:local_linear_integrator:derivation}.

\ifmeta
\begin{wrapfigure}[21]{r}{0.45\textwidth}
    \centering
    \includegraphics[width=\linewidth]{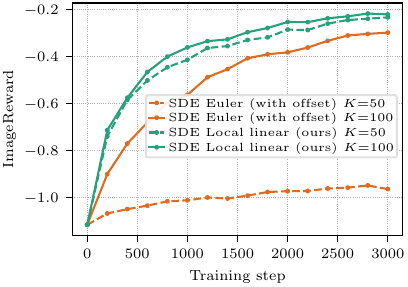}
    \caption{\textbf{JiT training reward during adjoint matching.}
    Local linear with $\delta{=}\varepsilon{=}\Delta t$ at $K{=}50$
    closely tracks the $K{=}100$ run, while Euler with the same offsets and
    step count barely improves the reward. Both local linear runs outperform
    the Euler run at $K{=}100$, showing that the integrator matters for both
    optimization quality and speed.
    }
    \label{fig:integrator_jit_training}
\end{wrapfigure}
\fi

Compared to Euler--Maruyama, the local linear integrator has the same
computational cost, namely one network evaluation per step, but it handles
the stiff linear part of the SDE exactly. The contraction factor
$\Phi_k = u_k / u_{k+1}$ stays in $(0, 1)$ throughout and is smallest
near $t=0$ (where the stiffness is largest), so it damps the early-time
dynamics that would otherwise drive Euler--Maruyama trajectories to blow
up. This is why local linear trajectories stay close to the
ODE reference in \cref{fig:integrator_norms_row}.
Moreover, no term has any problematic singularity.

\paragraph{Empirical improvements.}
Replacing Euler--Maruyama with the local linear integrator improves both
sampling and training. The clearest case is JiT, which was the most challenging
model in our experiments: \cref{fig:integrator_jit_samples} shows
that at $K=50$ steps the local linear integrator already produces coherent
images, while Euler--Maruyama with offsets produces visibly degraded samples
under the same Brownian motion.

This can also be verified quantitatively, and the pattern persists across step
counts and models.
\Cref{fig:integrator_appendix} reports Fr\'echet distances under the
Inception-v3, DINOv2-B, and DINOv3-L feature spaces for all four ImageNet
models. Across model families, local linear with $\delta=\varepsilon=\Delta t$
Pareto-dominates Euler--Maruyama with the same offsets, with
the clearest gains at low step counts. The exception is RAE under
Inception-v3, where both integrators produce high-quality samples and Euler-Maruyama with offsets slightly outperforms local linear.

These improvements carry over to training as well.
\Cref{fig:integrator_jit_training} shows the training reward for JiT on
ImageNet using the setup of the main reward-finetuning experiments. At $K=50$,
local linear matches the reward achieved by its $K=100$ equivalent. 
On the other hand, Euler--Maruyama with offsets at $K=50$ barely improves over the baseline, and even at $K=100$ it still underperforms the local linear runs. This shows that the choice of integrator changes the effective optimization problem that RL sees, and that a better integrator can lead to faster training and better final performance. The 2$\times$ speedup is significant as sampling is the most expensive part of training.

\ifmeta
\begin{figure}[!htb]
    \centering
    \includegraphics[width=0.82\textwidth]{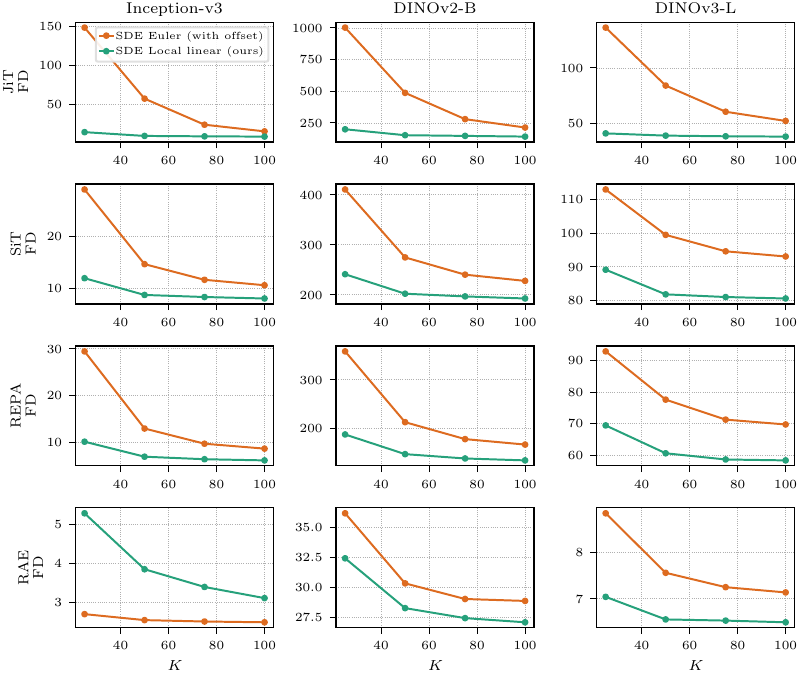}
    \caption{\textbf{Cross-model endpoint FD for Euler vs.\ local linear integrators.}
    Rows correspond to models and columns to the three evaluation feature spaces.
    Local linear with $\delta{=}\varepsilon{=}\Delta t$ consistently matches
    or improves upon Euler with the same offsets, especially at
    low $K$. All metrics were computed with $50,000$ samples.}
    \label{fig:integrator_appendix}
\end{figure}
\fi

\paragraph{Extension to the backward adjoint.}
The same idea applies to the backward adjoint ODE that appears in adjoint
matching. Defining $u(t) = t + \delta$ (so $u_k = u(t_k)$ matches the
shifted time used above) and writing $\Phi_k = u_k / u_{k+1}$ for brevity,
recall that adjoint matching requires integrating
\begin{equation}
\dot a_t
=
-\bigl[\nabla_x \bigl(2\, v_{\mathrm{base}}(X_t, t) - X_t/u(t)\bigr)\bigr]^\top a_t
=
-2\bigl[\nabla_x v_{\mathrm{base}}(X_t, t)\bigr]^\top a_t
+ \frac{a_t}{u(t)},
\label{eq:adjoint_ode_split}
\end{equation}
where the Jacobian of the stiff linear drift $-x/u(t)$ contributes the scalar
$-1/u(t)$ times the identity. (We write \cref{eq:adjoint_ode_split} in terms
of $u(t)$ rather than $t$, as in \cref{eq:fm_lean_adjoint}, given the
practical choice of shifting the time to avoid the singularity at $t=0$.)
\Citet{domingoenrich2025adjoint} use a standard backward Euler step to
discretize \cref{eq:adjoint_ode_split}:
\begin{tcolorbox}[colback=gray!10, boxrule=0pt, arc=0pt]
\textbf{Standard backward Euler adjoint step \citep{domingoenrich2025adjoint}:}
\vspace{-0.5em}
\begin{equation}
\tilde a_k^{\mathrm{Euler}}
=
\tilde a_{k+1}
+ \Delta t\left[
2\bigl(\nabla_x v_{\mathrm{base}}(X_{k+1}, t_{k+1})\bigr)^\top \tilde a_{k+1}
- \frac{\tilde a_{k+1}}{u_{k+1}}
\right].
\label{eq:euler_adjoint_backward}
\end{equation}
\end{tcolorbox}
\noindent Using a similar splitting idea, which we derive in
\cref{app:local_linear_integrator:derivation} below, one can instead
integrate the stiff scalar term exactly and replace this step by:
\begin{tcolorbox}[colback=gray!10, boxrule=0pt, arc=0pt]
\textbf{Local linear backward adjoint step (ours):}
\vspace{-0.5em}
\begin{equation}
\tilde a_k^{\mathrm{LL}}
=
\Phi_k\Bigl(
\tilde a_{k+1}
+ 2\Delta t\,
\bigl[\nabla_x v_{\mathrm{base}}(X_{k+1}, t_{k+1})\bigr]^\top \tilde a_{k+1}
\Bigr),
\qquad
\Phi_k = \frac{u_k}{u_{k+1}}.
\label{eq:local_linear_adjoint}
\end{equation}
\end{tcolorbox}
\noindent In practice we did not see gains as large as for the forward
integrator, but this step performed at least as well as the Euler
discretization, so we kept it.

\subsection{Derivation of the Local Linear Integrator}
\label{app:local_linear_integrator:derivation}

Both derivations rest on the following simple lemma (for further background in the context of ODEs see~\citet{arnold1978ode}).

\begin{lemma}
\label{lem:linear_ode_sde}
Let $f, g, h : [s, T] \to \R$ be continuous and deterministic. The
unique solution of the SDE
\[
dY_t = f(t)\, Y_t\, dt + g(t)\, dt + h(t)\, dB_t,
\qquad
Y_s = y_s,
\]
is
\[
Y_T
\;=\;
e^{\int_s^T f(\zeta)\, d\zeta}\, y_s
\;+\;
\int_s^T e^{\int_\tau^T f(\zeta)\, d\zeta}\, g(\tau)\, d\tau
\;+\;
\int_s^T e^{\int_\tau^T f(\zeta)\, d\zeta}\, h(\tau)\, dB_\tau.
\]
The stochastic integral is a centered Gaussian with variance
$\int_s^T e^{2\int_\tau^T f(\zeta)\, d\zeta}\, h(\tau)^2\, d\tau$. The ODE
case $h\equiv 0$ is recovered as a special case.
\end{lemma}

\begin{proof}
Let
$\mu(t) = \exp\!\bigl(-\!\int_s^t f(\zeta)\, d\zeta\bigr)$, so that
$\mu(s) = 1$ and $\dot\mu(t) = -f(t)\, \mu(t)$. Since $\mu$ is deterministic
and of bounded variation, Itô's product rule applied to $\mu(t)\, Y_t$ gives
\[
d[\mu(t) Y_t]
=
\dot\mu(t)\, Y_t\, dt + \mu(t)\, dY_t
=
-f(t)\, \mu(t)\, Y_t\, dt
+ \mu(t)\, \bigl[f(t) Y_t\, dt + g(t)\, dt + h(t)\, dB_t\bigr].
\]
The two $f(t)\, Y_t$ terms cancel, leaving
\[
d[\mu(t)\, Y_t] \;=\; \mu(t)\, g(t)\, dt \;+\; \mu(t)\, h(t)\, dB_t.
\]
Integrating from $s$ to $T$ and using $\mu(s)\, Y_s = y_s$,
\[
\mu(T)\, Y_T
\;=\;
y_s
\;+\; \int_s^T \mu(\tau)\, g(\tau)\, d\tau
\;+\; \int_s^T \mu(\tau)\, h(\tau)\, dB_\tau.
\]
Dividing through by $\mu(T)$ gives
\[
Y_T
\;=\;
\frac{1}{\mu(T)}\, y_s
\;+\; \int_s^T \frac{\mu(\tau)}{\mu(T)}\, g(\tau)\, d\tau
\;+\; \int_s^T \frac{\mu(\tau)}{\mu(T)}\, h(\tau)\, dB_\tau,
\]
and the identities
$1/\mu(T) = \exp\!\bigl(\int_s^T f(\zeta)\, d\zeta\bigr)$ and
$\mu(\tau)/\mu(T) = \exp\!\bigl(\int_\tau^T f(\zeta)\, d\zeta\bigr)$ yield the
claimed formula. Gaussianity and the variance formula for the stochastic
integral follow from Itô's isometry applied to the deterministic integrand
$\mu(\tau)\, h(\tau) / \mu(T)$.
\end{proof}

Throughout this section we keep the offsets $\delta, \varepsilon \geq 0$ and
continue to use
\[
u(t) = t + \delta,
\qquad
u_k = u(t_k),
\qquad
\Phi_k = \frac{u_k}{u_{k+1}}.
\]

\subsubsection*{Forward SDE}
The forward SDE
\begin{equation}
dX_t
=
\left[2\, v_\theta(X_t, t) - \frac{X_t}{u(t)}\right] dt
+ \sqrt{\frac{2(1 - t + \varepsilon)}{u(t)}}\, dB_t
\label{eq:memoryless_sde_shifted}
\end{equation}
matches \cref{lem:linear_ode_sde} with $f(t) = -1/u(t)$, smooth forcing
$g(t) = 2\, v_\theta(X_t, t)$, and noise coefficient
$h(t) = \sqrt{2(1-t+\varepsilon)/u(t)}$. The exponential factor in the lemma
simplifies because $-1/u$ has a logarithmic antiderivative:
\[
e^{\int_s^t f(\zeta)\, d\zeta}
\;=\;
e^{-[\log u(t) - \log u(s)]}
\;=\;
\frac{u(s)}{u(t)}.
\]
Applying \cref{lem:linear_ode_sde} over a single step $[t_k, t_{k+1}]$ gives
the \emph{exact} representation
\begin{equation}
X_{t_{k+1}}
=
\Phi_k\, X_{t_k}
+ \frac{2}{u_{k+1}}\!\int_{t_k}^{t_{k+1}}\! u(\tau)\, v_\theta(X_\tau, \tau)\, d\tau
+ \frac{1}{u_{k+1}}\!\int_{t_k}^{t_{k+1}}\! u(\tau) \sqrt{\frac{2(1-\tau+\varepsilon)}{u(\tau)}}\, dB_\tau.
\label{eq:memoryless_sde_mild}
\end{equation}
We make a single approximation---freezing the velocity at its left-endpoint
value, $v_\theta(X_\tau, \tau) \approx v_\theta(x_k, t_k) =: v_k$---and
integrate the rest exactly. The stochastic-integral clause of
\cref{lem:linear_ode_sde} turns the noise integral into a centered Gaussian.

The velocity term collapses to $\Omega_k\, v_k$ with
\begin{equation}
\Omega_k
\;=\;
\frac{2}{u_{k+1}}\!\int_{t_k}^{t_{k+1}}\! u(\tau)\, d\tau
\;=\;
\frac{u_{k+1}^2 - u_k^2}{u_{k+1}}
\;=\;
u_{k+1}\bigl(1 - \Phi_k^2\bigr).
\label{eq:local_linear_omega}
\end{equation}
The noise term is $\sqrt{V_k}\, \xi_k$ with $\xi_k \sim \mathcal{N}(0, I)$ and
\[
V_k
\;=\;
\frac{2}{u_{k+1}^2}\!\int_{t_k}^{t_{k+1}}\! u(\tau)\,(1 - \tau + \varepsilon)\, d\tau.
\]
Using the algebraic identity
$1 - \tau + \varepsilon = (1 + \delta + \varepsilon) - u(\tau)$, the
integrand splits into elementary terms:
\[
\int_{t_k}^{t_{k+1}}\! u(\tau)\, d\tau = \frac{u_{k+1}^2 - u_k^2}{2},
\qquad
\int_{t_k}^{t_{k+1}}\! u(\tau)^2\, d\tau = \frac{u_{k+1}^3 - u_k^3}{3}.
\]
Substituting and simplifying with $u_k^3 / u_{k+1}^2 = \Phi_k^2\, u_k$,
\begin{equation}
V_k
\;=\;
(1 + \delta + \varepsilon)\bigl(1 - \Phi_k^2\bigr)
\;-\;
\frac{2}{3}\bigl(u_{k+1} - \Phi_k^2\, u_k\bigr).
\label{eq:local_linear_variance}
\end{equation}
Setting $\delta = \varepsilon = \Delta t$ recovers the boxed forward
update \cref{eq:local_linear_forward_summary}.

\subsubsection*{Backward adjoint ODE}
We now derive the boxed local linear backward adjoint step
\cref{eq:local_linear_adjoint}. Recall the split form of the adjoint ODE
\cref{eq:adjoint_ode_split},
\[
\dot a_t
\;=\;
-2\bigl[\nabla_x v_{\mathrm{base}}(X_t, t)\bigr]^\top a_t
\;+\; \frac{a_t}{u(t)}.
\]
To handle the stiff scalar term exactly, apply \cref{lem:linear_ode_sde} to
this ODE---now with $f(t) = 1/u(t)$ and smooth forcing
$g(t) = -2[\nabla_x v_{\mathrm{base}}(X_t, t)]^\top a_t$. The exponential
factor is again a ratio of $u$'s, with the opposite sign from the forward
case:
\[
e^{\int_s^t f(\zeta)\, d\zeta}
\;=\;
\frac{u(t)}{u(s)}.
\]
Applying the lemma over $[t_k, t_{k+1}]$ to express $a_{t_{k+1}}$ in terms
of $a_{t_k}$ and rearranging for the backward direction gives the exact
identity
\[
\tilde a_k
\;=\;
\Phi_k\, \tilde a_{k+1}
\;+\;
2 \int_{t_k}^{t_{k+1}} \frac{u_k}{u(\tau)}\,
   \bigl[\nabla_x v_{\mathrm{base}}(X_\tau, \tau)\bigr]^\top a_\tau\, d\tau.
\]

The integrand depends on the unknown $a_\tau$, so we must approximate.
Following the same spirit as in the forward case, we freeze the integrand
at a single point and use the resulting constant times $\Delta t$. We freeze
at $\tau = t_{k+1}$, the only point where $a_\tau$ is known: this replaces
$a_\tau$ by $\tilde a_{k+1}$, the Jacobian by
$\nabla_x v_{\mathrm{base}}(X_{k+1}, t_{k+1})$, and the kernel
$u_k / u(\tau)$ by $u_k / u_{k+1} = \Phi_k$. The integral collapses to
$2\Phi_k\, \Delta t$ times the frozen value, yielding
\[
\tilde a_k
\;=\;
\Phi_k\Bigl(
\tilde a_{k+1}
+ 2\Delta t\,
\bigl[\nabla_x v_{\mathrm{base}}(X_{k+1}, t_{k+1})\bigr]^\top \tilde a_{k+1}
\Bigr),
\]
which is \cref{eq:local_linear_adjoint}.

\FloatBarrier

\FloatBarrier
\clearpage
\section{Additional Implementation Details}
\label{app:additional_implementation_details}

This appendix gives the method-level implementation recipe used in our DRL
experiments. The end-to-end procedure is summarized in
\cref{alg:drl_implementation}, which combines the buffered training pattern
used in practice for both stages with the local linear integrator for the
memoryless SDE and its companion adjoint
(\cref{app:local_linear_integrator}). We provide further details below. 

Experiment-level configuration (base
models, default hyperparameters, sampling, evaluation metrics, baseline and
ablation setups) is collected in \cref{app:ablation_setups}.
The shifted-time offsets in Stage~2 are written as $\delta, \varepsilon \geq
0$ throughout the algorithm; in practice we use the choice
$\delta = \varepsilon = \Delta t$ from
\cref{app:local_linear_integrator}, but stating the algorithm in terms of
$\delta, \varepsilon$ keeps the dependence on the offsets explicit.

\begin{algorithm}[!tbp]
\caption{Discriminator-Guided RL with Buffered Training and the Local Linear Integrator}
\label{alg:drl_implementation}
\footnotesize
\vspace{0.1em}
\centerline{\rule{0.04\linewidth}{0.3pt}\hspace{0.6em}%
  \textbf{Stage 1: Buffered Feature-Space Discriminator Training}%
  \hspace{0.6em}\rule{0.04\linewidth}{0.3pt}}
\begin{algorithmic}[1]
\REQUIRE Target distribution $q$, base model $p_{\mathrm{base}}$, frozen
  encoder $\phi$, buffer size $N_D$, epochs per refresh $E_D$,
  $R_1$ weight $\gamma$, learning rate $\eta_D$
\STATE Initialize discriminator parameters $\psi$
\REPEAT
    \STATE \textbf{Refresh buffer $\mathcal{B}_D$:}
    \FOR{$i = 1, \ldots, N_D$}
        \STATE Sample $x_{\mathrm{real}}^{(i)} \sim q$ and
               $x_{\mathrm{fake}}^{(i)} \sim p_{\mathrm{base}}$
        \STATE Append $\bigl(x_{\mathrm{real}}^{(i)},\,
               x_{\mathrm{fake}}^{(i)}\bigr)$ to $\mathcal{B}_D$
    \ENDFOR
    \STATE \vspace{0.3em}
    \STATE \textbf{Train discriminator on $\mathcal{B}_D$:}
    \FOR{epoch $n = 1, \ldots, E_D$}
        \FOR{minibatch $(x_{\mathrm{real}}, x_{\mathrm{fake}}) \subset \mathcal{B}_D$}
            \STATE $\ell_{\mathrm{disc}}(\psi) \leftarrow
                   -\log D_\psi(\phi(x_{\mathrm{real}}))
                   -\log\bigl(1-D_\psi(\phi(x_{\mathrm{fake}}))\bigr)
                   + \tfrac{\gamma}{2}\,\bigl\|\nabla_x D_\psi(\phi(x_{\mathrm{real}}))\bigr\|^2$
            \STATE $\psi \leftarrow \psi - \eta_D\, \nabla_\psi \ell_{\mathrm{disc}}(\psi)$
        \ENDFOR
    \ENDFOR
\UNTIL{convergence}
\STATE Define reward $\hat r(x) \leftarrow \mathrm{logit}\,D_\psi(\phi(x))$
\end{algorithmic}
\vspace{0.2em}
\centerline{\rule{0.02\linewidth}{0.3pt}\hspace{0.6em}%
  \textbf{Stage 2: Buffered Adjoint Matching with the Local Linear Integrator}%
  \hspace{0.6em}\rule{0.02\linewidth}{0.3pt}}
\begin{algorithmic}[1]
\REQUIRE Reward $\hat r$, base velocity $v_{\mathrm{base}}$, KL weight $\lambda$,
  uniform time grid $0=t_0<\cdots<t_K=1$ with $\Delta t = 1/K$,
  shifted-time offsets $\delta,\varepsilon \geq 0$ with
  $u_k := t_k+\delta$, trajectory buffer size $N_T$,
  epochs per refresh $E_T$, learning rate $\eta_T$
\STATE Initialize $v_\theta \leftarrow v_{\mathrm{base}}$
\REPEAT
    \STATE \textbf{Refresh trajectory buffer $\mathcal{B}_T$:}
    \FOR{$i = 1, \ldots, N_T$}
        \STATE Sample $X_0^{(i)} \sim \mathcal{N}(0, I)$
        \STATE \emph{Local linear forward pass} (with current $v_\theta$):
        \FOR{$k = 0, \ldots, K-1$}
            \STATE $\Phi_k \leftarrow u_k / u_{k+1}$,\quad
                   $\Omega_k \leftarrow u_{k+1}\bigl(1-\Phi_k^2\bigr)$,\quad
                   $V_k \leftarrow (1+\delta+\varepsilon)\bigl(1-\Phi_k^2\bigr)
                                   - \tfrac{2}{3}\bigl(u_{k+1}-\Phi_k^2 u_k\bigr)$
            \STATE $\xi_k^{(i)} \sim \mathcal{N}(0, I)$
            \STATE $X_{k+1}^{(i)} \leftarrow
                   \Phi_k\, X_k^{(i)}
                   + \Omega_k\, v_\theta(X_k^{(i)}, t_k)
                   + \sqrt{V_k}\, \xi_k^{(i)}$
        \ENDFOR
        \STATE Terminal adjoint
               $\tilde a_K^{(i)} \leftarrow -\lambda\, \nabla_x \hat r\bigl(X_K^{(i)}\bigr)$
        \STATE \emph{Local linear backward adjoint pass} (with $v_{\mathrm{base}}$):
        \FOR{$k = K-1, \ldots, 0$}
            \STATE $\tilde a_k^{(i)} \leftarrow
                   \Phi_k\Bigl(
                   \tilde a_{k+1}^{(i)}
                   + 2\Delta t\,
                   \bigl[\nabla_x v_{\mathrm{base}}\bigl(X_{k+1}^{(i)}, t_{k+1}\bigr)\bigr]^\top
                   \tilde a_{k+1}^{(i)}
                   \Bigr)$
        \ENDFOR
        \STATE Store
               $\bigl(\{X_k^{(i)}\}_{k=0}^{K},
                      \{\tilde a_k^{(i)}\}_{k=0}^{K}\bigr)$
               in $\mathcal{B}_T$ as stop-gradient quantities
    \ENDFOR
    \STATE \vspace{0.3em}
    \STATE \textbf{Train velocity on $\mathcal{B}_T$:}
    \FOR{epoch $n = 1, \ldots, E_T$}
        \FOR{minibatch $\bigl(\{X_k\}, \{\tilde a_k\}\bigr) \subset \mathcal{B}_T$}
            \STATE Set $\sigma_k \leftarrow \sqrt{2(1-t_k+\varepsilon)/u_k}$
            \STATE $L_{\mathrm{AM}}(\theta) \leftarrow
                   \tfrac{1}{2K}\sum_{k=0}^{K-1}
                   \Bigl\|
                   \tfrac{2}{\sigma_k}\bigl(v_\theta(X_k, t_k) - v_{\mathrm{base}}(X_k, t_k)\bigr)
                   + \sigma_k\, \tilde a_k
                   \Bigr\|^2$
            \STATE $\theta \leftarrow \theta - \eta_T\, \nabla_\theta L_{\mathrm{AM}}(\theta)$
        \ENDFOR
    \ENDFOR
\UNTIL{convergence}
\STATE \textbf{return} $v_\theta$
\end{algorithmic}
\end{algorithm}

\paragraph{Buffered discriminator training.}
To amortize the cost of generating model samples, we do not draw a fresh
mini-batch of fakes for every gradient step. Instead, we periodically refresh
a buffer $\mathcal{B}_D$ of size $N_D$ holding paired real samples
$x_{\mathrm{real}}$ and base-model samples $x_{\mathrm{fake}}$, and then take
$E_D$ epochs of gradient steps over that buffer with the standard logistic
loss plus an $R_1$ gradient penalty. The $R_1$ penalty is computed by
backpropagating the discriminator output through the frozen encoder $\phi$
to the image space. In principle, the flow-space could be used as well but we found this to work well. Given that it is also more memory efficient as we avoid backpropagating through the flow space, we decided to use it instead of the flow-space penalty. 

\paragraph{Buffered adjoint matching.}
We use the same buffering pattern for the RL stage. Periodically, we sample
$N_T$ trajectories from the memoryless training SDE under the current
$v_\theta$ using the local linear forward integrator
\eqref{eq:local_linear_forward_summary}, compute their adjoints by solving
the local linear backward adjoint \eqref{eq:local_linear_adjoint} starting
from $\tilde a_K = -\lambda\,\nabla_x \hat r(X_K)$, and store both
quantities in a trajectory buffer $\mathcal{B}_T$ as stop-gradient tensors.
We then take $E_T$ epochs of gradient steps over $\mathcal{B}_T$, with
gradients flowing only through the current velocity evaluations
$v_\theta(X_k, t_k)$. This is technically off-policy (the buffer drifts as
$v_\theta$ updates), but in practice we observed no degradation versus a
fully on-policy implementation, and Girsanov-based importance weighting
gave no measurable improvement either.

\paragraph{Training versus inference.}
The memoryless SDE used during post-training is integrated with the local
linear scheme described in \cref{app:local_linear_integrator}. After
fine-tuning, we sample from the updated model with the standard
deterministic ODE; differences between the two samplers were small but
non-negligible.

\paragraph{Other implementation details.} A few smaller choices, all
following or adapted from the original adjoint-matching implementation:
\begin{enumerate}[leftmargin=2em,itemsep=0.25em,topsep=0.25em]
\item \textbf{Final denoising step.} The forward loop in
\cref{alg:drl_implementation} runs $K-1$ local linear SDE steps and
replaces the final step from $t_{K-1}$ to $t_K$ with a single deterministic
denoising call to the network.

\item \textbf{Stratified timestep subsampling for the loss.} Rather than
summing the AM loss over all $K$ timesteps, each gradient step uses a
fresh subset $S \subset \{0, \ldots, K-1\}$ of size $\lceil f K
\rceil$, with $f = 0.4$ by default. The subset is split $50/50$
between early ($t < 0.6$) and late ($t \geq 0.6$) timesteps so that both
ends of the trajectory are always represented. The loss line in
Stage~2 then becomes $L_{\mathrm{AM}} = (1/(2|S|)) \sum_{k \in S}
\|\cdot\|^2$.

\item \textbf{Adaptive outlier clipping for the AM loss.} The per-sample
loss $\ell_k^{(i)} = \bigl\|\tfrac{2}{\sigma_k}(v_\theta - v_{\mathrm{base}})
+ \sigma_k\,\tilde a_k\bigr\|^2$ has a heavy right tail, dominated by a few
trajectories whose adjoints blow up and destabilize training. We maintain
an EMA-smoothed threshold $\tau_n$ across optimizer steps $n$ on the
per-sample norms $\sqrt{\ell_k^{(i)}}$,
\[
\tau_n = \rho\, \tau_{n-1}
       + (1-\rho)\, \min\!\bigl(q_\alpha\bigl(\sqrt{\ell_k^{(i)}}\bigr),\,
                                c\, \tau_{n-1}\bigr),
\]
where $q_\alpha$ is the empirical $\alpha$-quantile over the current batch,
$\rho$ is the EMA decay, and the spike cap $c$ prevents a single outlier
batch from rapidly inflating the threshold. Samples with $\sqrt{\ell_k^{(i)}}
\geq \tau_n$ (or non-finite) are masked out of the gradient. We use
$\alpha = 0.9$, $\rho = 0.9$ for SiT-XL/2, REPA SiT-XL/2, and JiT-H/16, and
$\alpha = 0.75$, $\rho = 0.995$ for RAE DiTDH-XL, with spike cap $c = 5$
throughout. The spike cap $c$ was particularly important for RAE, whose
loss exhibited noticeably more frequent spikes than the other models.
\end{enumerate}

\FloatBarrier
\clearpage
\section{Experimental Setup}
\label{app:experimental_setup}
\label{app:ablation_setups}

We collect all experiment-level configuration here: shared setup (base
models, sampling, evaluation metrics) followed by per-experiment subsections
covering our main distributional alignment experiments
(\cref{sec:experiments_main}), the baselines we compare against, and the
ablations discussed in \parheadref{sec:exp_ablations}. The general method recipe
is given separately in \cref{app:additional_implementation_details}.

\paragraph{Base models.}
We evaluate four pretrained ImageNet~\citep{deng2009imagenet} $256\times 256$, class-conditional generators spanning latent and pixel-space flow/diffusion architectures:
\begin{itemize}
    \item \textbf{SiT-XL/2}~\citep{ma2024sit} (0.72B params total: 675M flow model, 50M decoder): A generative model built on the DiT backbone~\citep{peebles2023dit} using the stochastic interpolant framework~\citep{albergo2023stochastic}. We use the XL/2 variant, operating in the standard SD VAE latent space ($4 \times 32 \times 32$).\looseness-1
    \item \textbf{JiT-H/16}~\citep{li2025jit} (0.95B params total: 953M flow model, 0 decoder): A plain Vision Transformer~\citep{dosovitskiy2021vit} that performs flow matching directly in pixel space on raw image patches, without a VAE tokenizer. We use the H/16 variant operating directly on $3 \times 256 \times 256$ pixel inputs.\looseness-1
    \item \textbf{REPA SiT-XL/2}~\citep{yu2025repa} (0.72B params total: 675M flow model, 50M decoder): A flow-based diffusion transformer regularized by aligning intermediate noisy representations with clean semantic features from a frozen encoder. We use the variant built on the SiT-XL/2 backbone (same SD VAE latent space, $4 \times 32 \times 32$), trained with REPA alignment against DINOv2-ViT-B features.\looseness-1
    \item \textbf{RAE DiTDH-XL}~\citep{zheng2026rae} (1.25B params total: 839M flow model, 415M ViT-XL decoder): A diffusion transformer that replaces the VAE latent space with a Representation Autoencoder---a frozen vision encoder paired with a lightweight decoder---and trains via flow matching in this semantic latent space. We use the DiTDH-XL stage-2 model with a frozen DINOv2-ViT-B encoder and ViT-XL decoder, operating in a $768 \times 16 \times 16$ semantic latent space.\looseness-1
\end{itemize}

\paragraph{Sampling.}
For all models, unless otherwise indicated, we use a Heun 50-step sampler with a linear schedule, except for RAE where we follow the original paper and use their suggested scheduler~\citep{zheng2026rae}. This corresponds to 99 NFEs since the final step reduces to a regular Euler step.

\paragraph{Evaluation metrics and feature spaces.}
For our main results, we report metrics in the DINOv2-Large, DINOv3-Large, SigLIP-Large, and Inception-v3 feature spaces. The feature-space ablation also reports the corresponding base-size variants (DINOv2-B, DINOv3-B, SigLIP-B); patterns are largely similar to their large-size counterparts. All distributional metrics are computed on $50\,000$ generated samples with class labels balanced across the $1000$ ImageNet classes, against an equally sized reference set so that the generated and reference sets are matched in size. As is commonly done in the literature, FD is reported against the $(\mu, \Sigma)$ feature statistics of the full ImageNet training set; $\text{FD}_{\rm{val}}$, KD, Precision, Recall, Density, and Coverage are reported against $50\,000$ ImageNet validation images. Precision/Recall are computed with $k{=}3$ nearest neighbors and Density/Coverage with $k{=}5$, matching the defaults of \citet{kynkaanniemi2019improved} and \citet{naeem2020reliable} respectively. All manifold metrics use Euclidean distance directly on the embedding vectors.\looseness-1

\paragraph{Best-CFG configuration.}
\label{app:cfg_tuning}
Whenever we report ``best CFG'' numbers, the configuration is selected per checkpoint by sweeping the CFG scale from $1$ to $3$ in steps of $0.25$ together with the application interval $\{(0.0, 1.0), (0.1, 1.0), (0.3, 1.0)\}$ (the time range $t \in [\text{low}, \text{high}]$ over which CFG is applied; outside it the guidance scale falls back to $1.0$)~\citep{kynkaanniemi2024applying}. For the autoguidance experiments (\cref{app:autoguidance}) we also sweep scales from $0$ to $1$ in the same step.

\subsection{Distributional alignment setup}
\label{app:setup_distributional_alignment}

DRL is run as two stages---a discriminator on frozen-encoder
features, followed by KL-regularized RL via adjoint matching---both using
buffered training (\cref{alg:drl_implementation}). The same checkpoints back
the \parheadref{sec:exp_alignment} results and the
\parheadref{sec:exp_reward_transfer} results.

\paragraph{Discriminator architecture.}
Unless stated otherwise, the discriminator operates on frozen DINOv2-Large features (CLS token; the encoder is never updated during DRL). The head is a class-conditional projection discriminator~\citep{miyato2018cgansprojectiondiscriminator} of the form $D(x, y) = \varphi(h(x)) + \langle h(x), e(y) \rangle + b(y)$, where $h$ is a linear map (no hidden layers) with hidden dim 512 and $e(y)$ is a learned class embedding. The reward used by the RL stage is the discriminator logit, $\hat{r}(x) = \mathrm{logit}\, D(\phi(x))$, as constructed in \cref{sec:implications_fixing_sft}.

\paragraph{Discriminator training.}
We optimize the standard logistic GAN loss with the R1 gradient penalty above. Each refresh of the discriminator buffer holds 20 mini-batches of paired (real, fake) images (in the ambient space of the model), and we take 10 epochs of gradient steps over that buffer before refreshing. We optimize with Adam at learning rate $10^{-4}$ at an effective batch size of $512$, for up to $10\,000$ optimizer steps, evaluating on a held-out validation set every 50 steps and early-stopping with patience 15.

\paragraph{RL training (adjoint matching).}
We optimize the RL objective with adjoint matching~\citep{domingoenrich2025adjoint}, sampling trajectories from the memoryless training SDE using the local linear integrator (\cref{app:local_linear_integrator}) at $K = 100$ uniform steps with $\eta = 1$ and offsets $\delta = \varepsilon = \Delta t$. CFG is disabled during training. We use Adam at learning rate $10^{-5}$ (no weight decay) with a $200$-step linear warmup, run for $3\,000$ optimizer steps, and sum the loss across timesteps. In addition to the per-trajectory adaptive clipping and stratified timestep subsampling described in \cref{app:additional_implementation_details}, trajectories and their adjoints are pre-computed and stored stop-gradient in a buffer; we then take 5 epochs of velocity gradient steps over that buffer per refresh. We use an effective batch size of $576$ for all models, with a buffer of $2304$ trajectories for SiT-XL/2 and REPA SiT-XL/2 and a buffer of $3072$ trajectories for JiT-H/16 and RAE DiTDH-XL. We observed that RAE was noticeably more unstable to train than the other models, so for RAE we increased the quantile clipping ($q = 0.75$ with EMA decay $0.995$, vs $q = 0.9$ and EMA decay $0.9$ elsewhere) and reduced the loss to its mean rather than its sum across timesteps to compensate for the larger latent space ($768 \times 16 \times 16$ vs $4 \times 32 \times 32$). We also use a slightly higher learning rate of $5 \times 10^{-5}$ (vs $10^{-5}$ elsewhere), which we found to work better in practice for RAE, together with gradient clipping of $1.0$ and a spike cap of $5$ on the clipped loss-threshold EMA.

\paragraph{Compute.}
Each base model requires a one-off discriminator training run plus one adjoint-matching run per $\lambda$. All numbers below are measured on a cluster of NVIDIA H200 GPUs (8 per node). The discriminator stage takes roughly $3$--$6$ hours of wall-clock time on a single 8$\times$H200 node, i.e. approximately $\sim$24--48 GPU-hours per base model. The per-$\lambda$ adjoint-matching cost is approximately $\sim$190 GPU-hours for SiT-XL/2 and REPA SiT-XL/2 ($\sim$110 in optimizer steps and $\sim$80 in buffer-fill sampling), $\sim$500 GPU-hours for JiT-H/16 ($\sim$290 in optimizer steps, $\sim$210 in buffer-fill sampling, with the higher cost reflecting pixel-space ViT-H/16 forward passes), and $\sim$420 GPU-hours for RAE DiTDH-XL ($\sim$190 in optimizer steps, $\sim$230 in buffer-fill sampling, with the buffer-fill share being largest because sampling in the $768 \times 16 \times 16$ semantic latent is more expensive than in the SD VAE latent).

For comparison, the base models we post-train are themselves the result of substantial pretraining budgets: SiT-XL/2 is reported to be trained for approximately $18{,}300$ TPU v4 chip-hours~\citep{ma2024sit}, and REPA SiT-XL/2 for $1{,}646$ H100 GPU-hours~\citep{yu2025repa} (to our knowledge no GPU-hour figures have been reported for JiT or RAE). \drl post-training is therefore a small fraction of the cost of producing the base model in either case. We also note that we did not spend significant effort tuning \drl for training-time efficiency: we believe it is straightforward to roughly halve the per-$\lambda$ cost by trading off the configuration knobs we held fixed---fewer discretization steps, larger trajectory buffers, more epochs per buffer refresh, and so on---without changing the qualitative results.

\subsection{Better image quality (reward-transfer evaluation) setup}
\label{app:setup_reward_transfer}
The \parheadref{sec:exp_reward_transfer} results reuse the distributional alignment checkpoints (\cref{app:setup_distributional_alignment}) without any additional training, and only the evaluation protocol is new. We evaluate four held-out preference reward models---ImageReward~\citep{xu2023imagereward}, PickScore~\citep{kirstain2023pickapic}, Aesthetics v2.5~\citep{discus0434_aesthetic_predictor_v2_5_2024}, and HPSv2~\citep{wu2023human}---over $50\,000$ generated images per checkpoint (balanced across the $1\,000$ ImageNet classes). For the text-conditioned rewards, the ImageNet class label is plugged into the prompt template ``a photo of a \{class\}''.

\subsection{Preference-based RL setup}
\label{app:setup_preference_rl}

For the \parheadref{sec:exp_reward} results, we post-train each base model with KL-regularized RL using ImageReward~\citep{xu2023imagereward} as the reward (computed on the rendered image with the same ``a photo of a \{class\}'' prompt template as in \cref{app:setup_reward_transfer}), and compare two starting points: (i) the base model (\textbf{Base+PRL}) and (ii) our DRL $\lambda{=}10$ checkpoint (\textbf{DRL+PRL}). For each starting point we sweep the PRL reward--KL trade-off $\lambda_{\mathrm{PRL}} \in \{1, 10, 40\}$. For DRL+PRL the DRL checkpoint serves as both the initial weights and the KL reference, so the KL is anchored to the distributionally aligned model rather than the base.

\paragraph{RL training (adjoint matching).}
PRL uses the same adjoint matching setup as the distributional alignment stage (\cref{app:setup_distributional_alignment}---same batch size, buffer sizes, and RAE-specific overrides), trained for $5\,000$ optimizer steps instead of $3\,000$. We chose this larger budget after observing that the Base+PRL runs at $3\,000$ steps were close to but not yet fully converged.

\paragraph{Compute.}
Same per-$\lambda$ figures as \cref{app:setup_distributional_alignment}, scaled proportionally to the longer $5\,000$-step budget (a factor of $5/3$).

\paragraph{Evaluation for the Pareto / over-optimization analysis.}
Beyond the same distributional and reward metrics from \cref{app:setup_distributional_alignment}, the analysis behind \cref{fig:pareto_hpsv2_cfg1} (and the appendix Pareto figures) tracks five low-level image statistics that reward-hacking can push to extreme values: mean brightness, saturation, contrast, colorfulness, and whiteness. Each statistic is computed per image on RGB pixels in $[0, 1]$ and then averaged across all images:
\begin{itemize}[leftmargin=2em,itemsep=0pt,topsep=0.25em]
    \item Brightness: mean grayscale luminance $Y = 0.299\,R + 0.587\,G + 0.114\,B$ over all pixels (BT.601 weights).
    \item Saturation: mean of $(\max_c - \min_c)/\max_c$ across pixels (channel-wise max/min over RGB; clamped to avoid division by zero).
    \item Contrast: pixel-wise standard deviation of the grayscale luminance $Y$.
    \item Colorfulness: the Hasler--S\"usstrunk metric $\sigma + 0.3\,\mu$, with $\sigma = \sqrt{\sigma_{rg}^2 + \sigma_{yb}^2}$ and $\mu = \sqrt{\mu_{rg}^2 + \mu_{yb}^2}$ on the opponent channels $rg = R - G$ and $yb = \tfrac{1}{2}(R + G) - B$.
    \item Whiteness: mean of $\min_c(R, G, B)$ across pixels (high when pixels are close to white).
\end{itemize}
Statistics are computed on $3\,000$ generated images per checkpoint (three images per ImageNet class, sampled with fixed noise seeds shared across checkpoints so trajectories are directly comparable). The ImageNet reference values are computed on $3\,000$ ImageNet validation images (three per class, resized and $256 \times 256$ center-cropped).

\subsection{Distillation from RL teachers setup}
\label{app:setup_distillation_details}
\label{app:distillation_details}
The distillation-gap experiment (\cref{sec:limitations}) on Stable Diffusion~1.5 has two stages: (i) train an RL teacher with adjoint matching, and (ii) distill its samples into a fresh student with standard flow / score matching.

\paragraph{RL teacher (adjoint matching).}
We use the official adjoint-matching codebase released by~\citet{domingoenrich2025adjoint} (\url{https://github.com/microsoft/soc-fine-tuning-sd}). Starting from Stable Diffusion~1.5, we run buffered adjoint matching with the codebase's \texttt{multi\_prompt\_buffer.yaml} settings, and only deviate as noted below. Training is on an $\sim$8\,000-prompt subset of the $10\,000$ prompts shipped in the codebase, themselves extracted from the ReFL training data of~\citet{xu2023imagereward}, with ImageReward as the reward and reward multiplier $300$ (vs $100$ in the default config). Sampling uses 50-step memoryless DDIM ($\eta = 1$). Optimization uses Adam at learning rate $3 \times 10^{-6}$, $\beta = (0.9, 0.95)$, fp32; per-rank batch $6$ with $11$ accumulation steps on $8$ H200 GPUs (effective batch $528$). Buffered training uses $\texttt{buffer\_size}{=}100$ trajectories with $\texttt{passes\_per\_buffer}{=}10$. We train for $10$ epochs over the prompt set and keep the checkpoint with the best validation ImageReward as the teacher. These are essentially the default settings provided in \texttt{multi\_prompt\_buffer.yaml}.

\paragraph{Distillation (SFT student).}
The student is a fresh Stable Diffusion~1.5 UNet trained with the standard DDPM $\varepsilon$-prediction loss on samples drawn from the RL teacher. We use the same prompt set used to train the RL teacher above. Teacher samples are generated without classifier-free guidance using 50-step DDIM with $\eta = 0$ on the standard SD~1.5 noise schedule, which is also the schedule the student trains on ($\beta_{\mathrm{start}} = 8.5 \times 10^{-4}$, $\beta_{\mathrm{end}} = 1.2 \times 10^{-2}$, scaled-linear). Training is buffered: every cycle we generate a fresh batch of $4\,096$ teacher images and take $10$ passes over the buffer, with per-rank batch $32$ and $2$ accumulation steps on $8$ H200 GPUs (effective batch $512$). Optimization uses Adam at learning rate $10^{-5}$ with $100$-step linear warmup, gradient clipping at $1.0$, and bfloat16 mixed precision. We evaluate ImageReward on a held-out validation split every $1\,000$ optimizer steps and early-stop when no improvement is observed for $5$ consecutive eval rounds. The run plotted in \cref{fig:distillation_gap} stopped after $\sim\!9\,000$ optimizer steps under this criterion---roughly $50\times$ more (prompt, teacher-image) pairs seen by the gradient than the RL teacher saw during its own training. The picture is unchanged if we evaluate on training prompts rather than the held-out validation split.

We report and train at CFG${=}1$ (no guidance) as the ``clean'' comparison: the RL teacher is never fine-tuned along the non-conditional branch (adjoint matching only works well without CFG). This also avoids issues given that the CFG distribution of the teacher would not be the CFG distribution of the student if we were to fine-tune the student on the CFG teacher samples.
This is also the setting we discuss in the main text.

\subsection{Distillation from DRL Teachers}
\label{app:setup_distillation_from_drl}

For the \parheadref{sec:exp_distillation_from_drl} experiment, we
use the same DRL model ($\lambda{=}1$, $R_1{=}0$, linear-conditional
discriminator on DINOv2-L) used throughout the main paper as the frozen
teacher. At a high level, the setup mirrors the distillation experiment on
Stable Diffusion 1.5 in \cref{sec:limitations}: we sample from the
post-DRL model to fill a buffer, train the student on those samples with the
standard flow matching velocity loss, and repeat. Unlike the
Stable Diffusion 1.5 experiments, we do not apply early stopping to drive home 
the point that the gap is not solvable with more compute.

\paragraph{Hyperparameters.}
The student is initialized from the base REPA SiT-XL/2 pretrained
checkpoint and trained with the standard flow matching velocity loss on
teacher-generated samples (no REPA projection loss). The teacher generates
samples with Heun-50. We use AdamW with lr\,$=$\,$10^{-4}$ (constant
schedule, 100-step warmup), no weight decay, gradient clipping at $1.0$,
and EMA decay of $0.9999$. Training runs in full precision (float32), as we
occasionally encountered instabilities with bf16 accumulation. The effective
batch size is $576$. Teacher-generated samples are stored in a buffer of
$57{,}600$ samples; the student trains on this buffer for $10$ passes
($1{,}000$ optimizer steps) before the buffer is discarded and refilled
with fresh teacher generations. Evaluation settings follow those described
in the main paper. The results reported in the main text are at CFG${=}1$
and with EMA. Non-EMA results are similar, albeit slightly worse.
Additionally, these hyperparameters correspond almost exactly to the original REPA training configuration, with the only differences being a larger batch size and the absence of the REPA projection loss. We tried other learning rates and buffer-refresh schedules, but found similar or worse results.

\paragraph{Training samples.}
We train for $900$k optimizer steps. Since the buffer is refreshed every
$1{,}000$ steps, this amounts to $900$ buffer fills, each producing
$57{,}600$ fresh samples. The student therefore sees a total of
${\sim}52$M teacher samples over the course of training---over
$150\times$ the number of samples consumed during DRL post-training
(${\sim}346$k) and ${\sim}40\times$ the size of ImageNet
(${\sim}1.28$M images).

\paragraph{Compute.}
The total cost of the $900$k-step run is approximately $4{,}050$
GPU-hours on H200 GPUs. Buffer filling dominates: each fill requires a
$50$-step Heun solve for $57{,}600$ samples and takes ${\sim}400$
seconds, so $900$ fills account for ${\sim}2{,}385$
GPU-hours (${\sim}59\%$ of the budget). The remaining ${\sim}1{,}710$
GPU-hours cover the optimization steps.

\subsection{Feature-space ablation setup}
\label{app:setup_feature_ablation}
We base our feature-space ablation on REPA SiT-XL/2. We train separate linear
conditional discriminators on features from seven embedders---DINOv2 (base,
large), DINOv3 (base, large), SigLIP (base, large), and Inception-v3---and
use each as the reward for adjoint matching. Only the embedder varies across
configs; all other discriminator and adjoint-matching settings follow the
distributional alignment setup
(\cref{app:setup_distributional_alignment}). Every fine-tuned model is then
evaluated across all seven feature spaces; full results are in
\cref{tab:feature_ablation}.

\subsection{Discriminator architecture and training ablation setup}
\label{app:setup_disc_ablation}
The discriminator ablation on REPA SiT-XL/2 with $\lambda = 10$ follows the distributional alignment setup (\cref{app:setup_distributional_alignment}) with the following differences and additions.

\paragraph{Discriminator.}
We explore three additional configurations on top of the default linear projection head on frozen DINOv2-Large:
\emph{(i)}~\emph{MLP head}---two GELU residual blocks of hidden dim $512$ on top of the frozen DINOv2-Large CLS token, followed by the same linear map used in the default;
\emph{(ii)}~\emph{fine-tuning the embedder}---unfreeze DINOv2-Large and train it jointly with the linear head;
\emph{(iii)}~\emph{training the embedder from scratch}---initialize the same DINOv2-L architecture randomly and train it jointly with the linear head.
In all three configurations we keep the class-conditional projection wrapper $D(x,y) = \varphi(h(x)) + \langle h(x), e(y) \rangle + b(y)$ from the alignment setup.

\paragraph{Training.}
Everything else matches the alignment setup, except for the embedder learning rate when the embedder is updated. We use $10^{-5}$ for fine-tuning---we initially tried $10^{-4}$ but training was too unstable and the resulting discriminators were noticeably worse. We nevertheless kept $10^{-4}$ for the from-scratch variant.
At these rates we observed no training instabilities and the discriminator reached the same convergence behavior as the frozen-embedder default.

\FloatBarrier
\clearpage
\section{Extended Results}
\label{app:additional}

This section reports the extended quantitative results that complement the
main-text experiments. Subsections are ordered to match the main-text
experiment flow: alignment (\parheadref{sec:exp_alignment}), image quality
(\parheadref{sec:exp_reward_transfer}), preference RL
(\parheadref{sec:exp_reward}), and ablations (\parheadref{sec:exp_ablations}).
Setups for each experiment, including the ablation training configurations,
are collected in \cref{app:ablation_setups}.

\paragraph{\textit{A note on autoguidance.}}
\label{app:autoguidance}
Standard classifier-free guidance (CFG) replaces the conditional velocity field
$v(x,y)$ with a combination of the conditional and unconditional fields,
$v_{\mathrm{cfg}}(x,y) = v(x,\varnothing) + w\,(v(x,y) - v(x,\varnothing))$, where
$w$ is a positive scalar, usually greater than $1$, and $\varnothing$ denotes the null class.
Autoguidance~\citep{karras2024guiding} proposes replacing $v(x,\varnothing)$ with the velocity
field of a weaker version of the same model, typically an earlier checkpoint from pretraining.
In practice, it has been shown to improve sample quality and diversity.

To our knowledge, autoguidance has only been applied to pretrained models. As a curiosity,
we explored applying it in the RL setting, guiding the RL fine-tuned conditional model with the
corresponding non-RL conditional version,
\begin{equation}
v_{\mathrm{ag}}(x,y) = v_{\mathrm{base}}(x,y) + w\,(v_{\mathrm{RL}}(x,y) - v_{\mathrm{base}}(x,y)).
\end{equation}
We tried this in two settings: using the non-\drl{} model to guide the \drl{} model, and,
after preference alignment, using the non-preference-aligned model to guide the
preference-aligned model. Generally, we found the results to be mixed.
In the \drl{} setting, when $\lambda$ was set to a high value (e.g., $40$), a value of $w$
lower than $1$ usually improved performance over both the base and the \drl{} model.
However, autoguidance generally underperformed standard CFG with an interval.
In the preference-alignment setting, $w > 1$ generally led to the highest reward values
across all models, but also to the worst image fidelity: with sufficiently large $w$, images
became very bright and white, which is presumably why they scored so highly under the proxy
reward despite being poorly aligned with true human preferences.

For these reasons, we do not feature autoguidance in the main text.
Nevertheless, we provide Pareto fronts with autoguidance in
\cref{fig:pareto_hpsv2_combined,fig:pareto_image_reward_combined,fig:pareto_aesthetic_combined}
and the corresponding sample grids in \cref{app:am_samples,app:rl_samples}, so that interested
readers can see the effect in more detail and so that the observations may prove useful for
future work on the topic.

\paragraph{Alignment.}
\cref{fig:lambda_sweep_1,fig:lambda_sweep_2} show how all seven distributional metrics vary with $\lambda$ at CFG${}=1$ (no guidance). For each (model, embedder, $\lambda$) combination, the plotted value corresponds to the interval achieving the lowest FD, with all other metrics taken from the same entry. The dashed gray line marks the base model value, and the dotted line (in the model's color) marks the theoretically motivated $\lambda{=}1$ ($R_1{=}0$) checkpoint.
\cref{tab:prdc_prdc_no_cfg,tab:prdc_prdc_best_cfg} extend the Fr\'echet Distance comparison from \cref{fig:alignment} to the full set of distributional metrics: FD, FD$_{\mathrm{val}}$, KD, Precision, Recall, Density, and Coverage. For each (model, embedder) pair, we select the DRL configuration ($\lambda$, interval) that achieves the lowest FD and report all seven metrics from that same generation run, along with the chosen configuration. This ensures internal consistency across all numbers in each row.

The takeaways are largely unchanged from the main text: DRL consistently improves over the base model across the board. Two patterns become clearer when seeing the full sweep of metrics. First, $\lambda \in \{5, 10\}$ generally performs best across most metrics, with the exception of FID, which tends to prefer smaller $\lambda$. Second, as discussed in the main paper, larger $\lambda$ generally brings improvements over the theoretically motivated baseline. Additionally, $\lambda{=}40$ is never selected as the best configuration and $\lambda{=}20$ is selected only sparingly, which further validates our choice of $\lambda{=}10$ as the default for the reward-improvement and preference-RL experiments.

Finally, while we select the best configuration according to its FD value (computed against the ImageNet \emph{training} distribution), the improvements persist on FD$_{\mathrm{val}}$ and KD, neither of which we select for. Together, these suggest genuine distributional alignment rather than overfitting to the optimized signal.

\paragraph{Image quality.}
\Cref{fig:reward_lambda_sweep} extends the reward-improvement comparison from \cref{fig:reward_improvement}, which fixes $\lambda{=}10$, to the full $\lambda$ sweep. Each panel plots normalized reward improvement $(\drl{-}\textrm{Base})/\sigma_{\textrm{base}}$ against $\lambda$ for one held-out reward. The dotted horizontal line per model marks the theoretically motivated $\lambda{=}1$ ($R_1{=}0$) checkpoint for reference. Improvements are largely monotone in $\lambda$ up to $\lambda{=}20$, with the strongest gains in ImageReward, HPSv2, and PickScore as reported in the main text. For $\lambda{=}40$, however, improvements often dip---consistent with the alignment results in \cref{fig:lambda_sweep_1,fig:lambda_sweep_2}, where $\lambda{=}40$ is rarely the best configuration.
\Cref{tab:best_drl_reward} reports the raw reward values underlying \cref{fig:reward_improvement}: $\lambda{=}10$.

\paragraph{Preference RL: Pareto plots.}
\cref{fig:pareto_hpsv2_cfg1} in the main text shows HPSv2 vs.\ low-level image
statistics at CFG${}=1$ (no guidance). \cref{fig:pareto_hpsv2_combined,fig:pareto_image_reward_combined,fig:pareto_aesthetic_combined} extend this to a full
$3 \times 3$ grid: three held-out rewards (HPSv2, ImageReward, Aesthetic v2.5)
$\times$ three guidance settings (no CFG, CFG${}=2$, autoguidance${}=2$; see \cref{app:autoguidance} for context on autoguidance).
We do not sweep over the guidance value because the results are easier to read this way; the plots look essentially identical for other CFG values.

The results are essentially unchanged from the main text: the patterns we showed for HPSv2 also hold for ImageReward and Aesthetic v2.5, with DRL achieving the best Pareto fronts in nearly every plot. Moreover, this conclusion is robust across the three guidance settings.
Furthermore, as discussed in \cref{app:autoguidance}, autoguidance pushes reward values higher at the cost of substantially worse image-statistic distortion---for example, on REPA the maximum brightness rises from ${\sim}0.6$ under base+CFG to ${\sim}0.7$ under autoguidance, an effect also visible in the sample grids of \cref{app:rl_samples}, where the images are noticeably brighter.
This further supports our claim that reward proxies are usually exploitable and negatively entangled with image fidelity.
Even so, DRL retains the best Pareto front in the autoguidance setting.

\paragraph{Ablations.}
\Cref{fig:disc_ablation_full} reports DINOv2-L FD across the full $R_1$ sweep at both $\lambda{=}1$ and $\lambda{=}10$, complementing the path-based view in \Cref{fig:disc_ablation}. The two regimes show qualitatively different behavior: at $\lambda{=}1$ all $R_1{>}0$ values hurt --- the cleanest baseline is at $R_1{=}0$ --- while at $\lambda{=}10$ a small $R_1$ rescues frozen-feature heads from collapse. This supports the claim in \parheadref{sec:exp_ablations} that $R_1$ here functions as a stabilizer for aggressive $\lambda$ rather than as a vanishing-gradient remedy as in the standard GAN literature~\citep{mescheder2018ganconvergence}.
\Cref{tab:feature_ablation} reports the full $7 \times 7$ feature-space ablation summarized by \cref{tab:feature_ablation_short} in the main text: FD and KD when training a discriminator on each of seven embedders (columns) and evaluating on each of seven embedders (rows). The training setup is identical across embedders (\cref{app:setup_feature_ablation}); only the embedder varies.
We see that even with the additional embedders the conclusions are unchanged.

\paragraph{CFG effect.}
Finally, to provide additional insight into how our method interacts with CFG, \Cref{fig:cfg_effect_reward_raw} and \Cref{fig:cfg_effect_fid_raw} break down the per-reward and per-embedder FD performance of \drl as a function of the CFG scale. The reward plot is constructed by running CFG and using the same setup for estimating rewards described earlier. The FD plot is constructed by taking the minimum over the three CFG interval values described in \cref{app:cfg_tuning} at each CFG scale. We didn't notice meaningful differences across the three intervals for the rewards. 
Several observations are in order.

For the reward values, we see that \drl has a strong positive effect that is maintained at all CFG levels. For some (model, reward) pairs (e.g., SiT on ImageReward) \drl with no CFG already attains higher reward than the base model at any CFG value, and the trend is consistent across model classes. The only exception is RAE, where a small $\lambda$ slightly decreases the reward; this disappears at larger $\lambda$.

For FD the picture is more mixed, with two evident patterns. First, on Inception and SigLIP smaller $\lambda$ generally yields better performance, while on DINOv2 and DINOv3 larger $\lambda$ is generally more beneficial. Second, CFG becomes much less effective the larger $\lambda$ gets.
We believe these are artifacts of the current training setup, in which we only finetune the conditional branch (a constraint imposed by adjoint matching). Finding ways to make the model less dependent on $\lambda$ and trainable jointly with both the CFG and non-CFG branches is an interesting direction for future work.

\subsection{Alignment: Quantitative $\lambda$ Sweep}
\label{app:full_results}

\ifmeta
\begin{figure}[H]
    \centering
    \includegraphics[width=\textwidth]{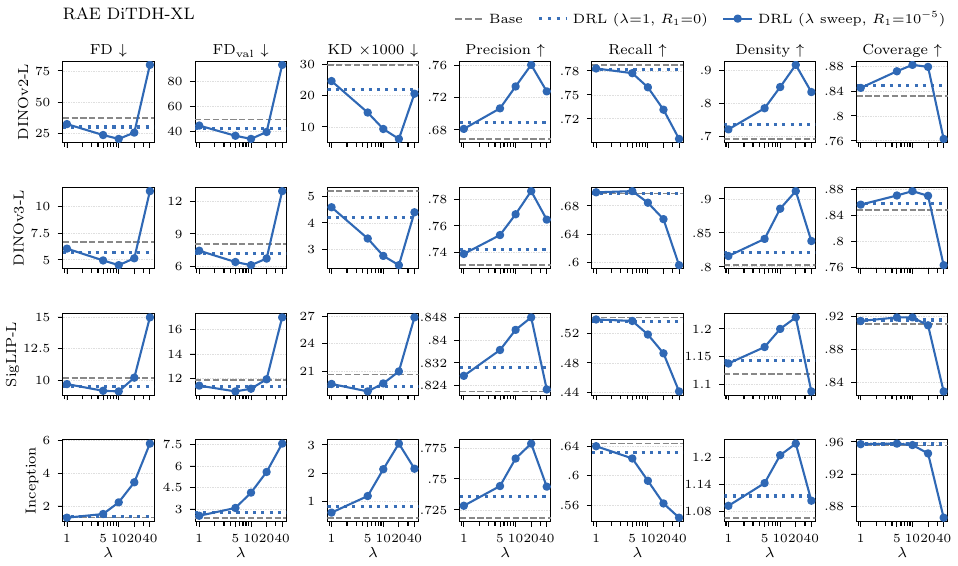}

    \vspace{1.5em}

    \includegraphics[width=\textwidth]{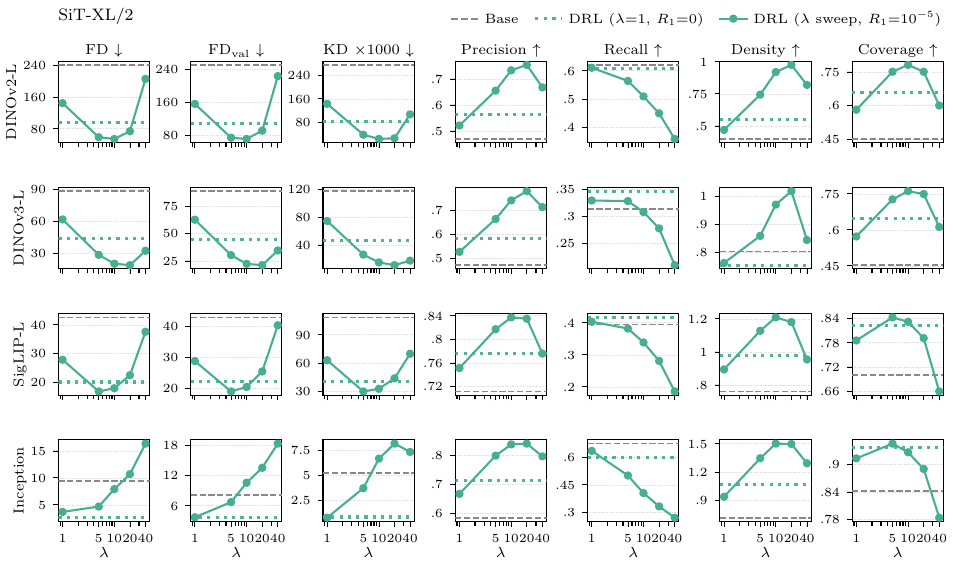}
    \caption{\textbf{Quantitative lambda sweep (RAE + SiT).} Each subplot shows one metric for one embedder. Dashed gray: base model at CFG${}=1$. Dotted (model color): theoretically motivated $\lambda{=}1$ ($R_1{=}0$) checkpoint. Solid line with markers: DRL values at $\lambda \in \{1, 5, 10, 20, 40\}$ from the $R_1{=}10^{-5}$ sweep. Rows: embedders (DINOv2-L, DINOv3-L, SigLIP-L, Inception). Columns: metrics (FD$\downarrow$, FD$_{\mathrm{val}}\downarrow$, KD$\downarrow$, Precision$\uparrow$, Recall$\uparrow$, Density$\uparrow$, Coverage$\uparrow$).}
    \label{fig:lambda_sweep_1}
\end{figure}
\fi

\ifmeta
\begin{figure}[H]
    \centering
    \includegraphics[width=\textwidth]{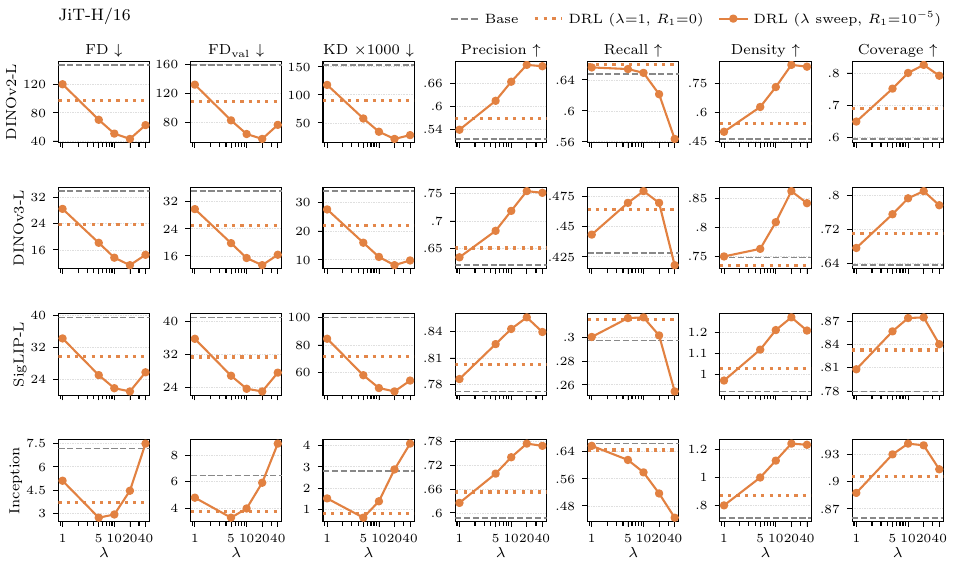}

    \vspace{1em}

    \includegraphics[width=\textwidth]{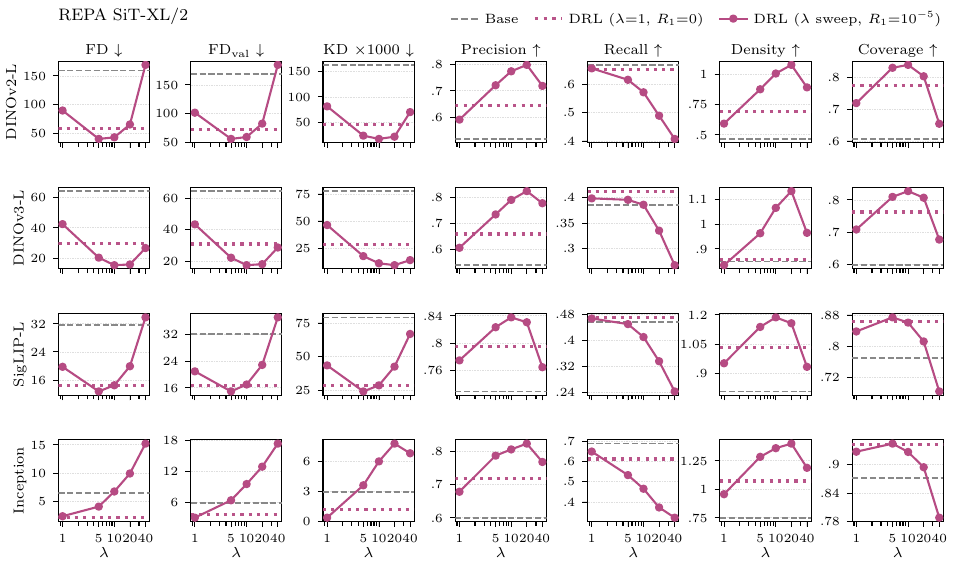}
    \caption{\textbf{Quantitative lambda sweep (JiT + REPA).} Same layout as \cref{fig:lambda_sweep_1}.}
    \label{fig:lambda_sweep_2}
\end{figure}
\fi

\clearpage

\subsection{Alignment: Full Distribution Metrics (PRDC)}
\label{app:prdc}

\begin{table}[H]
\centering
\scriptsize
\setlength{\tabcolsep}{2.5pt}
\caption{Distribution metrics — No CFG. For each (model, embedder) pair, all rows are evaluated at cfg${}=1$; the DRL ($\lambda^*$) row uses the ($\lambda$, time interval) with the lowest FD. All seven metrics come from a single sweep entry. For each (model, embedder), three rows compare Base, DRL at the theoretically motivated $\lambda{=}1$ ($R_1{=}0$), and DRL with $\lambda$ tuned over $\{1,5,10,20,40\}$. \textbf{Bold} marks the best value per metric within each block. FD is computed against the ImageNet training set; FD$_{\mathrm{val}}$ against the validation set. KD is reported as value~$\pm$~standard error of the subset-mean estimator (100 random subsets). The right-most column shows the $(\lambda, R_1)$ of the run each row is drawn from.}
\label{tab:prdc_prdc_no_cfg}
\begin{tabular}{lll c c c c c c c @{\hspace{8pt}} l}
\toprule
Model & Embedder & Method & FD$\downarrow$ & FD$_{\mathrm{val}}\downarrow$ & KD$\downarrow$ & Prec$\uparrow$ & Rec$\uparrow$ & Dens$\uparrow$ & Cov$\uparrow$ & $(\lambda, R_1)$ \\
\midrule
\multirow{12}{*}{RAE} & \multirow{3}{*}{DINOv2-L} & Base & 37.47 & 49.27 & 29.75\,{\scriptsize $\pm$ 0.33} & 0.669 & \textbf{0.787} & 0.692 & 0.832 & $(--,\,--)$ \\
 &  & DRL ($\lambda{=}1$, $R_1{=}0$) & 30.17 & 42.47 & 22.00\,{\scriptsize $\pm$ 0.29} & 0.689 & 0.781 & 0.737 & 0.849 & $(1,\,0)$ \\
 &  & DRL ($\lambda^*$) & \textbf{20.60} & \textbf{33.88} & \textbf{9.31}\,{\scriptsize $\pm$ 0.20} & \textbf{0.734} & 0.759 & \textbf{0.849} & \textbf{0.882} & $(10,\,10^{-5})$ \\
\cmidrule(lr){2-11}
 & \multirow{3}{*}{DINOv3-L} & Base & 6.66 & 8.04 & 5.20\,{\scriptsize $\pm$ 0.03} & 0.731 & 0.697 & 0.802 & 0.849 & $(--,\,--)$ \\
 &  & DRL ($\lambda{=}1$, $R_1{=}0$) & 5.71 & 7.14 & 4.21\,{\scriptsize $\pm$ 0.03} & 0.742 & \textbf{0.698} & 0.821 & 0.859 & $(1,\,0)$ \\
 &  & DRL ($\lambda^*$) & \textbf{4.49} & \textbf{6.10} & \textbf{2.75}\,{\scriptsize $\pm$ 0.02} & \textbf{0.769} & 0.684 & \textbf{0.886} & \textbf{0.877} & $(10,\,10^{-5})$ \\
\cmidrule(lr){2-11}
 & \multirow{3}{*}{SigLIP-L} & Base & 10.19 & 11.88 & 20.64\,{\scriptsize $\pm$ 0.13} & 0.822 & \textbf{0.541} & 1.118 & 0.910 & $(--,\,--)$ \\
 &  & DRL ($\lambda{=}1$, $R_1{=}0$) & 9.53 & 11.31 & \textbf{19.30}\,{\scriptsize $\pm$ 0.15} & 0.830 & 0.535 & 1.143 & 0.915 & $(1,\,0)$ \\
 &  & DRL ($\lambda^*$) & \textbf{9.13} & \textbf{11.15} & 19.68\,{\scriptsize $\pm$ 0.17} & \textbf{0.844} & 0.518 & \textbf{1.200} & \textbf{0.919} & $(10,\,10^{-5})$ \\
\cmidrule(lr){2-11}
 & \multirow{3}{*}{Incep.} & Base & 1.31 & \textbf{2.39} & \textbf{0.43}\,{\scriptsize $\pm$ 0.02} & 0.719 & \textbf{0.643} & 1.066 & 0.956 & $(--,\,--)$ \\
 &  & DRL ($\lambda{=}1$, $R_1{=}0$) & 1.38 & 2.72 & 0.82\,{\scriptsize $\pm$ 0.03} & \textbf{0.736} & 0.632 & \textbf{1.114} & \textbf{0.957} & $(1,\,0)$ \\
 &  & DRL ($\lambda^*$) & \textbf{1.29} & 2.52 & 0.61\,{\scriptsize $\pm$ 0.03} & 0.728 & 0.640 & 1.092 & 0.957 & $(1,\,10^{-5})$ \\
\midrule
\multirow{12}{*}{SiT} & \multirow{3}{*}{DINOv2-L} & Base & 241.93 & 249.99 & 276.55\,{\scriptsize $\pm$ 1.28} & 0.469 & \textbf{0.621} & 0.399 & 0.451 & $(--,\,--)$ \\
 &  & DRL ($\lambda{=}1$, $R_1{=}0$) & 95.51 & 108.56 & 81.43\,{\scriptsize $\pm$ 0.51} & 0.564 & 0.609 & 0.551 & 0.659 & $(1,\,0)$ \\
 &  & DRL ($\lambda^*$) & \textbf{54.37} & \textbf{70.97} & \textbf{22.62}\,{\scriptsize $\pm$ 0.18} & \textbf{0.733} & 0.510 & \textbf{0.920} & \textbf{0.782} & $(10,\,10^{-5})$ \\
\cmidrule(lr){2-11}
 & \multirow{3}{*}{DINOv3-L} & Base & 88.22 & 88.53 & 117.44\,{\scriptsize $\pm$ 0.51} & 0.470 & 0.314 & 0.803 & 0.453 & $(--,\,--)$ \\
 &  & DRL ($\lambda{=}1$, $R_1{=}0$) & 43.76 & 44.88 & 46.23\,{\scriptsize $\pm$ 0.23} & 0.581 & \textbf{0.347} & 0.754 & 0.647 & $(1,\,0)$ \\
 &  & DRL ($\lambda^*$) & \textbf{19.25} & \textbf{21.19} & \textbf{10.90}\,{\scriptsize $\pm$ 0.05} & \textbf{0.781} & 0.278 & \textbf{1.019} & \textbf{0.749} & $(20,\,10^{-5})$ \\
\cmidrule(lr){2-11}
 & \multirow{3}{*}{SigLIP-L} & Base & 42.59 & 42.91 & 108.37\,{\scriptsize $\pm$ 0.52} & 0.712 & 0.395 & 0.761 & 0.700 & $(--,\,--)$ \\
 &  & DRL ($\lambda{=}1$, $R_1{=}0$) & 20.17 & 22.11 & 40.37\,{\scriptsize $\pm$ 0.24} & 0.776 & \textbf{0.417} & 0.977 & 0.823 & $(1,\,0)$ \\
 &  & DRL ($\lambda^*$) & \textbf{16.90} & \textbf{18.94} & \textbf{30.20}\,{\scriptsize $\pm$ 0.21} & \textbf{0.817} & 0.382 & \textbf{1.128} & \textbf{0.843} & $(5,\,10^{-5})$ \\
\cmidrule(lr){2-11}
 & \multirow{3}{*}{Incep.} & Base & 9.38 & 8.13 & 5.27\,{\scriptsize $\pm$ 0.08} & 0.586 & \textbf{0.675} & 0.719 & 0.843 & $(--,\,--)$ \\
 &  & DRL ($\lambda{=}1$, $R_1{=}0$) & \textbf{2.62} & \textbf{3.63} & \textbf{0.85}\,{\scriptsize $\pm$ 0.02} & \textbf{0.713} & 0.597 & \textbf{1.068} & \textbf{0.937} & $(1,\,0)$ \\
 &  & DRL ($\lambda^*$) & \textbf{2.62} & \textbf{3.63} & \textbf{0.85}\,{\scriptsize $\pm$ 0.02} & \textbf{0.713} & 0.597 & \textbf{1.068} & \textbf{0.937} & $(1,\,0)$ \\
\midrule
\multirow{12}{*}{JiT} & \multirow{3}{*}{DINOv2-L} & Base & 148.03 & 158.85 & 153.22\,{\scriptsize $\pm$ 0.71} & 0.515 & 0.648 & 0.463 & 0.596 & $(--,\,--)$ \\
 &  & DRL ($\lambda{=}1$, $R_1{=}0$) & 97.33 & 109.00 & 89.20\,{\scriptsize $\pm$ 0.52} & 0.568 & \textbf{0.659} & 0.545 & 0.690 & $(1,\,0)$ \\
 &  & DRL ($\lambda^*$) & \textbf{42.84} & \textbf{57.29} & \textbf{21.31}\,{\scriptsize $\pm$ 0.20} & \textbf{0.709} & 0.622 & \textbf{0.847} & \textbf{0.825} & $(20,\,10^{-5})$ \\
\cmidrule(lr){2-11}
 & \multirow{3}{*}{DINOv3-L} & Base & 33.89 & 35.11 & 33.84\,{\scriptsize $\pm$ 0.13} & 0.620 & 0.428 & 0.749 & 0.636 & $(--,\,--)$ \\
 &  & DRL ($\lambda{=}1$, $R_1{=}0$) & 23.66 & 25.08 & 21.87\,{\scriptsize $\pm$ 0.09} & 0.651 & 0.464 & 0.735 & 0.711 & $(1,\,0)$ \\
 &  & DRL ($\lambda^*$) & \textbf{11.32} & \textbf{13.24} & \textbf{8.13}\,{\scriptsize $\pm$ 0.03} & \textbf{0.754} & \textbf{0.470} & \textbf{0.862} & \textbf{0.810} & $(20,\,10^{-5})$ \\
\cmidrule(lr){2-11}
 & \multirow{3}{*}{SigLIP-L} & Base & 39.41 & 40.94 & 100.22\,{\scriptsize $\pm$ 0.38} & 0.772 & 0.297 & 0.921 & 0.779 & $(--,\,--)$ \\
 &  & DRL ($\lambda{=}1$, $R_1{=}0$) & 29.72 & 31.36 & 71.41\,{\scriptsize $\pm$ 0.28} & 0.802 & \textbf{0.315} & 1.030 & 0.833 & $(1,\,0)$ \\
 &  & DRL ($\lambda^*$) & \textbf{20.99} & \textbf{23.04} & \textbf{45.97}\,{\scriptsize $\pm$ 0.20} & \textbf{0.856} & 0.302 & \textbf{1.270} & \textbf{0.875} & $(20,\,10^{-5})$ \\
\cmidrule(lr){2-11}
 & \multirow{3}{*}{Incep.} & Base & 7.16 & 6.47 & 2.81\,{\scriptsize $\pm$ 0.05} & 0.588 & \textbf{0.664} & 0.714 & 0.860 & $(--,\,--)$ \\
 &  & DRL ($\lambda{=}1$, $R_1{=}0$) & 3.72 & 3.73 & 0.79\,{\scriptsize $\pm$ 0.02} & 0.652 & 0.645 & 0.871 & 0.905 & $(1,\,0)$ \\
 &  & DRL ($\lambda^*$) & \textbf{2.73} & \textbf{3.28} & \textbf{0.61}\,{\scriptsize $\pm$ 0.01} & \textbf{0.699} & 0.616 & \textbf{1.001} & \textbf{0.930} & $(5,\,10^{-5})$ \\
\midrule
\multirow{12}{*}{REPA} & \multirow{3}{*}{DINOv2-L} & Base & 159.46 & 168.91 & 163.05\,{\scriptsize $\pm$ 0.95} & 0.517 & \textbf{0.668} & 0.465 & 0.606 & $(--,\,--)$ \\
 &  & DRL ($\lambda{=}1$, $R_1{=}0$) & 58.33 & 72.52 & 45.83\,{\scriptsize $\pm$ 0.33} & 0.644 & 0.653 & 0.693 & 0.773 & $(1,\,0)$ \\
 &  & DRL ($\lambda^*$) & \textbf{40.33} & \textbf{56.12} & \textbf{23.77}\,{\scriptsize $\pm$ 0.22} & \textbf{0.721} & 0.615 & \textbf{0.877} & \textbf{0.829} & $(5,\,10^{-5})$ \\
\cmidrule(lr){2-11}
 & \multirow{3}{*}{DINOv3-L} & Base & 63.87 & 64.46 & 77.62\,{\scriptsize $\pm$ 0.38} & 0.540 & 0.385 & 0.848 & 0.597 & $(--,\,--)$ \\
 &  & DRL ($\lambda{=}1$, $R_1{=}0$) & 29.43 & 30.81 & 28.16\,{\scriptsize $\pm$ 0.16} & 0.660 & \textbf{0.412} & 0.855 & 0.763 & $(1,\,0)$ \\
 &  & DRL ($\lambda^*$) & \textbf{15.58} & \textbf{17.37} & \textbf{10.88}\,{\scriptsize $\pm$ 0.06} & \textbf{0.791} & 0.386 & \textbf{1.066} & \textbf{0.827} & $(10,\,10^{-5})$ \\
\cmidrule(lr){2-11}
 & \multirow{3}{*}{SigLIP-L} & Base & 31.56 & 32.13 & 79.18\,{\scriptsize $\pm$ 0.45} & 0.729 & 0.456 & 0.809 & 0.769 & $(--,\,--)$ \\
 &  & DRL ($\lambda{=}1$, $R_1{=}0$) & 14.34 & 16.56 & 28.71\,{\scriptsize $\pm$ 0.21} & 0.795 & \textbf{0.470} & 1.033 & 0.864 & $(1,\,0)$ \\
 &  & DRL ($\lambda^*$) & \textbf{12.77} & \textbf{14.88} & \textbf{23.97}\,{\scriptsize $\pm$ 0.21} & \textbf{0.823} & 0.449 & \textbf{1.138} & \textbf{0.874} & $(5,\,10^{-5})$ \\
\cmidrule(lr){2-11}
 & \multirow{3}{*}{Incep.} & Base & 6.48 & 5.81 & 2.94\,{\scriptsize $\pm$ 0.06} & 0.600 & \textbf{0.687} & 0.750 & 0.872 & $(--,\,--)$ \\
 &  & DRL ($\lambda{=}1$, $R_1{=}0$) & \textbf{2.14} & \textbf{3.66} & \textbf{1.16}\,{\scriptsize $\pm$ 0.03} & \textbf{0.719} & 0.611 & \textbf{1.072} & \textbf{0.943} & $(1,\,0)$ \\
 &  & DRL ($\lambda^*$) & \textbf{2.14} & \textbf{3.66} & \textbf{1.16}\,{\scriptsize $\pm$ 0.03} & \textbf{0.719} & 0.611 & \textbf{1.072} & \textbf{0.943} & $(1,\,0)$ \\
\bottomrule
\end{tabular}
\end{table}

\clearpage

\begin{table}[H]
\centering
\scriptsize
\setlength{\tabcolsep}{2.5pt}
\caption{Distribution metrics — Best CFG. For each (model, embedder) pair and each method row, the (CFG scale, time interval) combination minimising FD is selected; all seven metrics are taken from that single run. For each (model, embedder), three rows compare Base, DRL at the theoretically motivated $\lambda{=}1$ ($R_1{=}0$), and DRL with $\lambda$ tuned over $\{1,5,10,20,40\}$. \textbf{Bold} marks the best value per metric within each block. FD is computed against the ImageNet training set; FD$_{\mathrm{val}}$ against the validation set. KD is reported as value~$\pm$~standard error of the subset-mean estimator (100 random subsets). The right-most column shows the configuration $(\lambda, R_1, \mathrm{cfg}, [\mathrm{interval}])$ of the run each row is drawn from.}
\label{tab:prdc_prdc_best_cfg}
\begin{tabular}{lll c c c c c c c @{\hspace{8pt}} l}
\toprule
Model & Embedder & Method & FD$\downarrow$ & FD$_{\mathrm{val}}\downarrow$ & KD$\downarrow$ & Prec$\uparrow$ & Rec$\uparrow$ & Dens$\uparrow$ & Cov$\uparrow$ & $(\lambda, R_1, \mathrm{cfg}, [\mathrm{int.}])$ \\
\midrule
\multirow{12}{*}{RAE} & \multirow{3}{*}{DINOv2-L} & Base & 25.94 & 40.78 & 16.59\,{\scriptsize $\pm$ 0.22} & \textbf{0.764} & 0.677 & \textbf{0.924} & 0.887 & $(--,\,--,\,1.5,\,[0,1])$ \\
 &  & DRL ($\lambda{=}1$, $R_1{=}0$) & 23.66 & 37.74 & 15.22\,{\scriptsize $\pm$ 0.22} & 0.743 & \textbf{0.723} & 0.870 & 0.884 & $(1,\,0,\,1.25,\,[0,1])$ \\
 &  & DRL ($\lambda^*$) & \textbf{20.24} & \textbf{34.68} & \textbf{10.51}\,{\scriptsize $\pm$ 0.18} & 0.762 & 0.719 & 0.915 & \textbf{0.897} & $(5,\,10^{-5},\,1.25,\,[0,1])$ \\
\cmidrule(lr){2-11}
 & \multirow{3}{*}{DINOv3-L} & Base & 4.83 & 6.66 & 3.37\,{\scriptsize $\pm$ 0.02} & 0.810 & 0.608 & 0.992 & \textbf{0.895} & $(--,\,--,\,1.5,\,[0,1])$ \\
 &  & DRL ($\lambda{=}1$, $R_1{=}0$) & 4.57 & 6.44 & 2.91\,{\scriptsize $\pm$ 0.02} & \textbf{0.820} & 0.605 & \textbf{1.013} & 0.895 & $(1,\,0,\,1.5,\,[0,1])$ \\
 &  & DRL ($\lambda^*$) & \textbf{4.26} & \textbf{5.97} & \textbf{2.71}\,{\scriptsize $\pm$ 0.02} & 0.800 & \textbf{0.651} & 0.954 & 0.895 & $(5,\,10^{-5},\,1.25,\,[0,1])$ \\
\cmidrule(lr){2-11}
 & \multirow{3}{*}{SigLIP-L} & Base & 9.50 & 11.81 & 22.86\,{\scriptsize $\pm$ 0.20} & 0.855 & 0.494 & 1.252 & 0.925 & $(--,\,--,\,1.25,\,[0,1])$ \\
 &  & DRL ($\lambda{=}1$, $R_1{=}0$) & 9.40 & 11.77 & 23.36\,{\scriptsize $\pm$ 0.23} & \textbf{0.860} & 0.488 & \textbf{1.280} & \textbf{0.927} & $(1,\,0,\,1.25,\,[0,1])$ \\
 &  & DRL ($\lambda^*$) & \textbf{9.06} & \textbf{11.03} & \textbf{19.18}\,{\scriptsize $\pm$ 0.17} & 0.839 & \textbf{0.524} & 1.185 & 0.917 & $(10,\,10^{-5},\,1.75,\,[.1,1])$ \\
\cmidrule(lr){2-11}
 & \multirow{3}{*}{Incep.} & Base & 1.28 & 2.32 & \textbf{0.35}\,{\scriptsize $\pm$ 0.02} & 0.714 & 0.651 & 1.043 & \textbf{0.953} & $(--,\,--,\,2,\,[.1,1])$ \\
 &  & DRL ($\lambda{=}1$, $R_1{=}0$) & 1.25 & 2.43 & 0.49\,{\scriptsize $\pm$ 0.02} & \textbf{0.716} & 0.645 & \textbf{1.052} & 0.953 & $(1,\,0,\,2.5,\,[.1,1])$ \\
 &  & DRL ($\lambda^*$) & \textbf{1.19} & \textbf{2.31} & 0.38\,{\scriptsize $\pm$ 0.02} & 0.714 & \textbf{0.652} & 1.042 & 0.952 & $(1,\,10^{-5},\,2.25,\,[.1,1])$ \\
\midrule
\multirow{12}{*}{SiT} & \multirow{3}{*}{DINOv2-L} & Base & 46.96 & 62.68 & 28.17\,{\scriptsize $\pm$ 0.25} & \textbf{0.801} & 0.436 & \textbf{1.171} & \textbf{0.862} & $(--,\,--,\,2.75,\,[.1,1])$ \\
 &  & DRL ($\lambda{=}1$, $R_1{=}0$) & \textbf{42.36} & \textbf{58.94} & \textbf{23.46}\,{\scriptsize $\pm$ 0.25} & 0.796 & \textbf{0.465} & 1.148 & 0.860 & $(1,\,0,\,2.75,\,[.3,1])$ \\
 &  & DRL ($\lambda^*$) & \textbf{42.36} & \textbf{58.94} & \textbf{23.46}\,{\scriptsize $\pm$ 0.25} & 0.796 & \textbf{0.465} & 1.148 & 0.860 & $(1,\,0,\,2.75,\,[.3,1])$ \\
\cmidrule(lr){2-11}
 & \multirow{3}{*}{DINOv3-L} & Base & 21.63 & 24.10 & 16.70\,{\scriptsize $\pm$ 0.08} & 0.835 & 0.156 & 1.337 & 0.800 & $(--,\,--,\,3.5,\,[0,1])$ \\
 &  & DRL ($\lambda{=}1$, $R_1{=}0$) & 18.32 & 20.61 & 15.16\,{\scriptsize $\pm$ 0.08} & 0.837 & 0.220 & 1.303 & \textbf{0.840} & $(1,\,0,\,3,\,[.1,1])$ \\
 &  & DRL ($\lambda^*$) & \textbf{16.40} & \textbf{18.54} & \textbf{10.99}\,{\scriptsize $\pm$ 0.05} & \textbf{0.851} & \textbf{0.234} & \textbf{1.354} & 0.831 & $(10,\,10^{-5},\,2.25,\,[.3,1])$ \\
\cmidrule(lr){2-11}
 & \multirow{3}{*}{SigLIP-L} & Base & 11.81 & \textbf{13.48} & 25.30\,{\scriptsize $\pm$ 0.21} & 0.842 & 0.400 & 1.241 & \textbf{0.898} & $(--,\,--,\,3.5,\,[.3,1])$ \\
 &  & DRL ($\lambda{=}1$, $R_1{=}0$) & 12.17 & 14.95 & 27.46\,{\scriptsize $\pm$ 0.29} & 0.842 & \textbf{0.405} & 1.238 & 0.887 & $(1,\,0,\,2.25,\,[.3,1])$ \\
 &  & DRL ($\lambda^*$) & \textbf{11.58} & 13.64 & \textbf{24.76}\,{\scriptsize $\pm$ 0.26} & \textbf{0.845} & 0.401 & \textbf{1.253} & 0.896 & $(1,\,10^{-5},\,2.75,\,[.3,1])$ \\
\cmidrule(lr){2-11}
 & \multirow{3}{*}{Incep.} & Base & \textbf{1.50} & \textbf{2.90} & \textbf{0.70}\,{\scriptsize $\pm$ 0.02} & 0.713 & \textbf{0.626} & 1.065 & 0.941 & $(--,\,--,\,2,\,[.3,1])$ \\
 &  & DRL ($\lambda{=}1$, $R_1{=}0$) & 2.43 & 4.17 & 1.77\,{\scriptsize $\pm$ 0.04} & \textbf{0.742} & 0.582 & \textbf{1.169} & \textbf{0.947} & $(1,\,0,\,1.25,\,[.3,1])$ \\
 &  & DRL ($\lambda^*$) & 1.78 & 3.26 & 1.08\,{\scriptsize $\pm$ 0.03} & 0.734 & 0.610 & 1.136 & 0.946 & $(1,\,10^{-5},\,1.5,\,[.3,1])$ \\
\midrule
\multirow{12}{*}{JiT} & \multirow{3}{*}{DINOv2-L} & Base & 41.55 & 54.96 & 31.69\,{\scriptsize $\pm$ 0.25} & 0.737 & 0.576 & 0.896 & 0.851 & $(--,\,--,\,4,\,[.1,1])$ \\
 &  & DRL ($\lambda{=}1$, $R_1{=}0$) & 36.10 & 49.96 & 24.87\,{\scriptsize $\pm$ 0.23} & 0.753 & \textbf{0.585} & 0.943 & 0.867 & $(1,\,0,\,3.5,\,[.1,1])$ \\
 &  & DRL ($\lambda^*$) & \textbf{30.64} & \textbf{45.59} & \textbf{15.39}\,{\scriptsize $\pm$ 0.17} & \textbf{0.788} & 0.574 & \textbf{1.054} & \textbf{0.884} & $(10,\,10^{-5},\,2.75,\,[.1,1])$ \\
\cmidrule(lr){2-11}
 & \multirow{3}{*}{DINOv3-L} & Base & 10.89 & 13.16 & 8.56\,{\scriptsize $\pm$ 0.04} & 0.813 & 0.376 & 1.038 & 0.841 & $(--,\,--,\,2.75,\,[0,1])$ \\
 &  & DRL ($\lambda{=}1$, $R_1{=}0$) & 10.19 & 12.43 & 7.74\,{\scriptsize $\pm$ 0.03} & 0.816 & 0.397 & \textbf{1.039} & 0.847 & $(1,\,0,\,2.25,\,[0,1])$ \\
 &  & DRL ($\lambda^*$) & \textbf{8.10} & \textbf{10.14} & \textbf{5.94}\,{\scriptsize $\pm$ 0.03} & \textbf{0.818} & \textbf{0.456} & 1.029 & \textbf{0.862} & $(20,\,10^{-5},\,3.25,\,[.1,1])$ \\
\cmidrule(lr){2-11}
 & \multirow{3}{*}{SigLIP-L} & Base & 18.21 & 19.93 & 46.80\,{\scriptsize $\pm$ 0.19} & 0.837 & 0.348 & 1.168 & 0.888 & $(--,\,--,\,3.5,\,[.1,1])$ \\
 &  & DRL ($\lambda{=}1$, $R_1{=}0$) & 17.11 & 18.94 & 42.95\,{\scriptsize $\pm$ 0.20} & 0.846 & \textbf{0.348} & 1.206 & 0.896 & $(1,\,0,\,3,\,[.1,1])$ \\
 &  & DRL ($\lambda^*$) & \textbf{16.47} & \textbf{18.43} & \textbf{42.03}\,{\scriptsize $\pm$ 0.21} & \textbf{0.854} & 0.339 & \textbf{1.253} & \textbf{0.899} & $(5,\,10^{-5},\,3,\,[.1,1])$ \\
\cmidrule(lr){2-11}
 & \multirow{3}{*}{Incep.} & Base & 1.91 & 3.16 & 0.74\,{\scriptsize $\pm$ 0.02} & 0.684 & \textbf{0.629} & 0.950 & 0.928 & $(--,\,--,\,2.25,\,[.1,1])$ \\
 &  & DRL ($\lambda{=}1$, $R_1{=}0$) & \textbf{1.87} & \textbf{3.05} & \textbf{0.74}\,{\scriptsize $\pm$ 0.02} & \textbf{0.698} & 0.627 & \textbf{0.976} & \textbf{0.932} & $(1,\,0,\,2.5,\,[.3,1])$ \\
 &  & DRL ($\lambda^*$) & \textbf{1.87} & \textbf{3.05} & \textbf{0.74}\,{\scriptsize $\pm$ 0.02} & \textbf{0.698} & 0.627 & \textbf{0.976} & \textbf{0.932} & $(1,\,0,\,2.5,\,[.3,1])$ \\
\midrule
\multirow{12}{*}{REPA} & \multirow{3}{*}{DINOv2-L} & Base & 36.87 & 52.75 & 22.95\,{\scriptsize $\pm$ 0.22} & \textbf{0.809} & 0.511 & \textbf{1.111} & \textbf{0.884} & $(--,\,--,\,2.75,\,[.1,1])$ \\
 &  & DRL ($\lambda{=}1$, $R_1{=}0$) & \textbf{32.87} & \textbf{49.11} & \textbf{20.44}\,{\scriptsize $\pm$ 0.21} & 0.797 & \textbf{0.567} & 1.058 & 0.882 & $(1,\,0,\,2.75,\,[.3,1])$ \\
 &  & DRL ($\lambda^*$) & \textbf{32.87} & \textbf{49.11} & \textbf{20.44}\,{\scriptsize $\pm$ 0.21} & 0.797 & \textbf{0.567} & 1.058 & 0.882 & $(1,\,0,\,2.75,\,[.3,1])$ \\
\cmidrule(lr){2-11}
 & \multirow{3}{*}{DINOv3-L} & Base & 18.60 & 20.97 & 14.02\,{\scriptsize $\pm$ 0.08} & 0.831 & 0.238 & 1.233 & 0.818 & $(--,\,--,\,3.25,\,[0,1])$ \\
 &  & DRL ($\lambda{=}1$, $R_1{=}0$) & 14.58 & 16.80 & 11.71\,{\scriptsize $\pm$ 0.06} & 0.841 & 0.324 & 1.235 & \textbf{0.866} & $(1,\,0,\,2.75,\,[.1,1])$ \\
 &  & DRL ($\lambda^*$) & \textbf{13.01} & \textbf{15.00} & \textbf{9.02}\,{\scriptsize $\pm$ 0.05} & \textbf{0.852} & \textbf{0.354} & \textbf{1.243} & 0.862 & $(10,\,10^{-5},\,2.25,\,[.3,1])$ \\
\cmidrule(lr){2-11}
 & \multirow{3}{*}{SigLIP-L} & Base & 9.52 & \textbf{11.18} & \textbf{20.65}\,{\scriptsize $\pm$ 0.20} & 0.830 & \textbf{0.488} & 1.141 & \textbf{0.910} & $(--,\,--,\,3.75,\,[.3,1])$ \\
 &  & DRL ($\lambda{=}1$, $R_1{=}0$) & 10.02 & 12.77 & 23.03\,{\scriptsize $\pm$ 0.27} & 0.832 & 0.475 & \textbf{1.172} & 0.899 & $(1,\,0,\,2.25,\,[.3,1])$ \\
 &  & DRL ($\lambda^*$) & \textbf{9.32} & 11.34 & 20.71\,{\scriptsize $\pm$ 0.23} & \textbf{0.833} & 0.482 & 1.165 & 0.907 & $(1,\,10^{-5},\,3,\,[.3,1])$ \\
\cmidrule(lr){2-11}
 & \multirow{3}{*}{Incep.} & Base & \textbf{1.21} & \textbf{2.77} & \textbf{0.71}\,{\scriptsize $\pm$ 0.03} & 0.702 & \textbf{0.642} & 1.004 & 0.939 & $(--,\,--,\,2.25,\,[.3,1])$ \\
 &  & DRL ($\lambda{=}1$, $R_1{=}0$) & 2.14 & 3.66 & 1.16\,{\scriptsize $\pm$ 0.03} & \textbf{0.719} & 0.611 & \textbf{1.072} & 0.943 & $(1,\,0,\,1,\,[0,1])$ \\
 &  & DRL ($\lambda^*$) & 1.50 & 3.05 & 0.94\,{\scriptsize $\pm$ 0.03} & 0.718 & 0.621 & 1.067 & \textbf{0.945} & $(1,\,10^{-5},\,1.5,\,[.3,1])$ \\
\bottomrule
\end{tabular}
\end{table}

\clearpage

\subsection{Image Quality: Reward Improvement vs.\ $\lambda$}
\label{app:reward_lambda_sweep}

\ifmeta
\begin{figure}[H]
    \centering
    \caption{\textbf{Reward improvement vs.\ $\lambda$.} Per-model normalized improvement on each of four held-out preference rewards. Solid lines with markers: DRL at $\lambda \in \{1,5,10,20,40\}$ from the $R_1{=}10^{-5}$ sweep. Dotted lines (in the model's color): theoretically motivated $\lambda{=}1$ ($R_1{=}0$) checkpoint. Dashed gray line at zero: Base model. Top row: CFG${}=1$ (no guidance). Bottom row: best CFG. Same units as \cref{fig:reward_improvement}.}
    \label{fig:reward_lambda_sweep}
    \includegraphics[width=\textwidth]{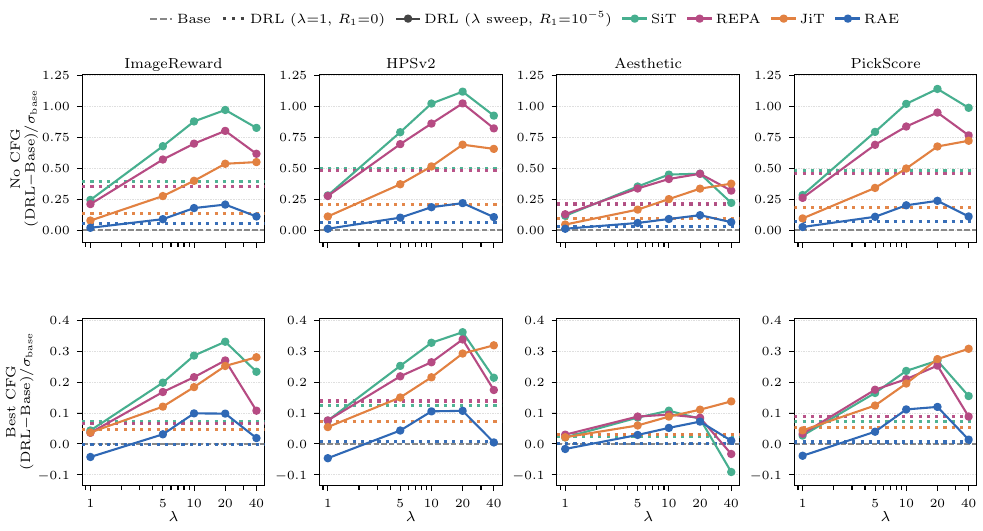}
\end{figure}
\fi

\begin{table}[H]
\centering
\small
\caption{Numerical values for \Cref{fig:reward_improvement}. DRL is pinned to the headline $\lambda{=}10$; the guidance scale is tuned per (model, metric) cell. \textbf{Bold} indicates the better value within each Base/DRL pair.}
\label{tab:best_drl_reward}
\setlength{\tabcolsep}{4pt}
\begin{tabular}{ll cccc}
\toprule
& & \multicolumn{2}{c}{No guidance} & \multicolumn{2}{c}{Guidance} \\
\cmidrule(lr){3-4}\cmidrule(lr){5-6}
Model & Metric & Base & DRL & Base & DRL \\
\midrule
\multirow{4}{*}{SiT-XL/2} & ImageReward$\uparrow$ & -1.252 & \textbf{-0.486}  & -0.584 & \textbf{-0.309} \\
 & HPSv2$\uparrow$ & 0.197 & \textbf{0.224}  & 0.233 & \textbf{0.242} \\
 & Aesthetic$\uparrow$ & 3.149 & \textbf{3.362}  & 3.473 & \textbf{3.533} \\
 & PickScore$\uparrow$ & 19.140 & \textbf{20.086}  & 20.153 & \textbf{20.383} \\
\midrule
\multirow{4}{*}{REPA SiT-XL/2} & ImageReward$\uparrow$ & -1.187 & \textbf{-0.562}  & -0.622 & \textbf{-0.415} \\
 & HPSv2$\uparrow$ & 0.200 & \textbf{0.223}  & 0.230 & \textbf{0.237} \\
 & Aesthetic$\uparrow$ & 3.185 & \textbf{3.386}  & 3.456 & \textbf{3.508} \\
 & PickScore$\uparrow$ & 19.240 & \textbf{20.044}  & 20.094 & \textbf{20.292} \\
\midrule
\multirow{4}{*}{JiT-H/16} & ImageReward$\uparrow$ & -1.159 & \textbf{-0.804}  & -0.723 & \textbf{-0.551} \\
 & HPSv2$\uparrow$ & 0.196 & \textbf{0.210}  & 0.226 & \textbf{0.231} \\
 & Aesthetic$\uparrow$ & 3.078 & \textbf{3.192}  & 3.321 & \textbf{3.366} \\
 & PickScore$\uparrow$ & 19.239 & \textbf{19.710}  & 19.904 & \textbf{20.079} \\
\midrule
\multirow{4}{*}{RAE DiTDH-XL} & ImageReward$\uparrow$ & -0.933 & \textbf{-0.764}  & -0.583 & \textbf{-0.488} \\
 & HPSv2$\uparrow$ & 0.217 & \textbf{0.222}  & 0.227 & \textbf{0.230} \\
 & Aesthetic$\uparrow$ & 3.292 & \textbf{3.338}  & 3.398 & \textbf{3.425} \\
 & PickScore$\uparrow$ & 19.784 & \textbf{19.990}  & 20.187 & \textbf{20.297} \\
\bottomrule
\end{tabular}
\end{table}

\clearpage

\subsection{Image Quality / Preference RL: Pareto Plots}
\label{app:pareto}

\ifmeta
\begin{figure}[H]
    \centering
    \includegraphics[height=0.78\textheight]{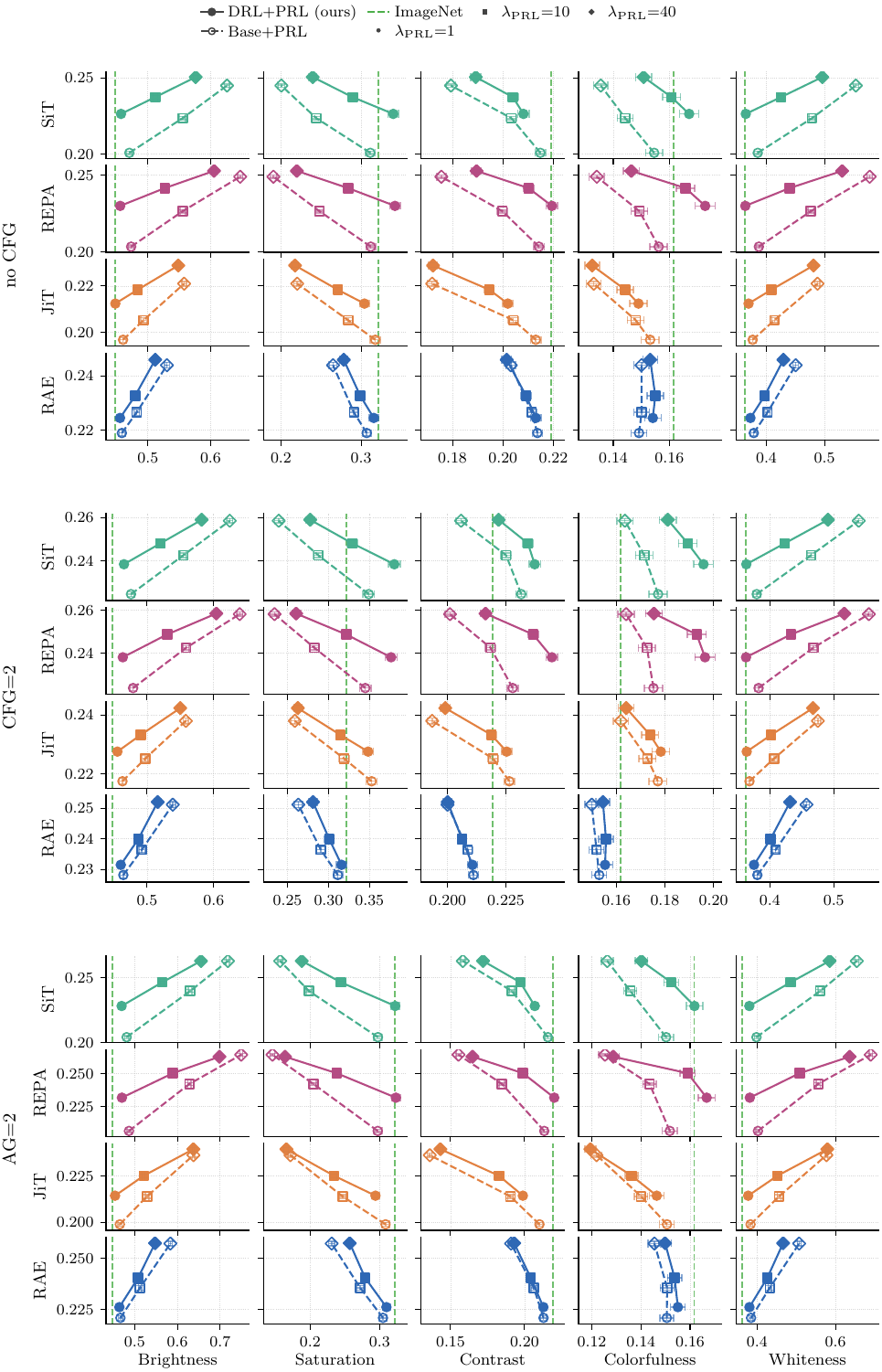}
    \caption{\textbf{HPSv2 vs.\ image statistics across guidance settings.} Each panel-group is a 4 (model) $\times$ 5 (image statistic) grid; rows correspond to no CFG (CFG${=}1$), CFG${=}2$, and autoguidance${=}2$ (see \cref{app:autoguidance} for context on autoguidance). Markers: base ($\bigstar$), Base$+$PRL (dashed, hollow), DRL$+$PRL (solid, filled) at $\lambda\in\{1,10,40\}$. Vertical dashed green line: ImageNet reference value per statistic. The autoguidance group uses each method's pre-PRL checkpoint as the negative signal at scale~2 (Base for Base$+$PRL, DRL for DRL$+$PRL); the base reference still uses standard CFG at the same scale. DRL$+$PRL consistently attains better Pareto fronts across guidance settings.}
    \label{fig:pareto_hpsv2_combined}
\end{figure}
\fi

\FloatBarrier
\clearpage

\ifmeta
\begin{figure}[H]
    \centering
    \includegraphics[height=0.78\textheight]{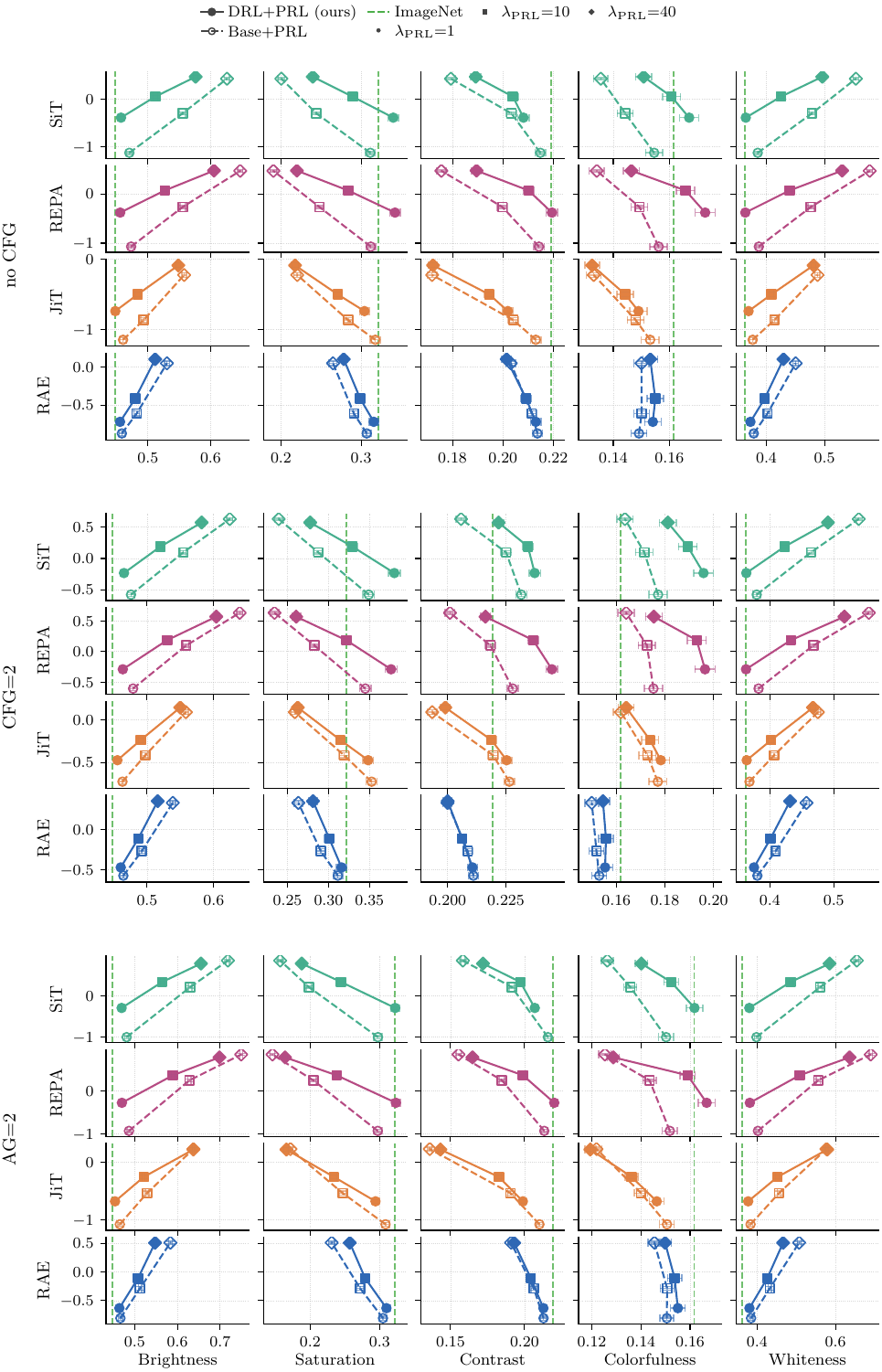}
    \caption{\textbf{ImageReward vs.\ image statistics across guidance settings.} Mirrors \cref{fig:pareto_hpsv2_combined} using ImageReward as the held-out reward. DRL$+$PRL consistently attains better Pareto fronts across guidance settings.}
    \label{fig:pareto_image_reward_combined}
\end{figure}
\fi

\FloatBarrier
\clearpage

\ifmeta
\begin{figure}[H]
    \centering
    \includegraphics[height=0.78\textheight]{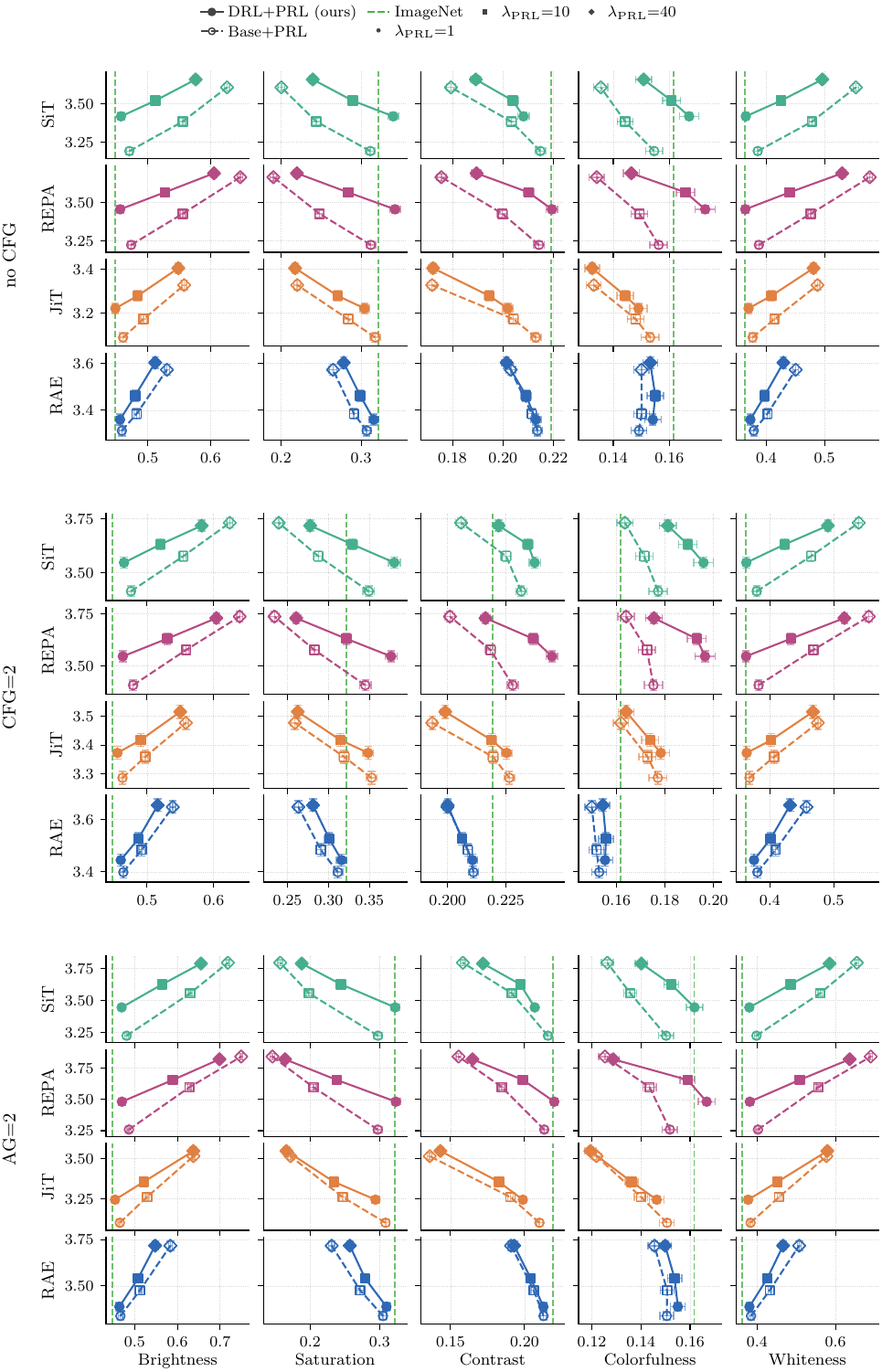}
    \caption{\textbf{Aesthetic v2.5 vs.\ image statistics across guidance settings.} Mirrors \cref{fig:pareto_hpsv2_combined} using Aesthetic v2.5 as the held-out reward. DRL$+$PRL consistently attains better Pareto fronts across guidance settings.}
    \label{fig:pareto_aesthetic_combined}
\end{figure}
\fi

\FloatBarrier
\clearpage

\subsection{Ablations: Discriminator Full $R_1$ Sweep}
\label{app:disc_ablation_full}

\ifmeta
\begin{figure}[h!]
    \centering
    \caption{\textbf{Discriminator ablation: full $R_1$ sweep.} DINOv2-L FD vs.\ $R_1$ for four head/training configurations, at $\lambda{=}1$ (left) and $\lambda{=}10$ (right). Top: fixed-seed samples at four representative configurations. Dashed gray line: base model FD.}
    \label{fig:disc_ablation_full}
    \includegraphics[width=0.78\textwidth]{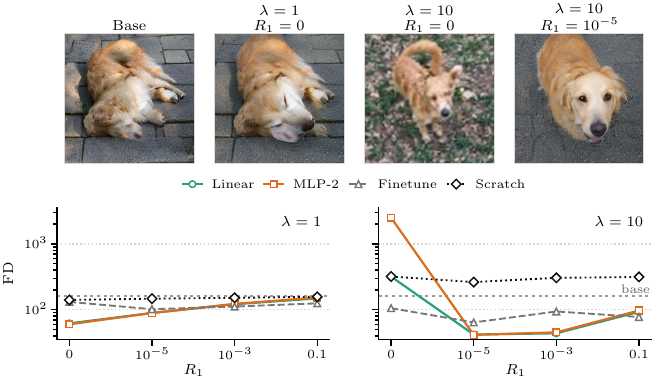}
\end{figure}
\fi

\subsection{Ablations: Full Feature-Space Sweep}
\label{app:feature_ablation_full}

\begin{table}[h!]
    \centering
    \caption{\textbf{Feature space ablation results.} Fr\'echet Distance (FD, $\downarrow$) and Kernel Distance (KD$\times 10^3$, $\downarrow$) when training the discriminator in different feature spaces (columns) and evaluating in different feature spaces (rows). Best fine-tuned result per row in \textbf{bold}. Training a discriminator on a particular embedding consistently improves downstream alignment metrics based on that embedder; DINOv2-L gives the best cross-feature transfer.}
    \label{tab:feature_ablation}
    \small
    \setlength{\tabcolsep}{3.5pt}
    \begin{tabular}{l l c c c c c c c c }
\toprule
 & & & \multicolumn{7}{c}{\textit{Training feature space}} \\
\cmidrule(lr){4-10}
 & \textit{Eval.\ space} & Base & DINOv2-B & DINOv2-L & DINOv3-B & DINOv3-L & SigLIP-B & SigLIP-L & Incep. \\
\midrule
\multirow{7}{*}{\rotatebox[origin=c]{90}{FD}} & Inception & \cellcolor[HTML]{F5D3BC}6.43 & \cellcolor[HTML]{BFE3D8}2.18 & \cellcolor[HTML]{BEE3D7}\textbf{2.14} & \cellcolor[HTML]{BFE3D8}2.19 & \cellcolor[HTML]{D5EDE5}3.02 & \cellcolor[HTML]{C9E8DE}2.57 & \cellcolor[HTML]{C3E5DB}2.34 & \cellcolor[HTML]{F8E0D0}5.68 \\
 & DINOv2-B & \cellcolor[HTML]{F5D3BC}130 & \cellcolor[HTML]{C7E7DD}50.8 & \cellcolor[HTML]{BEE3D7}\textbf{44.3} & \cellcolor[HTML]{CBE9DF}54.3 & \cellcolor[HTML]{C2E4DA}47.3 & \cellcolor[HTML]{E4F3EE}72.9 & \cellcolor[HTML]{D9EEE8}64.5 & \cellcolor[HTML]{F6D9C5}123 \\
 & DINOv2-L & \cellcolor[HTML]{F5D3BC}159 & \cellcolor[HTML]{CCE9E0}71.0 & \cellcolor[HTML]{BEE3D7}\textbf{58.3} & \cellcolor[HTML]{D0EBE3}74.7 & \cellcolor[HTML]{C3E5DB}63.2 & \cellcolor[HTML]{E8F5F1}96.1 & \cellcolor[HTML]{DEF1EB}87.1 & \cellcolor[HTML]{F6D8C5}151 \\
 & DINOv3-B & \cellcolor[HTML]{F5D3BC}22.4 & \cellcolor[HTML]{C9E8DE}10.8 & \cellcolor[HTML]{C0E4D9}9.81 & \cellcolor[HTML]{C8E7DD}10.7 & \cellcolor[HTML]{BEE3D7}\textbf{9.53} & \cellcolor[HTML]{E5F4EF}14.0 & \cellcolor[HTML]{D7EEE7}12.4 & \cellcolor[HTML]{F6D8C4}21.5 \\
 & DINOv3-L & \cellcolor[HTML]{F5D3BC}63.7 & \cellcolor[HTML]{DDF0EB}31.1 & \cellcolor[HTML]{D9EFE8}29.4 & \cellcolor[HTML]{DDF0EA}31.0 & \cellcolor[HTML]{BEE3D7}\textbf{18.5} & \cellcolor[HTML]{FCF4EF}44.0 & \cellcolor[HTML]{F1F9F6}38.8 & \cellcolor[HTML]{F5D5C0}62.0 \\
 & SigLIP-B & \cellcolor[HTML]{F5D3BC}7.54 & \cellcolor[HTML]{CAE8DF}3.85 & \cellcolor[HTML]{BEE3D7}\textbf{3.41} & \cellcolor[HTML]{C9E7DE}3.80 & \cellcolor[HTML]{D6EDE6}4.28 & \cellcolor[HTML]{D1EBE3}4.11 & \cellcolor[HTML]{CBE8DF}3.88 & \cellcolor[HTML]{F6D9C6}7.19 \\
 & SigLIP-L & \cellcolor[HTML]{F5D3BC}31.2 & \cellcolor[HTML]{CAE8DF}16.1 & \cellcolor[HTML]{BEE3D7}\textbf{14.3} & \cellcolor[HTML]{C9E8DE}16.0 & \cellcolor[HTML]{CDE9E0}16.5 & \cellcolor[HTML]{D6EDE6}18.0 & \cellcolor[HTML]{C9E8DE}16.1 & \cellcolor[HTML]{F6D9C5}29.9 \\
\midrule
\multirow{7}{*}{\rotatebox[origin=c]{90}{KD$\times 10^3$}} & Inception & \cellcolor[HTML]{F5D3BC}2.89 & \cellcolor[HTML]{CBE9E0}0.66 & \cellcolor[HTML]{E1F2ED}1.16 & \cellcolor[HTML]{CAE8DF}0.62 & \cellcolor[HTML]{FDF7F3}1.69 & \cellcolor[HTML]{BFE3D8}0.38 & \cellcolor[HTML]{BEE3D7}\textbf{0.36} & \cellcolor[HTML]{F9E5D7}2.30 \\
 & DINOv2-B & \cellcolor[HTML]{F5D3BC}192 & \cellcolor[HTML]{C5E6DC}57.0 & \cellcolor[HTML]{BEE3D7}\textbf{47.5} & \cellcolor[HTML]{CBE8DF}63.8 & \cellcolor[HTML]{C2E5DA}52.9 & \cellcolor[HTML]{E2F3EE}94.5 & \cellcolor[HTML]{D7EEE7}80.3 & \cellcolor[HTML]{F6D9C5}181 \\
 & DINOv2-L & \cellcolor[HTML]{F5D3BC}165 & \cellcolor[HTML]{CCE9E0}60.5 & \cellcolor[HTML]{BEE3D7}\textbf{45.8} & \cellcolor[HTML]{D0EBE2}64.7 & \cellcolor[HTML]{C5E6DC}53.1 & \cellcolor[HTML]{E7F5F0}89.1 & \cellcolor[HTML]{DDF0EA}78.9 & \cellcolor[HTML]{F6D9C5}156 \\
 & DINOv3-B & \cellcolor[HTML]{F5D3BC}36.3 & \cellcolor[HTML]{CAE8DF}16.3 & \cellcolor[HTML]{C1E4D9}14.5 & \cellcolor[HTML]{C7E7DD}15.7 & \cellcolor[HTML]{BEE3D7}\textbf{13.9} & \cellcolor[HTML]{E6F4F0}21.9 & \cellcolor[HTML]{D8EEE7}19.1 & \cellcolor[HTML]{F6D7C3}34.9 \\
 & DINOv3-L & \cellcolor[HTML]{F5D3BC}78.3 & \cellcolor[HTML]{DAEFE8}30.0 & \cellcolor[HTML]{D6EDE6}28.2 & \cellcolor[HTML]{D8EEE7}29.2 & \cellcolor[HTML]{BEE3D7}\textbf{14.3} & \cellcolor[HTML]{FDF6F1}49.0 & \cellcolor[HTML]{EDF7F4}41.1 & \cellcolor[HTML]{F6D6C1}75.8 \\
 & SigLIP-B & \cellcolor[HTML]{F5D3BC}18.5 & \cellcolor[HTML]{CDE9E1}8.42 & \cellcolor[HTML]{BEE3D7}\textbf{6.84} & \cellcolor[HTML]{CEEAE1}8.47 & \cellcolor[HTML]{E5F4EF}10.8 & \cellcolor[HTML]{D2EBE4}8.90 & \cellcolor[HTML]{CCE9E0}8.28 & \cellcolor[HTML]{F6D9C5}17.5 \\
 & SigLIP-L & \cellcolor[HTML]{F5D3BC}78.8 & \cellcolor[HTML]{CBE8DF}34.4 & \cellcolor[HTML]{BEE3D7}\textbf{28.7} & \cellcolor[HTML]{CAE8DE}33.9 & \cellcolor[HTML]{D4EDE5}38.7 & \cellcolor[HTML]{D4ECE5}38.5 & \cellcolor[HTML]{C8E7DE}33.2 & \cellcolor[HTML]{F6D9C5}74.8 \\
\bottomrule
\end{tabular}

\end{table}

\FloatBarrier
\clearpage
\subsection{Effect of CFG on Reward Scores and Fr\'echet Distance}
\label{app:cfg_effect}

\ifmeta
\begin{figure}[H]
    \centering
    \includegraphics[width=\textwidth]{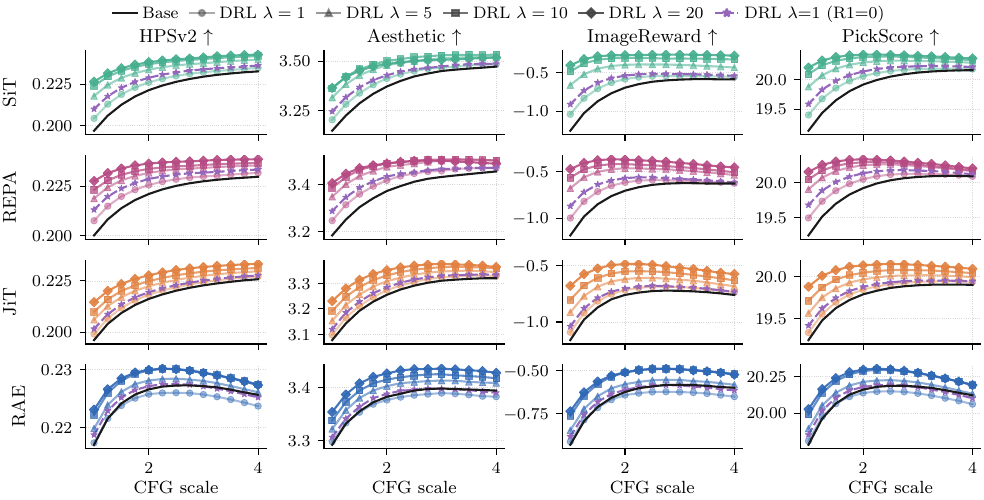}
    \caption{{\textbf{Effect of CFG on reward scores.}
    Reward vs.\ CFG scale for four models (rows) and four reward metrics
    (columns).  Black: base model
    (drawn on top of the DRL curves for visibility); colored lines: DRL at
    $\lambda \in \{1, 5, 10, 20\}$ (darker = higher $\lambda$); purple dashed:
    DRL $\lambda{=}1$ with $R_1{=}0$ discriminator.
    DRL beats base at every CFG, and sometimes (e.g., SiT on ImageReward)
    even without CFG.}}
    \label{fig:cfg_effect_reward_raw}
\end{figure}
\fi

\ifmeta
\begin{figure}[H]
    \centering
    \includegraphics[width=\textwidth]{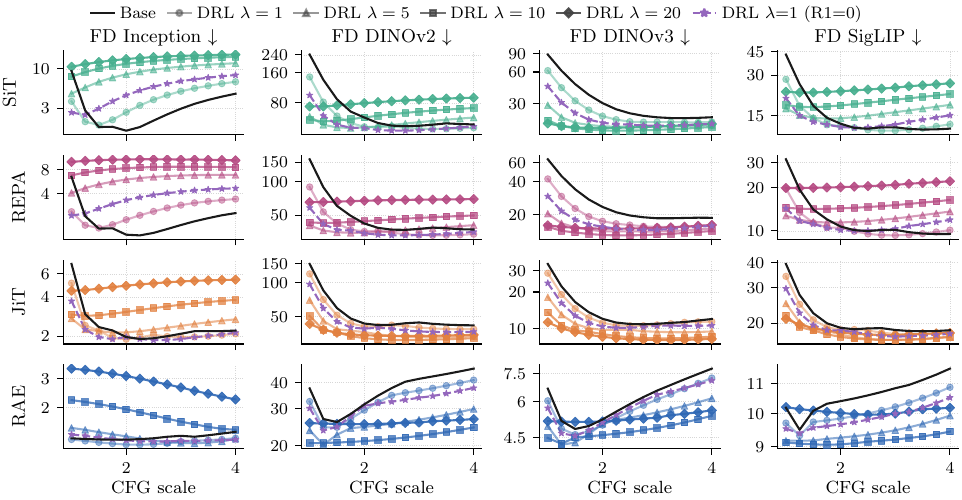}
    \caption{{\textbf{Effect of CFG on Fr\'echet Distance.}
    FD vs.\ CFG scale for four models (rows) and four evaluation embedders
    (columns). All panels use a log $y$-axis. Each point shows the best
    result across the three CFG intervals defined in \cref{app:cfg_tuning}.
    Same legend as \cref{fig:cfg_effect_reward_raw}.
    Inception/SigLIP prefer smaller $\lambda$, DINOv2/v3 prefer larger;
    CFG also helps less as $\lambda$ grows. Both are likely artifacts of
    fine-tuning only the conditional branch (required by adjoint matching).}}
    \label{fig:cfg_effect_fid_raw}
\end{figure}
\fi

\FloatBarrier
\clearpage
\section{Qualitative Samples}
\label{app:qualitative_samples}

This section collects the qualitative sample grids referenced from the main
text. All grids use matched noise seeds and class labels across conditions, so
visual differences are attributable to model weights and guidance only.

\paragraph{Base vs.\ Fine-tuned.}
Uncurated samples comparing the base model against the DRL fine-tuned model at $\lambda{=}10$ for each architecture. Each figure shows two sections side by side: \textit{No CFG} and \textit{Best CFG}, where the latter selects the best CFG value according to DINOv3-L FD. We use DINOv3 as a held-out metric since the discriminator is trained on DINOv2 features.

\paragraph{Effect of DRL Strength.}
Samples for each model under varying DRL strength ($\lambda \in \{1, 5, 10, 20, 40\}$) and three guidance strategies. Each figure uses matched noise seeds and class labels across all conditions, so differences are purely due to model weights. Reading left-to-right shows the effect of increasing $\lambda$; the three sections compare no guidance (CFG${}=1$), standard CFG${}=2$, and the per-variant optimal guidance (DINOv3-L FD-optimal CFG or autoguidance, annotated above each column; see \cref{app:autoguidance} for context on autoguidance).
As remarked in \cref{app:autoguidance}, autoguidance is often the best guidance setting at large $\lambda$, and noticeably improves the images compared to the CFG variant.

\paragraph{RL Fine-tuning Sample Images.}
Generated samples for each model under varying RL strength ($\lprl \in \{1, 10, 40\}$) and guidance strategies (no CFG, CFG${}=2$, and autoguidance${}=2$; see \cref{app:autoguidance} for context on autoguidance). Each figure uses matched noise seeds across all conditions, so differences are purely due to model weights and guidance. The three column groups compare Base, DRL$+$PRL, and Base$+$PRL.
As discussed in the main text, increasing $\lprl$ makes the images progressively brighter and more distorted. This effect is particularly pronounced under autoguidance, which---as noted in \cref{app:pareto,app:autoguidance}---yields the highest reward but the largest distortion.

\clearpage
\subsection{Base vs.\ Fine-tuned}
\label{app:base_vs_ours}

\ifmeta
\begin{figure}[H]
    \centering
    \includegraphics[height=0.92\textheight,keepaspectratio]{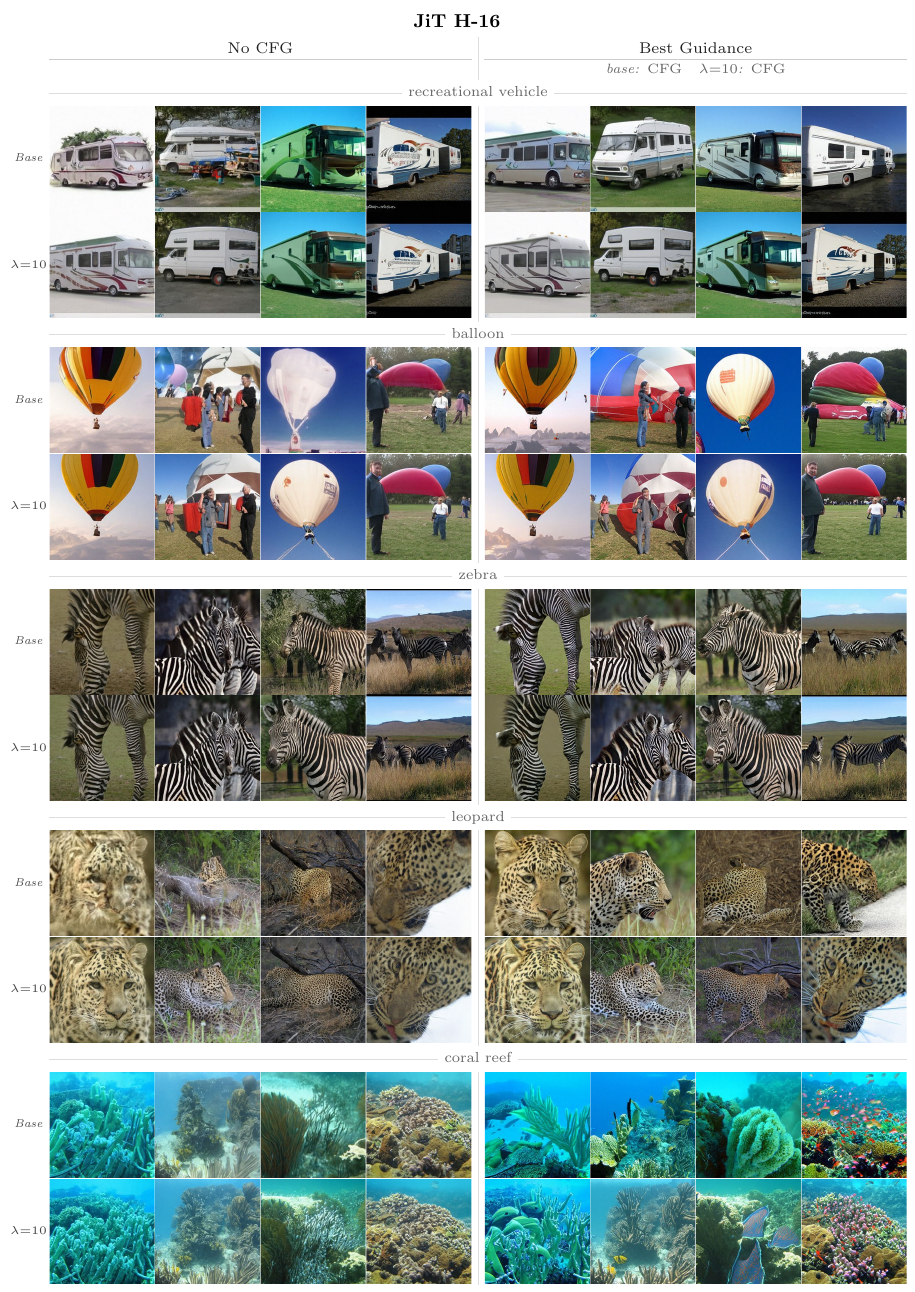}
    \caption{Base vs.\ DRL fine-tuned ($\lambda{=}10$) samples for JiT H-16.}
    \label{fig:base_vs_ours_jit}
\end{figure}
\fi

\ifmeta
\begin{figure}[H]
    \centering
    \includegraphics[height=0.92\textheight,keepaspectratio]{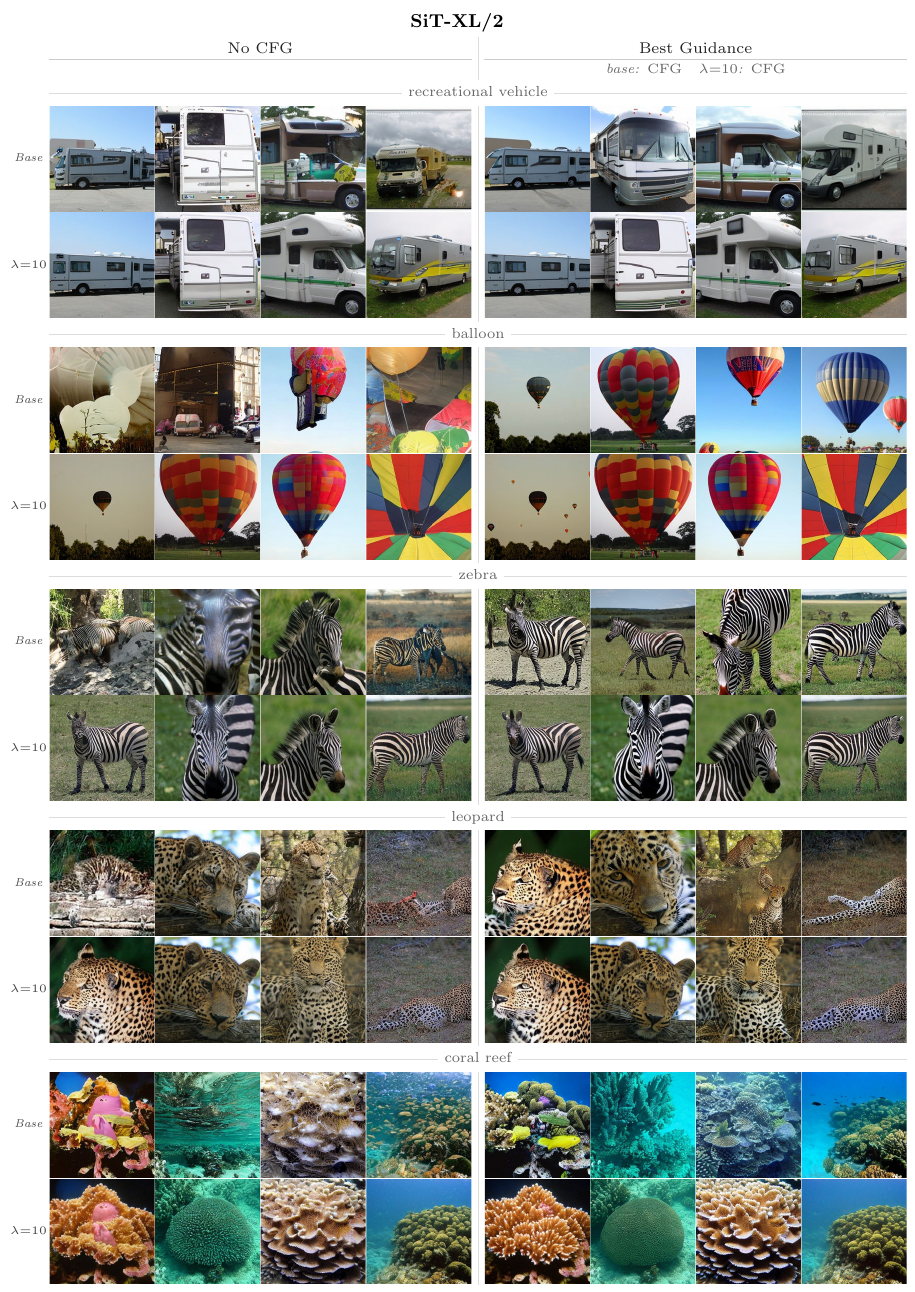}
    \caption{Base vs.\ DRL fine-tuned ($\lambda{=}10$) samples for SiT-XL/2.}
    \label{fig:base_vs_ours_sit}
\end{figure}
\fi

\ifmeta
\begin{figure}[H]
    \centering
    \includegraphics[height=0.92\textheight,keepaspectratio]{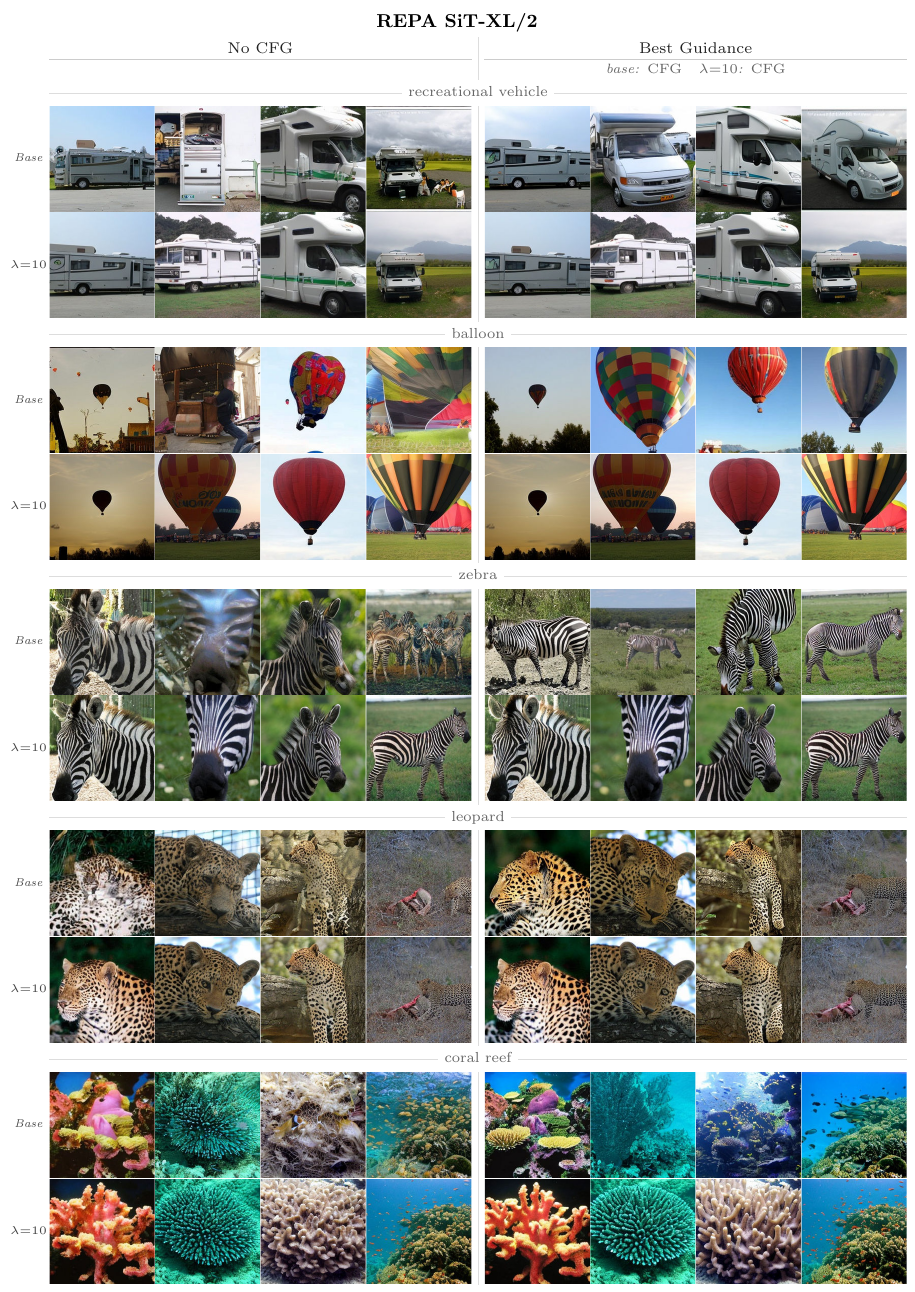}
    \caption{Base vs.\ DRL fine-tuned ($\lambda{=}10$) samples for REPA SiT-XL/2.}
    \label{fig:base_vs_ours_repa}
\end{figure}
\fi

\ifmeta
\begin{figure}[H]
    \centering
    \includegraphics[height=0.92\textheight,keepaspectratio]{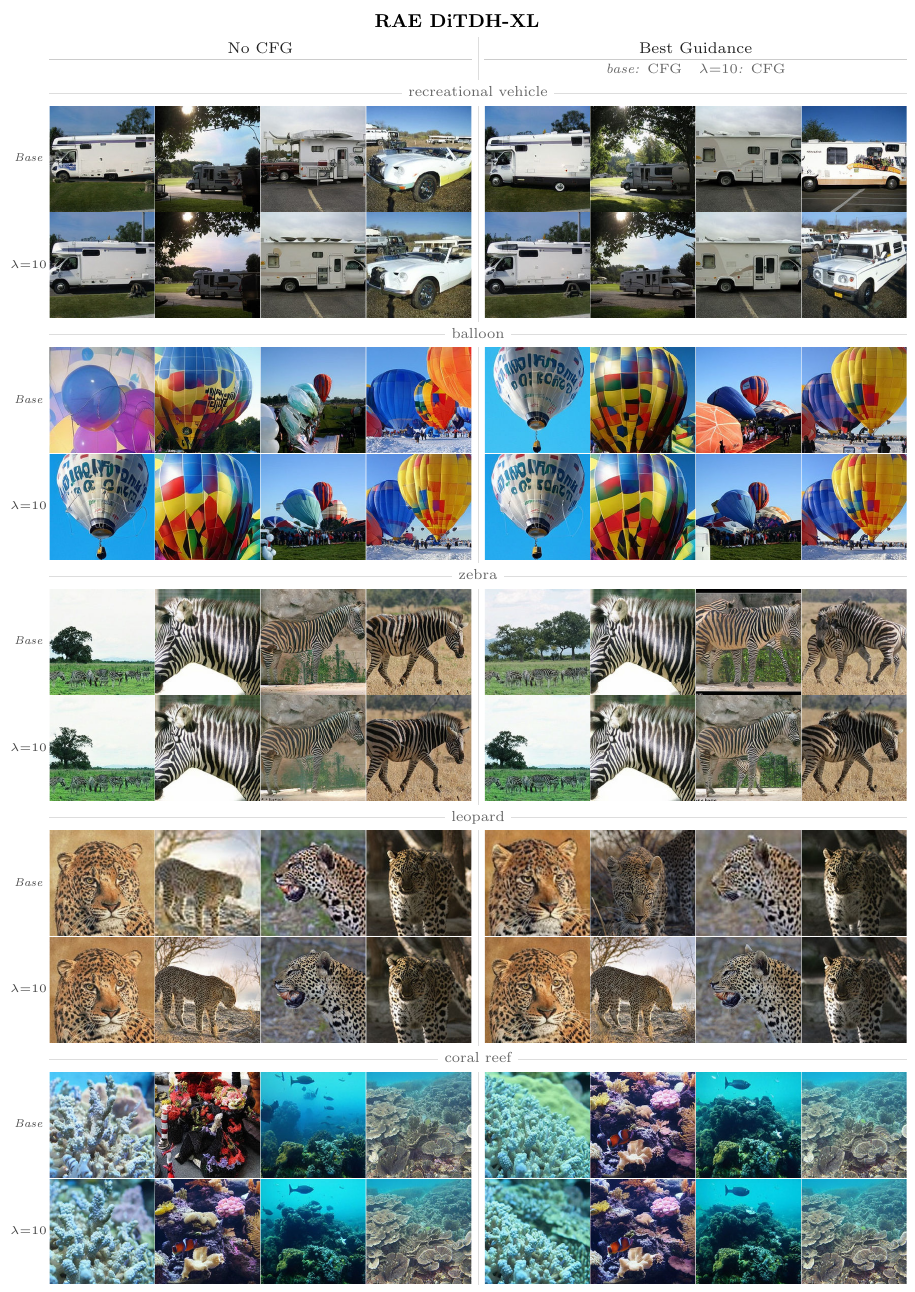}
    \caption{Base vs.\ DRL fine-tuned ($\lambda{=}10$) samples for RAE DiTDH-XL.}
    \label{fig:base_vs_ours_rae}
\end{figure}
\fi

\clearpage
\subsection{Effect of DRL Strength}
\label{app:am_samples}

\ifmeta
\begin{figure}[H]
    \centering
    \includegraphics[width=\textwidth,height=0.92\textheight,keepaspectratio]{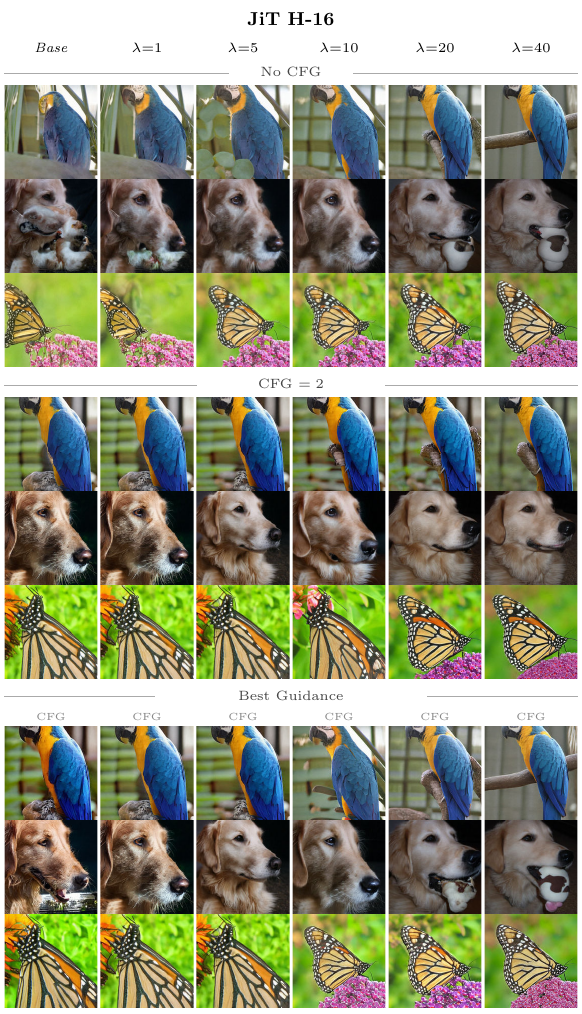}
    \caption{Effect of DRL strength ($\lambda$) on samples for JiT H-16.}
    \label{fig:am_samples_jit}
\end{figure}
\fi

\ifmeta
\begin{figure}[H]
    \centering
    \includegraphics[width=\textwidth,height=0.92\textheight,keepaspectratio]{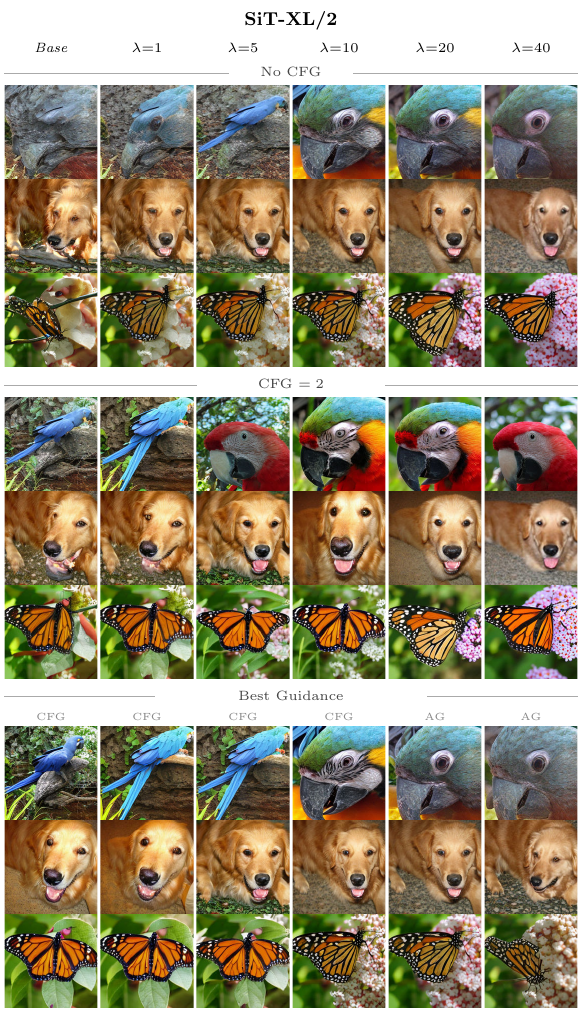}
    \caption{Effect of DRL strength ($\lambda$) on samples for SiT-XL/2.}
    \label{fig:am_samples_sit}
\end{figure}
\fi

\ifmeta
\begin{figure}[H]
    \centering
    \includegraphics[width=\textwidth,height=0.92\textheight,keepaspectratio]{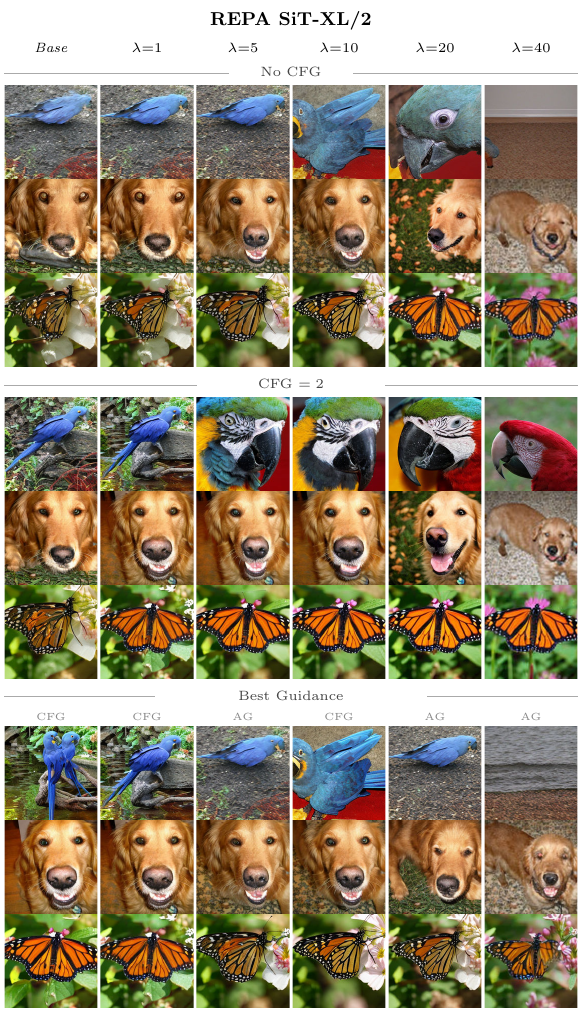}
    \caption{Effect of DRL strength ($\lambda$) on samples for REPA SiT-XL/2.}
    \label{fig:am_samples_repa}
\end{figure}
\fi

\ifmeta
\begin{figure}[H]
    \centering
    \includegraphics[width=\textwidth,height=0.92\textheight,keepaspectratio]{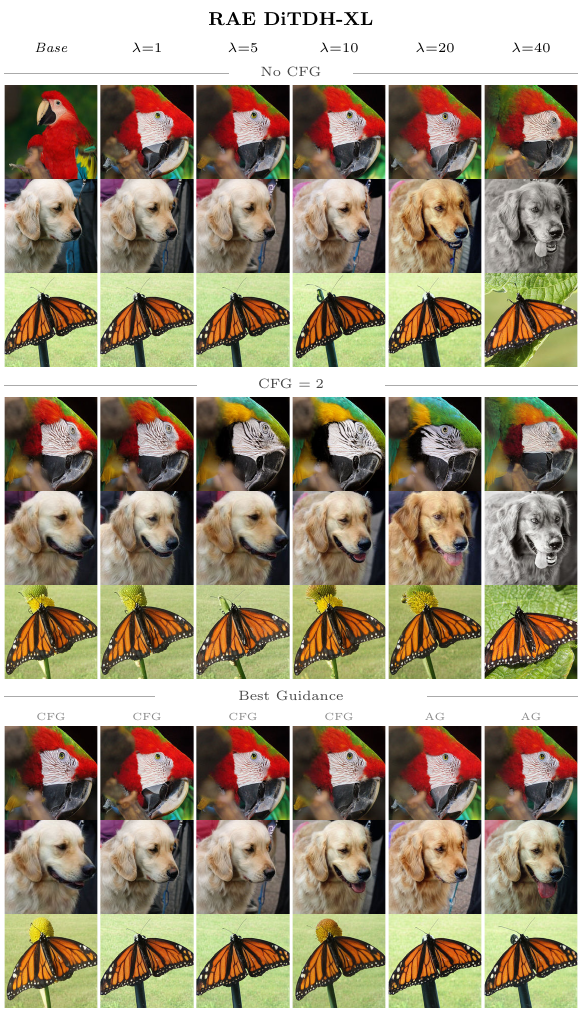}
    \caption{Effect of DRL strength ($\lambda$) on samples for RAE DiTDH-XL.}
    \label{fig:am_samples_rae}
\end{figure}
\fi

\clearpage
\subsection{RL Fine-tuning Sample Images}
\label{app:rl_samples}

\ifmeta
\begin{figure}[H]
    \centering
    \includegraphics[width=\textwidth,height=0.92\textheight,keepaspectratio]{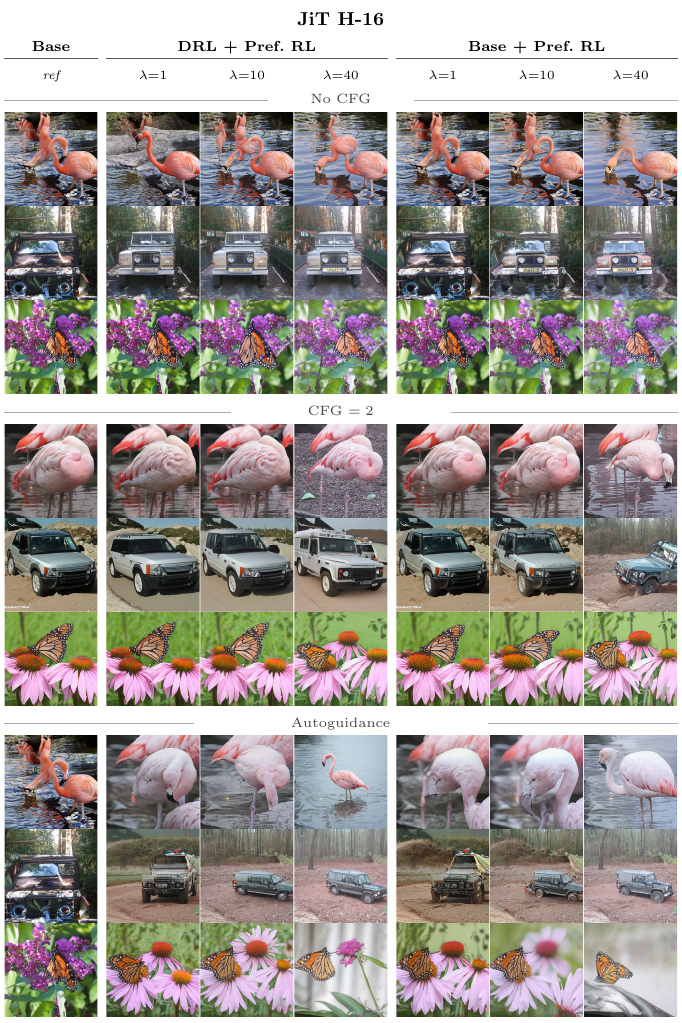}
    \caption{RL fine-tuning samples for JiT H-16.}
    \label{fig:rl_samples_jit}
\end{figure}
\fi

\ifmeta
\begin{figure}[H]
    \centering
    \includegraphics[width=\textwidth,height=0.92\textheight,keepaspectratio]{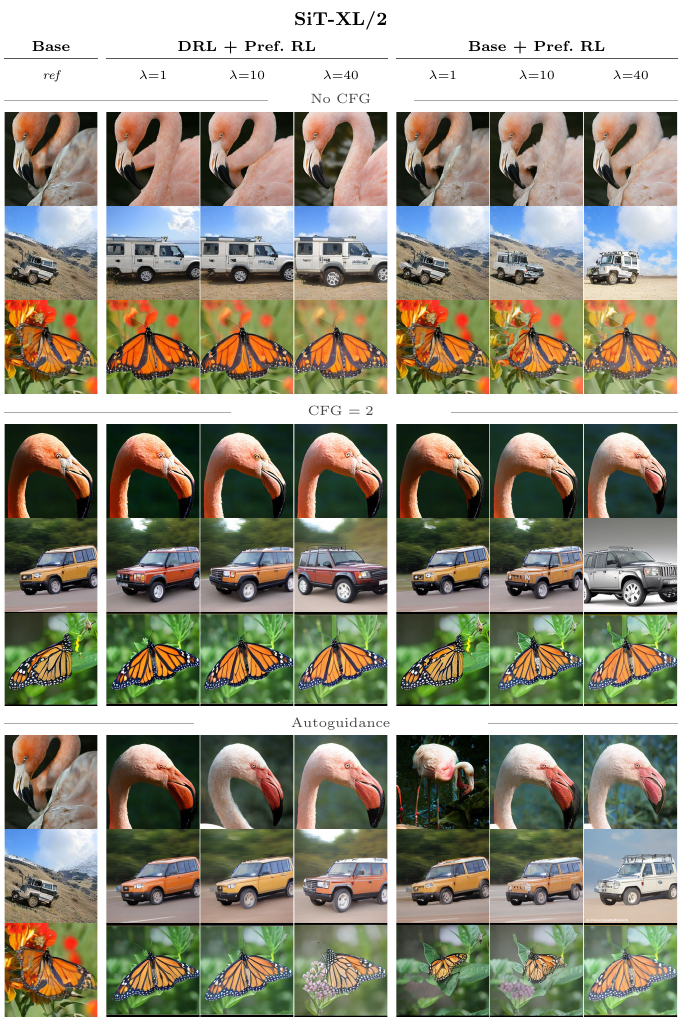}
    \caption{RL fine-tuning samples for SiT-XL/2.}
    \label{fig:rl_samples_sit}
\end{figure}
\fi

\ifmeta
\begin{figure}[H]
    \centering
    \includegraphics[width=\textwidth,height=0.92\textheight,keepaspectratio]{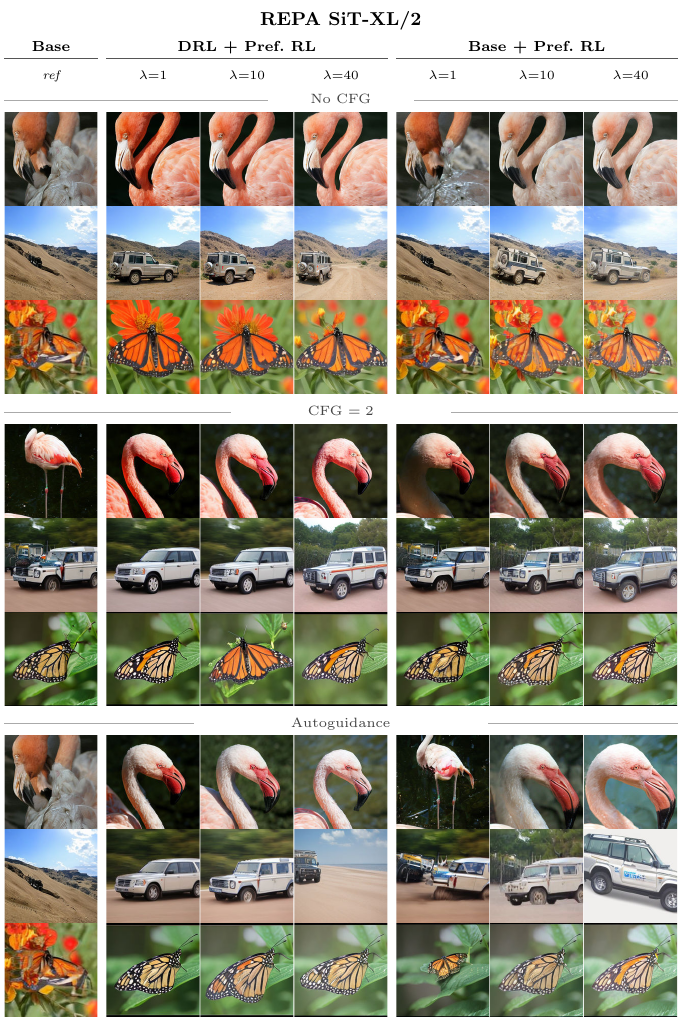}
    \caption{RL fine-tuning samples for REPA SiT-XL/2.}
    \label{fig:rl_samples_repa}
\end{figure}
\fi

\ifmeta
\begin{figure}[H]
    \centering
    \includegraphics[width=\textwidth,height=0.92\textheight,keepaspectratio]{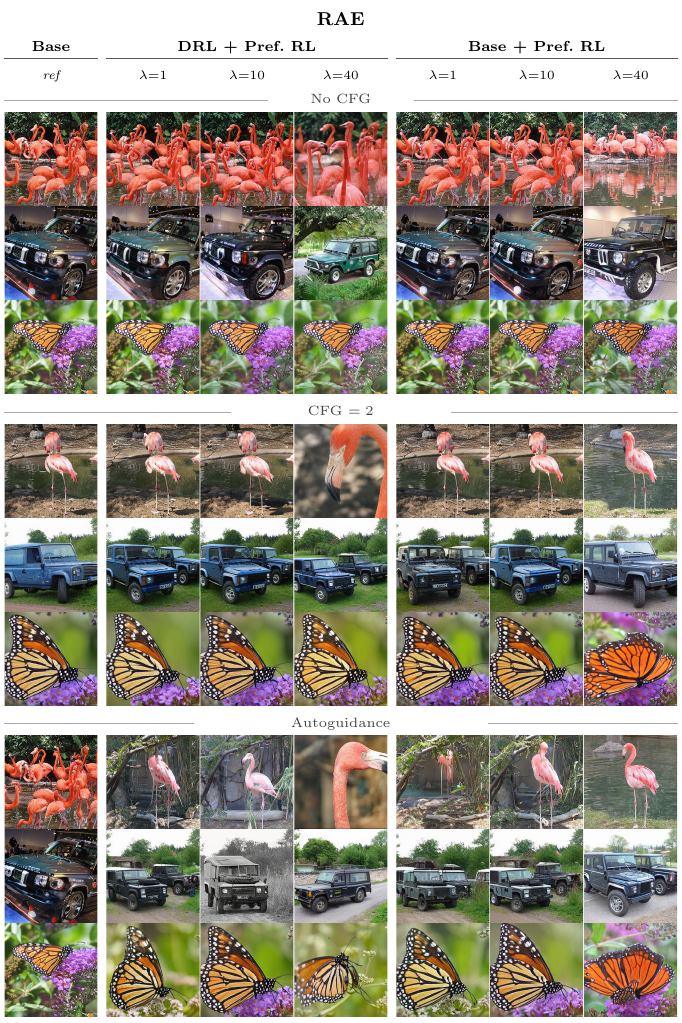}
    \caption{RL fine-tuning samples for RAE.}
    \label{fig:rl_samples_rae}
\end{figure}
\fi

\end{document}